\pdfoutput=1
\documentclass{article}

\usepackage[dvipsnames]{xcolor}
\usepackage{microtype}
\usepackage{graphicx}
\usepackage{subcaption}
\usepackage{booktabs}
\usepackage{hyperref}

\usepackage[accepted]{icml2024}

\usepackage{amsmath}
\usepackage{amssymb}
\usepackage{mathtools}
\usepackage{amsthm}
\usepackage{amsfonts}
\usepackage{bbm}
\usepackage{mathrsfs}
\usepackage{enumitem}
\usepackage{resizegather}
\usepackage{pifont}
\usepackage{balance}
\usepackage{multirow}

\usepackage{listings}
\usepackage[capitalize,noabbrev]{cleveref}

\newcommand{\ie}{\emph{i.e., }}
\newcommand{\eg}{\emph{e.g., }}

\newcommand{\st}{\emph{s.t. }}

\newcommand{\wrt}{\emph{w.r.t. }}

\newcommand{\ms}[2]{#1\footnotesize{$\pm$#2}}
\newcommand{\bms}[2]{\textbf{#1}\footnotesize{$\pm$#2}}
\newcommand{\ums}[2]{\underline{#1}\footnotesize{$\pm$#2}}

\theoremstyle{plain}
\newtheorem{theorem}{Theorem}[section]

\newtheorem{lemma}[theorem]{Lemma}
\newtheorem{corollary}[theorem]{Corollary}
\theoremstyle{definition}

\theoremstyle{remark}

\usepackage[textsize=tiny]{todonotes}

\begin{document}

\twocolumn[
\icmltitle{On the Maximal Local Disparity of Fairness-Aware Classifiers}

\begin{icmlauthorlist}
\icmlauthor{Jinqiu Jin}{ustc}
\icmlauthor{Haoxuan Li}{pku}
\icmlauthor{Fuli Feng}{ustc}
\end{icmlauthorlist}
\icmlaffiliation{ustc}{University of Science and Technology of China}
\icmlaffiliation{pku}{Peking University}
\icmlcorrespondingauthor{Haoxuan Li}{hxli@stu.pku.edu.cn}

\icmlkeywords{Machine Learning, ICML}

\vskip 0.3in
]
\printAffiliationsAndNotice{}  

\begin{abstract}

Fairness has become a crucial aspect in the development of trustworthy machine learning algorithms. Current fairness metrics to measure the violation of demographic parity have the following drawbacks: (i) the \emph{average difference} of model predictions on two groups cannot reflect their \emph{distribution disparity}, and (ii) the \emph{overall} calculation along all possible predictions conceals the \emph{extreme local disparity} at or around certain predictions. In this work, we propose a novel fairness metric called \textbf{M}aximal \textbf{C}umulative ratio \textbf{D}isparity along varying \textbf{P}redictions' neighborhood (MCDP), for measuring the maximal local disparity of the fairness-aware classifiers. To accurately and efficiently calculate the MCDP, we develop a provably exact and an approximate calculation algorithm that greatly reduces the computational complexity with low estimation error. We further propose a bi-level optimization algorithm using a differentiable approximation of the MCDP for improving the algorithmic fairness. Extensive experiments on both tabular and image datasets validate that our fair training algorithm can achieve superior fairness-accuracy trade-offs.

\end{abstract}

\section{Introduction}

\begin{figure*}[t]
    \centering
    \subcaptionbox{Small $\Delta\mathrm{DP}$, but large $\mathrm{ABCC}$\label{fig:1a}}{
        \vspace{-5pt}
        \centering
        \includegraphics[width=0.253\linewidth]{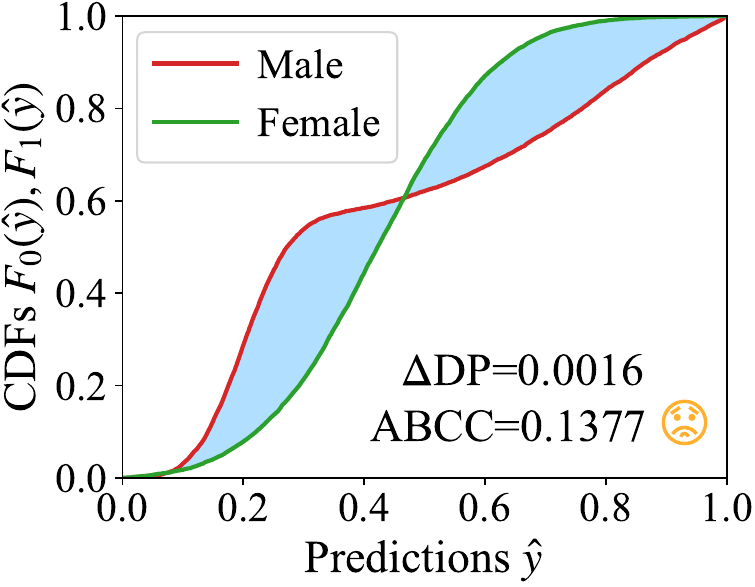}}\qquad 
    \subcaptionbox{Large maximal local disparity\label{fig:1b}}{
        \vspace{-5pt}
        \includegraphics[width=0.253\linewidth]{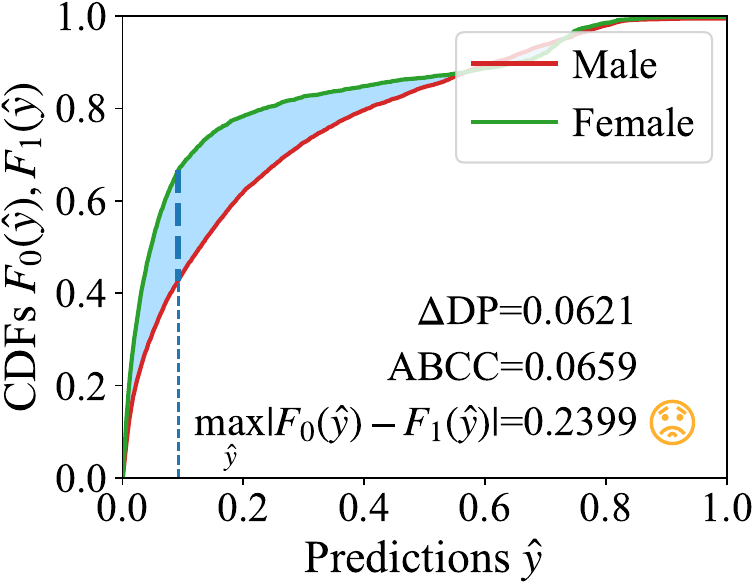}}\qquad 
    \subcaptionbox{Small maximal local disparity\label{fig:1c}}{
        \vspace{-5pt}
        \includegraphics[width=0.253\linewidth]{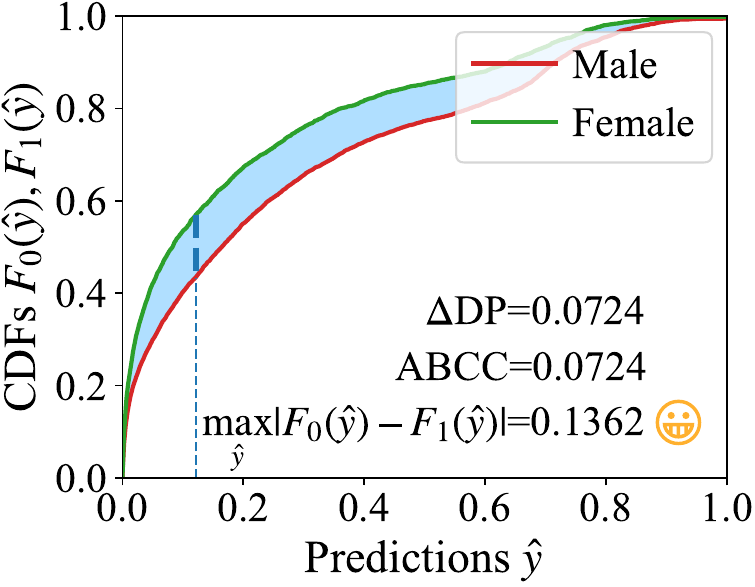}}\ 
    \vspace{-8pt}
    \caption{The empirical distribution functions of model predictions over male and female groups. (a) shows a toy example, while (b) and (c) are the testing results of FairMixup \cite{chuang2021fair} and our proposed algorithm on the Adult dataset, respectively.} %
    \label{fig:1}
    \vspace{-10pt}
\end{figure*}

Nowadays, machine learning algorithms have been widely-used in high-stake applications such as loan management \cite{mukerjee2002multi}, job-hiring \cite{faliagka2012application}, and recidivism prediction \cite{berk2021fairness}. Nonetheless, these algorithms are prone to exhibit discrimination against particular groups, leading to unfair decision-making results \cite{tolan2019machine,raghavan2020mitigating,mehrabi2021survey}.  To address this issue, growing attentions have been paid to developing comprehensive fairness criterions \cite{jacobs2021measurement,han2023retiring} and effective fair learning algorithms \cite{wan2023processing,han2023ffb}.

Existing group fairness notions require algorithms to treat different groups equally, and the degree of fairness violation is usually measured via the dissimilarity of model predictions. For example, \textbf{D}emographic \textbf{P}arity (DP) requires model predictions to be independent of sensitive attributes \cite{dwork2012fairness,kamishima2012fairness,jiang2020wasserstein}. To measure the violation of DP, most of existing works adopt $\Delta\mathrm{DP}$ metric, which calculates the \textbf{difference in average} predictions between the two demographic groups \cite{zemel2013learning,chuang2021fair,li2023fairer}. However, since having the same values in average predictions between the two groups cannot ensure that the distributions are also the same, we argue that the widely used $\Delta\mathrm{DP}$ may fail to detect the violation of demographic parity. Figure \ref{fig:1a} illustrates a toy example where the red and green curves represent the \textbf{C}umulative \textbf{D}istribution \textbf{F}unction (CDF) for the male and female groups, respectively. Despite the small $\Delta\mathrm{DP}$, it is clear that the prediction distribution is not independent of gender as a sensitive attribute, with males more likely to be assigned extremely small and large predictive probabilities, while females are more concentrated in the middle-sized probabilities. To show the limitation of $\Delta\mathrm{DP}$, we further calculate the \textbf{A}rea \textbf{B}etween \textbf{C}DF \textbf{C}urves ($\mathrm{ABCC}$) to measure the total variation between CDFs~\cite{han2023retiring}, which demonstrates significantly greater differences between groups.

Moreover, as most of the existing metrics measure the \textbf{overall} disparity along all possible predictions, they fail to capture the \textbf{local} disparity at or around certain predictions. As a matter of fact, different metric values do not imply the relative magnitude of extreme local disparities. 
For example, although both $\Delta\mathrm{DP}$ and $\mathrm{ABCC}$ values of predictions in Figure \ref{fig:1b} is better than those in Figure \ref{fig:1c}, its CDF disparity along predictions is less uniform, resulting in much more serious maximal disparity ($0.2399>0.1362$). Generally, a pre-defined threshold is used to make a binary decision based on model predictions, where varying thresholds lead to changing proportions of positive decisions of different groups \cite{chen2020towards}. If the decision threshold takes the value where the maximal CDF disparity is achieved, the group unfairness would be seriously exacerbated (\eg the difference of group positive proportion in Figure \ref{fig:1b} will be up to nearly 24\%). Therefore, it's important to capture and measure the extreme local disparity, which is yet prone to be ``averaged" by overall variation of previous fairness metrics.

To address the limitations of previous fairness metrics, we propose a novel fairness metric called \textbf{M}aximal \textbf{C}umulative ratio \textbf{D}isparity along varying \textbf{P}redictions' $\epsilon$-neighborhood, denoted as $\mathrm{MCDP}(\epsilon)$, whose core idea is calculating the maximal local disparity of the CDFs of two demographic groups. Considering that vast of predictions falling in a small prediction interval may result in sharp distribution changes, we firstly replace the exact CDF disparity at each prediction as the minimal disparity within its $\epsilon$-neighborhood. In this way, the group disparity along varying predictions becomes smoother, thus the metric value becomes more robust to sharp distributions. We also theoretically prove several properties of the proposed $\mathrm{MCDP}(\epsilon)$ metric, including but not limited to its monotonicity \wrt $\epsilon$, and its relationship with previous fairness metrics.

Furthermore, given the model predictions for real instances in two demographic groups, we adopt empirical distribution function as the estimated CDF, and further propose two algorithms which exactly and approximately calculate the empirical metric value, respectively. Specifically, the approximation algorithm can greatly reduce the computational complexity with low value error. To train a fair classifier in view of maximal local disparity, we firstly adopt a differentiable approximation of CDF disparity based on temperature sigmoid function \cite{han1995influence}, and then minimize the maximal estimated CDF disparity using the bi-level optimization approach \cite{ji2021bilevel}. In this way, the distribution disparity becomes more uniform (Figure \ref{fig:1c}), further approaching demographic parity.

The contributions of this paper are summarized as follows:

$\bullet$ We propose a novel fairness metric called $\mathrm{MCDP}(\epsilon)$ to measure the maximal local disparity of a classifier, and derive several theoretical properties about the metric.

\vspace{-3pt}
$\bullet$ To empirically estimate $\mathrm{MCDP}(\epsilon)$ with finite instances, we propose an exact and an approximate calculation algorithm, where the approximate algorithm greatly reduces computational complexity with low estimation error.

\vspace{-3pt}
$\bullet$ We further develop a bi-level optimization algorithm using differentiable approximation of $\mathrm{MCDP}(0)$ to train a fair classifier which minimizes the maximal local disparity.

\vspace{-3pt}
$\bullet$ Experiments on tabular and image datasets (Adult, Bank and CelebA) demonstrate that our learning algorithm can effectively achieve better fairness-accuracy trade-offs.

\section{Preliminaries}

\subsection{Demographic Parity and Measurements}
Without loss of generality, we consider the binary classification task where each instance consists of an input $\boldsymbol{x}\in \mathcal{X}\subset\mathbb{R}^d$, a class label $y\in\mathcal{Y}=\{0,1\}$ and a group label $s\in\{0,1\}$, which is defined by sensitive attributes such as gender, age or race. We focus on demographic parity which requires the model's predictive probabilities $\hat{y}=f(\boldsymbol{x})$ to be independent of the group label $s$, where $f:\mathcal{X}\to[0,1]$ is a classifier with model parameter $\boldsymbol{\theta}$. To measure the model's fairness violation, the following metric calculates the average prediction difference of two groups 
\begin{equation*}
    \Delta{\mathrm{DP}}(f)= \left|\mathbb{E}_{\boldsymbol{x}\sim \mathcal{P}_0} f(\boldsymbol{x}) -\mathbb{E}_{\boldsymbol{x}\sim \mathcal{P}_1} f(\boldsymbol{x})\right|,
\end{equation*}
where $\mathcal{P}_a=\mathbb{P}(\boldsymbol{x}|s=a),a\in\{0,1\}$ denotes the distributions of instances in two groups. In addition, given a finite dataset $\mathcal{D}=\{\boldsymbol{x}_i,y_i,s_i\}_{i=1}^N$ with $N$ samples and model predictions $\{\hat{y}_i\}_{i=1}^N$ with $\hat{y}_i=f(\boldsymbol{x}_i)$, the empirical estimation of $\Delta{\mathrm{DP}}(f)$ can be obtained as 
\begin{equation*}
    \Delta\widehat{\mathrm{DP}}(f,\mathcal{D}) =\left|\frac{1}{|\mathcal{S}_0|} \sum\limits_{i\in\mathcal{S}_0}\hat{y}_i -\frac{1}{|\mathcal{S}_1|} \sum\limits_{i\in\mathcal{S}_1}\hat{y}_i\right|,
\end{equation*}
where $\mathcal{S}_a=\{i:s_i=a\},a\in\{0,1\}$ is the index set of instances in two groups. Apart from the widely-used $\Delta{\mathrm{DP}}$ ($\Delta\widehat{\mathrm{DP}}$) metric\footnote{To improve the notation briefness, we omit the arguments $f$ and $\mathcal{D}$, \eg $\Delta{\mathrm{DP}}, \Delta\widehat{\mathrm{DP}}$ is short for $\Delta{\mathrm{DP}}(f),\Delta{\widehat{\mathrm{DP}}}(f,\mathcal{D})$.} which measures unfairness in expectation-level, recent work \cite{han2023retiring} proposes a distribution-level metric called $\mathrm{ABCC}$ as follows
\begin{equation*}
    \mathrm{ABCC}(f)= \int_{0}^{1}\left|F_0(\hat{y})-F_1(\hat{y})\right| \mathrm{d}\hat{y},
\end{equation*}
where $F_0(\hat{y})$ and $F_1(\hat{y})$ ($\hat{y}\in[0,1]$) are the CDFs of model predictions of instances from $\mathcal{P}_0$ and $\mathcal{P}_1$
\begin{equation*}
    F_a(\hat{y})= \mathbb{P}(f(\boldsymbol{x})\leq\hat{y}),\ \boldsymbol{x}\sim\mathcal{P}_a,\ a\in\{0,1\}.
\end{equation*}
Similar to $\Delta\widehat{\mathrm{DP}}$, the estimated $\mathrm{ABCC}$ value using $\{\hat{y}_i\}_{i=1}^N$ can be computed as follows
\begin{equation*}
    \widehat{\mathrm{ABCC}}(f,\mathcal{D})= \int_{0}^{1}\left|\hat{F}_0(\hat{y}) -\hat{F}_1(\hat{y})\right|\mathrm{d}\hat{y},
\end{equation*}
where $\hat{F}_0(\hat{y})$ and $\hat{F}_1(\hat{y})$ are the empirical distribution functions of model predictions of instances in two groups
\begin{equation}\label{eq:ecdf}
    \hat{F}_a(\mathcal{D};\hat{y})=\frac{1}{|\mathcal{S}_a|} \sum\limits_{i\in\mathcal{S}_a}\mathbb{I}(\hat{y}_i\leq\hat{y}),\ a\in\{0,1\},
\end{equation}
and $\mathbb{I}(\cdot)$ denotes the indicator function.

\subsection{Discussions about Previous Metrics}
\noindent\textbf{Drawbacks of $\Delta\mathrm{DP}$.} 
Although $\Delta\mathrm{DP}$ has become the \emph{de facto} fairness criterion in previous literatures, it is insufficient to measure the violation of demographic parity. The reason is that $\Delta\mathrm{DP}=0$ does not indicate identical distributions of group prediction, thus the independency of predictions and group labels cannot be guaranteed. As shown in Figure \ref{fig:1a}, the distribution gap between the two groups is evident despite $\Delta\mathrm{DP}$ is very close to $0$.

\noindent\textbf{Drawbacks of $\mathrm{ABCC}$.} 
Unlike $\Delta\mathrm{DP}$, $\mathrm{ABCC}=0$ is a necessary and sufficient condition for establishing demogarphic parity. However, as $\mathrm{ABCC}$ value cannot reflect the \emph{local} distribution disparity, it fails to accurately measure the degree of unfairness in cases where extreme local disparity is emphasized. For example, although the maximal disparity in Figure \ref{fig:1c} ($0.1362$) is much smaller than that of Figure \ref{fig:1b} ($0.2399$), its $\mathrm{ABCC}$ value is misleadingly larger (\ie $0.0724>0.0659$).

\noindent\textbf{Summary.}
Previous expectation-level and distribution-level metrics tend to average the extreme but important local disparity through overall calculation, thus their values cannot accurately measure the fairness violation in certain cases. This enlighten us to develop new fairness metrics to capture the maximal local disparity of classifiers.
\section{The Proposed Metric}

\begin{figure}
    \centering
    \includegraphics[width=0.99\linewidth]{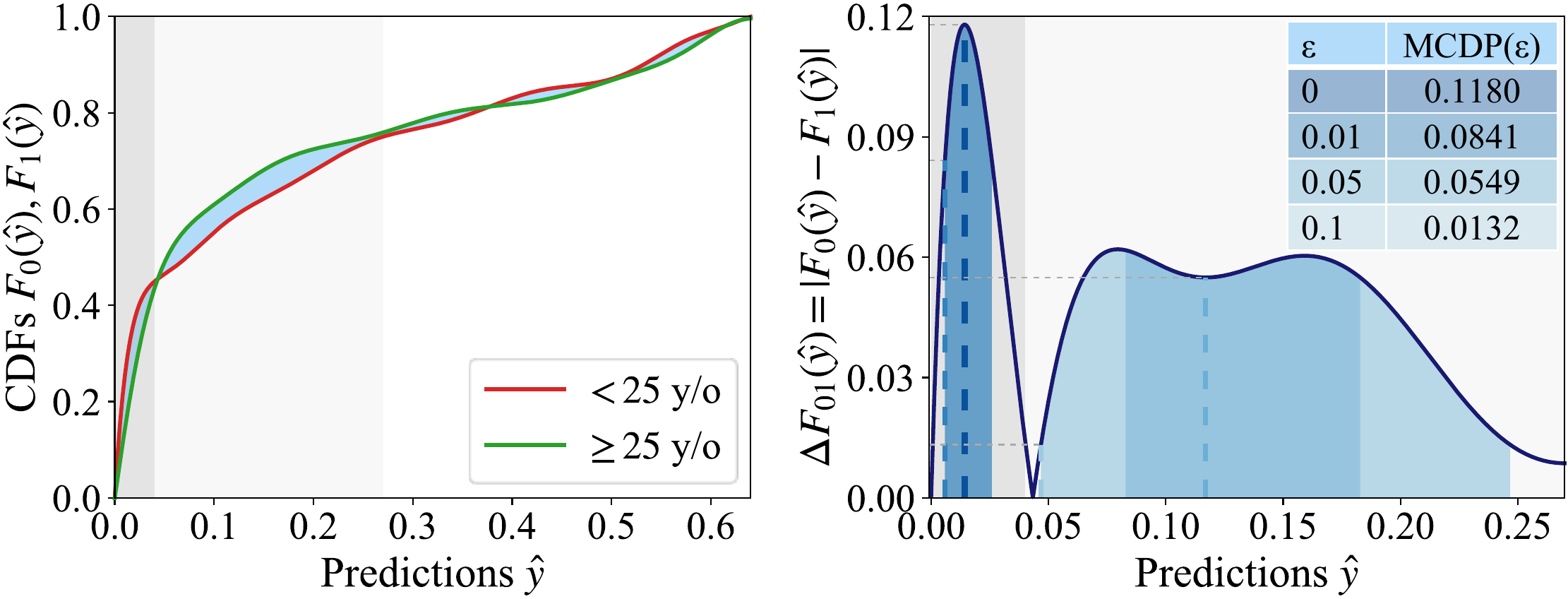}
    \vspace{-12pt}
    \caption{Predictions of FairMixup on the Bank dataset. The neighborhood hyper-parameter $\epsilon$ decides the manner of calculating local disparity, leading to different $\mathrm{MCDP}(\epsilon)$ values.}
    \label{fig:2}
    \vspace{-10pt}
\end{figure}

In this section, we propose a novel fairness metric called $\mathrm{MCDP}(\epsilon)$ to measure the maximal local disparity of classifiers, and provide theoretical properties, estimation algorithms and an optimization framework about the metric. All the proofs of the theorems can be referred in Appendix \ref{sec:proofs}.

\subsection{$\boldsymbol{\mathrm{MCDP}(\epsilon)}$ Metric}
Denote the absolute difference of the two groups' CDFs of model predictions as $\Delta F(\hat{y})=\left|F_0(\hat{y})-F_1(\hat{y})\right|$. To capture the worst case that a classifier violates demographic parity, an intuitive idea is to calculate the \textbf{M}aximal \textbf{C}umulative ratio \textbf{D}isparity along varying \textbf{P}redictions using
\begin{equation}\label{eq:MCDP0}
    \mathrm{MCDP}(f)= \max\limits_{\hat{y}\in[0,1]} \Delta F(\hat{y}).
\end{equation}
As $\mathrm{MCDP}$ only focuses on the maximal disparity, it serves as a more rigorous fairness measurement compared with previous metrics. However, it is susceptible to extremely sharp distributions within certain intervals. As an example shown in Figure \ref{fig:2}, the $\mathrm{MCDP}$ value is large ($0.1180$) due to a large number of instances with predictions around $0.02$, which may misleadingly reflect the unfairness. To address this issue, we introduce a local measurement $\epsilon\geq0$ to smooth $\Delta F(\hat{y})$. For a specific prediction point $y_0$, we take the minimal disparity within its $\epsilon$-neighborhood $[y_0-\epsilon,y_0+\epsilon]$ as the maximal local disparity instead of the exact disparity $\Delta F(y_0)$. In this way, we slightly modify Eq.~\eqref{eq:MCDP0} to compute the \textbf{M}aximal \textbf{C}umulative ratio \textbf{D}isparity along varying \textbf{P}redictions' $\epsilon$-neighborhood as
\begin{equation}\label{eq:MCDPeps}
    \mathrm{MCDP}(f;\epsilon)= \max\limits_{{y}_0\in[0,1]} \min\limits_{\hat{y}:\left|\hat{y}-{y}_0\right|\leq\epsilon} \Delta F(\hat{y}).
\end{equation}
$\mathrm{MCDP}(\epsilon)$ can also be interpreted as the maximum of the minimal CDF disparity of any prediction intervals with length $2\epsilon$ (or $\geq\epsilon$ if endpoints $0$ or $1$ is included). In particular, as the zero length interval degenerates to a point, $\mathrm{MCDP}(0)$ degenerates to $\mathrm{MCDP}$ in Eq.~\eqref{eq:MCDP0}.

\subsection{Theoretical Analysis of $\boldsymbol{\mathrm{MCDP}(\epsilon)}$}
We derive several properties of $\mathrm{MCDP}(\epsilon)$ metric as below:
\begin{theorem}[Properties of $\mathrm{MCDP}(\epsilon)$]\label{theo:prop_MCDP}
The proposed $\mathrm{MCDP}(\epsilon)$ metric has the following desired properties: 
\ding{172} $\mathrm{MCDP}(\epsilon)$ has a range of $[0,1]$. 
\ding{173} $\mathrm{MCDP}(0)=0$ holds if and only if demographic parity is established.
\ding{174} $\mathrm{MCDP}(0)$ is invariant to any monotone and invertible transformation $T:[0,1]\to[0,1]$.
\ding{175} $\mathrm{MCDP}(\epsilon)$ is a monotonically decreasing function \wrt $\epsilon$.
\ding{176} Assume $\Delta F(\hat{y})$ is continuous on $[0,1]$ with Lipschitz constant $L$ \cite{goldstein1977optimization}, then 
\begin{equation*}
    \mathrm{ABCC}\leq
    \begin{cases}
        \mathrm{MCDP}(\epsilon)+\frac{\epsilon L}{2},& \text{if }L\leq2, \\
        \mathrm{MCDP}(\epsilon)+2\epsilon\left(1-\frac{1}{L}\right),& \text{if }L>2.
    \end{cases}
\end{equation*}
\end{theorem}
\textbf{Remarks.} Properties \ding{173} and \ding{174} show that $\mathrm{MCDP}(0)$ satisfies the sufficiency and fidelity criteria for fairness measurement \cite{han2023retiring}. A visualized interpretation of property \ding{175} can be referred in Figure \ref{fig:2}, where wider intervals (\ie larger $2\epsilon$ values) lead to smaller $\mathrm{MCDP}(\epsilon)$ values. Property \ding{176} provides an upper bound of the $\mathrm{ABCC}$ metric with given $\mathrm{MCDP}(\epsilon)$ values, which also suggests that the distribution-level disparity can be controlled by minimizing the maximal local disparity.

\subsection{Estimating the $\boldsymbol{\mathrm{MCDP}(\epsilon)}$ Metric}\label{sec:estimate}

\begin{algorithm}[t]
\caption{Exact Calculation of $\widehat{\mathrm{MCDP}}(\epsilon)$}
\label{alg:exact}
    \begin{algorithmic}[1]
    \INPUT Dataset $\mathcal{D}=\{\boldsymbol{x}_i,y_i,s_i\}_{i=1}^N$, model predictions $\{\hat{y}_i\}_{i=1}^N$, local measurement $\epsilon\ge0$.
    \STATE Calculate $\Delta\hat{F}(\hat{y})$ using $\{\hat{y}_i,s_i\}_{i=1}^N$ by Eq.~\eqref{eq:ecdf};
    \STATE Set $\hat{y}_0=0,\ \hat{y}_{N+1}=1$;
    \IF{$\epsilon=0$}
    \STATE $\widetilde{\mathrm{MCDP}}(\epsilon)= \max_i\Delta\hat{F}(\hat{y}_i),\ i\in[0,N]$;
    \ELSE
    \STATE $\widetilde{\mathrm{MCDP}}(\epsilon)= \min_{i:\hat{y}_i\leq\epsilon}\Delta\hat{F}(\hat{y}_i)$;
    \STATE $\mathcal{I}= \{i\in[0,N+1]: \hat{y}_i\leq1-\epsilon\}$;
    \FOR{$i\in\mathcal{I}$}
    \STATE $\mathcal{J}= \{j\in[0,N+1]: \hat{y}_j\in[\hat{y}_i,\hat{y}_i+2\epsilon]\}$;
    \STATE $M=\min_j \Delta\hat{F}(\hat{y}_j),\ j\in\mathcal{J}$;
    \STATE $\widetilde{\mathrm{MCDP}}(\epsilon)= \max\{\widetilde{\mathrm{MCDP}}(\epsilon),M\}$;
    \ENDFOR
    \ENDIF
    \OUTPUT $\widetilde{\mathrm{MCDP}}(\epsilon)$.
    \end{algorithmic}
\end{algorithm}

In real-world settings where algorithmic fairness is evaluated on finite data samples, we adopt the empirical distribution function as the estimated CDF. Formally, the empirical $\mathrm{MCDP}(\epsilon)$ metric estimated over $\mathcal{D}$ can be written as
\begin{equation}\label{eq:mcdphat}
\begin{aligned}
    \widehat{\mathrm{MCDP}}(f,\mathcal{D};\epsilon)&= \max\limits_{{y}_0\in[0,1]} \min\limits_{\hat{y}:\left|\hat{y}-y_0\right|\leq\epsilon}\Delta\hat{F}(\hat{y}),
\end{aligned}
\end{equation}
where $\Delta\hat{F}(\hat{y})= |\hat{F}_0(\hat{y})- \hat{F}_1(\hat{y})|$. By the Glivenko-Cantelli theorem \cite{tucker1959generalization}, the estimated metric value above converges to the true metric value in Eq.~\eqref{eq:MCDPeps} almost surely with increasing sample size $N$ (see Appendix \ref{sec:tractability} for formal proofs), which demonstrates that $\widehat{\mathrm{MCDP}}(\epsilon)$ serves as a preferable estimation of $\mathrm{MCDP}(\epsilon)$.

However, it is intractable to traverse all possible $y_0$ values in $[0,1]$ and $\hat{y}$ values in $y_0$'s $\epsilon$-neighborhood, which poses a challenge to compute the metric in Eq.~\eqref{eq:mcdphat}. Nevertheless, thanks to the step-like pattern's property of the empirical distribution function, we can traverse finite $y_0$ and $\hat{y}$ values. Based on different traversing strategies, we develop two algorithms which calculates the exact and approximate value of $\widehat{\mathrm{MCDP}}(\epsilon)$, respectively. 
The exact algorithm only traverses predictions of instances in $\mathcal{D}$ (lines 7-10 in Algorithm \ref{alg:exact}). In contrast, the approximate algorithm firstly samples prediction points that are equally spaced by $\frac{\epsilon}{K}$ on $[0,1]$, where $K\in\mathbb{N}^+$ is a pre-defined hyper-parameter to control the sampling frequency. Afterwards, it traverses consecutive $2K$ sampled points to estimate the maximal local disparity (lines 4-5 in Algorithm \ref{alg:approximate}). Notably, the two algorithms have the following properties:

\begin{theorem}[Exactness]\label{theo:exactalg}
    The $\widetilde{\mathrm{MCDP}}(\epsilon)$ value returned by Algorithm \ref{alg:exact} equals to the true value in Eq.~\eqref{eq:mcdphat}, \ie it calculates $\widehat{\mathrm{MCDP}}(\epsilon)$ exactly.
\end{theorem}
\vspace{-2pt}

\begin{theorem}[Over-estimation]\label{theo:approover}
    The $\widetilde{\mathrm{MCDP}}(\epsilon)$ value returned by Algorithm \ref{alg:approximate} never underestimates the true metric value, \ie $\widetilde{\mathrm{MCDP}}(\epsilon)\geq\widehat{\mathrm{MCDP}}(\epsilon)$.
\end{theorem}
\vspace{-2pt}

\begin{theorem}[Monotonicity \wrt sampling frequency]\label{theo:appromono}
     Denote $\widetilde{\mathrm{MCDP}}(\epsilon;K)$ as the $\widehat{\mathrm{MCDP}}(\epsilon)$ value returned by Algorithm \ref{alg:approximate} with sampling frequency $K$. For any $p>q\ge0,\ p,q\in\mathbb{N}$, we have $\widetilde{\mathrm{MCDP}}(\epsilon;2^p)\le\widetilde{\mathrm{MCDP}}(\epsilon;2^q)$.
\end{theorem}
\vspace{-2pt}

\begin{algorithm}[t]
\caption{Approximate Calculation of $\widehat{\mathrm{MCDP}}(\epsilon)$}
\label{alg:approximate}
    \begin{algorithmic}[1]
    \INPUT Dataset $\mathcal{D}=\{\boldsymbol{x}_i,y_i,s_i\}_{i=1}^N$, model predictions $\{\hat{y}_i\}_{i=1}^N$, local measurement $\epsilon>0$, sampling frequency $K\in\mathbb{N}_+$.
    \STATE Calculate $\Delta\hat{F}(\hat{y})$ using $\{\hat{y}_i,s_i\}_{i=1}^N$ by Eq.~\eqref{eq:ecdf}; 
    \STATE Set the step-size $\delta=\frac{\epsilon}{K}$;
    \STATE $\widetilde{\mathrm{MCDP}}(\epsilon)= \min_j\Delta\hat{F}(j\delta),\ j\in\{0,\cdots,K\}$;
    \FOR{$j=1,\cdots,\lceil\frac{1}{\delta}\rceil-2K$}
    \STATE $M=\min_k \Delta\hat{F}(k\delta),\ k\in\{j,\cdots,j+2K-1\}$; 
    \STATE $\widetilde{\mathrm{MCDP}}(\epsilon)= \max\{\widetilde{\mathrm{MCDP}}(\epsilon),M\}$;
    \ENDFOR
    \OUTPUT $\widetilde{\mathrm{MCDP}}(\epsilon;K)$.
    \end{algorithmic}
\end{algorithm}

\textbf{Computational Complexity. }
The calculation process of Algorithms \ref{alg:exact} and \ref{alg:approximate} are mainly based on specific traverse strategies on $y_0$ and $\hat{y}$ in Eq.~\eqref{eq:mcdphat}. As the exact algorithm needs to traverse instances in $\mathcal{D}$ twice (lines 7-9 in Algorithm \ref{alg:exact}), its computational complexity is $\mathcal{O}(N^2)$. As to the approximate algorithm, the complexity of traversing sampled prediction points (line 4 in Algorithm \ref{alg:approximate}) is $\mathcal{O}(\frac{K}{\epsilon})$, and the complexity of finding the minimal CDF disparity in each consecutive $2K$ points is $\mathcal{O}(K)$. Therefore, the overall computational complexity is $\mathcal{O}(\frac{K^2}{\epsilon})$. More detailed analysis can be referred in Appendix \ref{sec:complexity}. In practice, as $K$ can be set with very small values (\ie $K\ll N$), the computational complexity of the approximate algorithm is greatly reduced compared to the exact algorithm.

\textbf{More Discussions about $K$. }
As discussed before, increasing $K$ values would boost the computational complexity of Algorithm \ref{alg:approximate}. On the flip side, according to Theorems \ref{theo:approover} and \ref{theo:appromono}, the estimation error would keep decreasing as $K$ increases, indicating a trade-off between efficiency and accuracy. In practice, both estimation accuracy and efficiency can achieve promising results with varying $K$ values.

\subsection{DiffMCDP: Bi-Level Optimization Algorithm}
\begin{algorithm}[t]
\caption{DiffMCDP: Bi-Level Optimization  Algorithm}
\label{alg:learning}
    \begin{algorithmic}[1]
    \INPUT Training data $\mathcal{D}=\{\boldsymbol{x}_i,y_i,s_i\}_{i=1}^N$, classifier $f$ parameterized by $\boldsymbol{\theta}$, regularization strength $\lambda$, temperature $\tau$, epoch number $E$, batch size $B$, learning rate $\eta$.
    \STATE Randomly initialize $\boldsymbol{\theta}\leftarrow\boldsymbol{\theta}_0$;
    \FOR{$e=0,\cdots,E-1$}
    \STATE Draw a mini-batch $\mathcal{D}_e=\{\boldsymbol{x}_i^{\prime},y_i^{\prime},s_i^{\prime}\}_{i=1}^B$ from $\mathcal{D}$;
    \STATE \emph{// The following calculations are based on $\mathcal{D}_e$}
    \STATE $\hat{y}_{i}^{\prime}=f(\boldsymbol{x}_i^{\prime}),\ i=1,\cdots,B$;
    \STATE $\hat{y}^*= \mathop{\arg\max}_{\hat{y}\in[0,1]} \Delta\Tilde{F}_{\tau}(\hat{y})$;
    \STATE $\mathcal{L}= \frac{1}{B}\sum_{i=1}^{B} \ell(\boldsymbol{\theta};\hat{y}_i^{\prime},y_i^{\prime}) +\lambda\Delta\Tilde{F}_{\tau}(\hat{y}^*)$;
    \STATE $\boldsymbol{\theta}_{e+1}= \boldsymbol{\theta}_e -\eta\nabla_{\boldsymbol{\theta}}\mathcal{L}$;
    \ENDFOR
    \OUTPUT A fair classifier $f$ parameterized by $\boldsymbol{\theta}_E$.
    \end{algorithmic}
\end{algorithm}

Based on previous analysis, it is essential to train a fair classifier which minimizes the maximal local disparity to approach demographic parity. According to property \ding{176} in Theorem \ref{theo:prop_MCDP}, we can minimize $\mathrm{MCDP}(0)$ as an upper bound of $\mathrm{MCDP}(\epsilon)$ for any $\epsilon>0$. A natural idea is to impose the $\widehat{\mathrm{MCDP}}(0)$ metric as a regularization term on the classification loss. However, as the empirical distribution functions are not differentiable \wrt model parameter $\boldsymbol{\theta}$, directly regularizing $\widehat{\mathrm{MCDP}}(0)$ is implausible. To address this issue, we firstly estimate $\Delta\hat{F}(\hat{y})$ in a differentiable way 
\begin{equation*}\label{eq:tempsigmoid}
    \Delta\Tilde{F}_{\tau}(\hat{y})= \left|\frac{1}{|\mathcal{S}_0|} \sum\limits_{i\in\mathcal{S}_0} \sigma_{\tau}(\hat{y}-\hat{y}_i)- \frac{1}{|\mathcal{S}_1|} \sum\limits_{i\in\mathcal{S}_1} \sigma_{\tau}(\hat{y}-\hat{y}_i)\right|,
\end{equation*}
where $\sigma_{\tau}(x)= \frac{1}{1+\exp(-\tau x)}$ is a variant of sigmoid function with temperature $\tau>0$ as a hyper-parameter \cite{han1995influence}. Notably, when the temperature tends to infinity, we have $\Delta\Tilde{F}_{\tau}(\hat{y})$ converges to the $\Delta\hat{F}(\hat{y})$ as follows.
\begin{theorem}\label{theo:tempsigmoid}
    $\Delta\Tilde{F}_{\tau}(\hat{y}) \overset{\text{a.e.}}{\longrightarrow} \Delta\hat{F}(\hat{y})$ as $\tau\to\infty$.
\end{theorem}

With the differentiable estimation above, the fairness-regularized objective function can be written as
\begin{equation}\label{eq:regobj}
    \min\limits_{\boldsymbol{\theta}} \left(\frac{1}{|\mathcal{D}|}\sum\limits_{i=1}^{N} \ell(\boldsymbol{\theta};\hat{y}_i,y_i) +\lambda\cdot \max\limits_{\hat{y}\in[0,1]} \Delta\Tilde{F}_{\tau}(\hat{y})\right),
\end{equation} 
where $\ell(\boldsymbol{\theta};\hat{y}_i,y_i)$ denotes the classification loss for $\boldsymbol{x}_i$, and $\lambda>0$ controls the trade-off between accuracy and fairness. To solve Eq.~\eqref{eq:regobj}, we adopt the bi-level optimization approach -- firstly find the prediction which achieves the maximal CDF disparity as $\hat{y}^*= \mathop{\arg\max}_{\hat{y}\in[0,1]} \Delta\Tilde{F}_{\tau}(\hat{y})$, and then find the optimal model parameter by $\boldsymbol{\theta}^*= \mathop{\arg\min}_{\boldsymbol{\theta}} \frac{1}{|\mathcal{D}|}\sum_{i=1}^{N} \ell(\boldsymbol{\theta};\hat{y}_i,y_i) +\lambda\Delta\Tilde{F}_{\tau}(\hat{y}^*)$. We provide the detailed learning algorithm in Algorithm \ref{alg:learning}.

\section{Experiments}

\begin{figure*}[t]
    \centering
    \vspace{-5pt}
    \centering
    \includegraphics[width=0.99\linewidth]{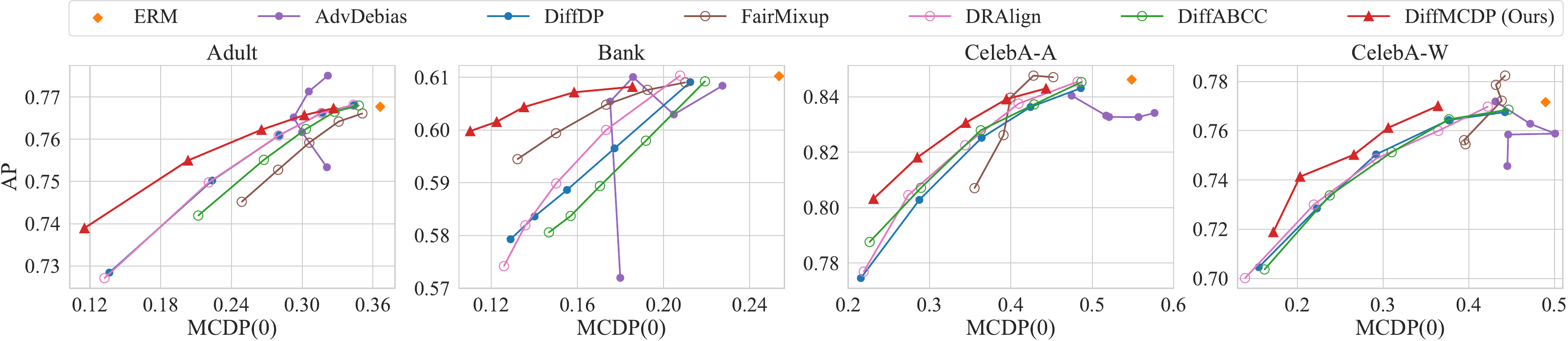}
    \vspace{-6pt}
    \caption{Trade-offs between $\mathrm{AP}$ and $\mathrm{MCDP}(0)$ of baselines and the proposed method. Each marker represents the average testing result of 5 runs with a specific fairness-accuracy trade-off coefficient. The curves closer to the upper-left corners indicate better performances.}
    \label{fig:ap_mcdp0}
    \vspace{-10pt}
\end{figure*}

\textbf{Datasets and Backbone Models. }
In our experiments, we adopt two tabular datasets and one image dataset for evaluation: \ding{172} \textbf{Adult} \cite{kohavi1996scaling} is a popular UCI dataset which contains personal information of over 40K individuals from US 1994 census data. The task is to predict whether a person's annual income is over \$50K or not, and we treat \emph{gender} as the sensitive attribute. \ding{173} \textbf{Bank} \cite{moro2014data} dataset is collected from a Portuguese banking institution's marketing campaigns, and its goal is to predict whether a client will make a deposit subscription or not. We take \emph{age} as the sensitive attribute (whether the age is over 25 or not). \ding{174} \textbf{CelebA} \cite{liu2015deep} dataset contains over 20K face images of celebrities, where each image has 40 human-labeled binary face attributes. We use \emph{gender} as the sensitive attribute, and choose \emph{attractive face} and \emph{wavy hair} as target attributes to create two meta-datasets, denoted as \textbf{CelebA-A} and \textbf{CelebA-W} respectively. For tabular datasets, we adopt a two-layer multi-layer perceptron as the backbone model. For CelebA-A and CelebA-W, we use ResNet-18 \cite{he2016deep} initialized with pretrained weights.

\textbf{Baselines and Evaluation Protocols. }
We compare our proposed method (denoted as \textbf{DiffMCDP}) with the following baselines: \ding{172} \textbf{ERM} trains the model with the vanilla classification loss without fairness objectives. \ding{173} \textbf{AdvDebias} \cite{zhang2018mitigating} minimizes the adversary's ability of inferring sensitive attributes from model representations. \ding{174} \textbf{DiffDP} imposes the $\Delta\mathrm{DP}$ metric as a regularization term on the empirical risk. \ding{175} \textbf{FairMixup} \cite{chuang2021fair} regularizes models on interpolated distributions between two groups. \ding{176} \textbf{DRAlign} \cite{li2023fairer} uses gradient-guided parity alignment to encourage gradient-weighted consistency of neurons across groups. \ding{177} \textbf{DiffABCC} regularizes the $\mathrm{ABCC}$ metric on the empirical risk. We use average precision ($\mathrm{AP}$) to evaluate the classification accuracy, and use $\Delta{\mathrm{DP}}$, ${\mathrm{ABCC}}$, and ${\mathrm{MCDP}}(\epsilon)$ to measure the algorithmic fairness. More implementation details can be referred in Appendix \ref{sec:moresetup}.

\subsection{Performance Comparison}\label{sec:exp_learning} 

Figure \ref{fig:ap_mcdp0} shows the trade-off relationships between $\mathrm{AP}$ and $\mathrm{MCDP}(0)$ of baselines and our proposed method. We can observe that DiffMCDP consistently outperforms other baselines in terms of fairness-accuracy trade-offs across all datasets, which suggests that our proposed fair training algorithm can effectively reduce maximal local disparity to improve fairness. In addition, traditional fair training algorithms can achieve desired trade-offs between accuracy and our proposed fairness metric. For example, the $\mathrm{MCDP}(0)$ values of both DiffDP and DiffABCC decrease with increasing regularization strengths, which demonstrates that solely optimizing expectation-level or distribution-level metrics do contribute to minimizing the maximal local disparity. Nevertheless, they obtain inferior performance compared to DiffMCDP which directly regularizes $\mathrm{MCDP}(0)$, indicating their limitation in approaching demographic parity.

\begin{figure}[t]
    \centering
    \includegraphics[width=0.99\linewidth]{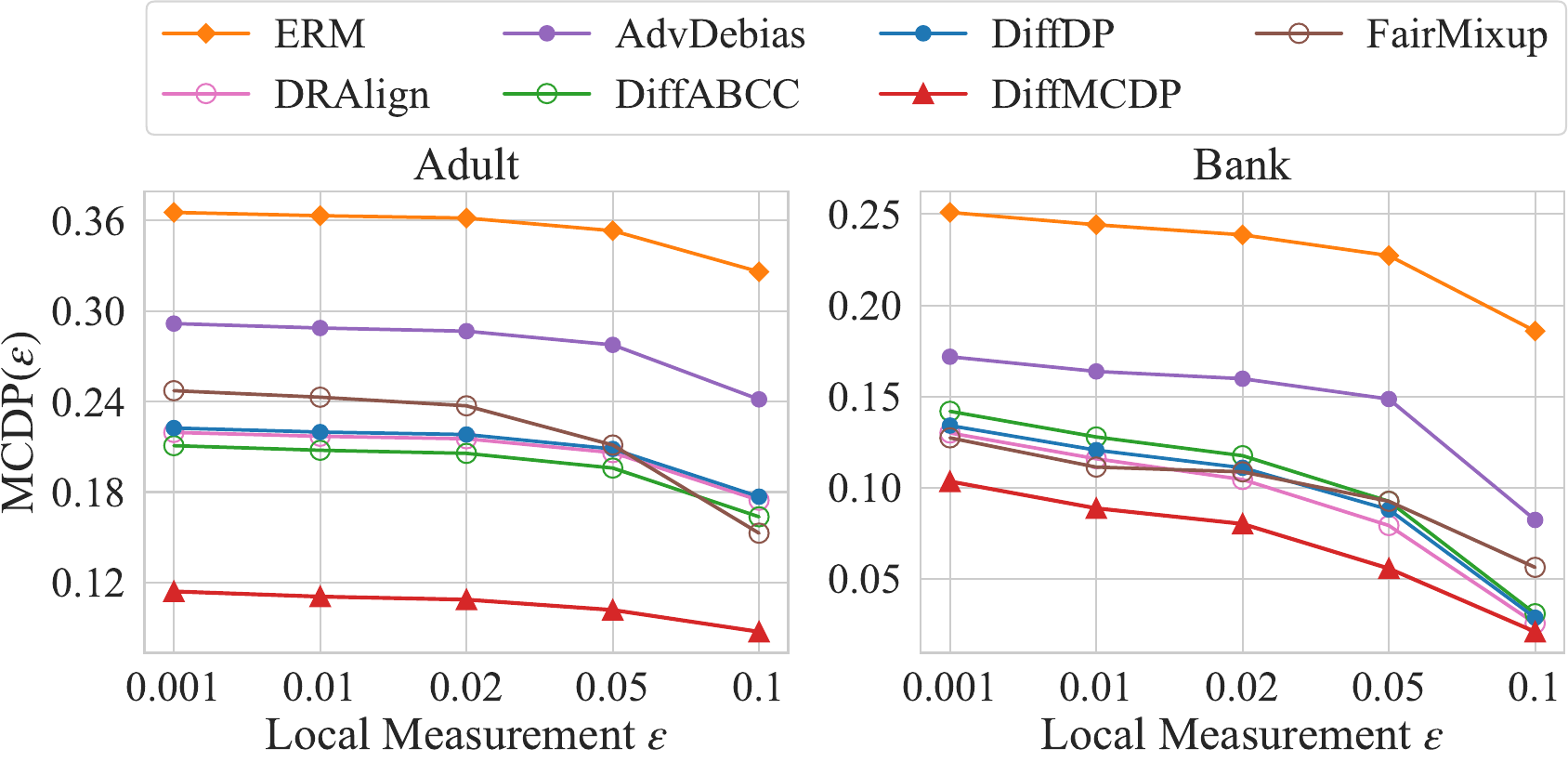}
    \vspace{-6pt}
    \centering
    \caption{Comparison of $\mathrm{MCDP}(\epsilon)$ results with varying $\epsilon$. The metric values are calculated by the exact algorithm (Algorithm \ref{alg:exact}).}
    \label{fig:MCDP_eps}
    \vspace{-15pt}
\end{figure}

To explore the performance of our method under different fairness measurements, we report the values of $\Delta\mathrm{DP}$, $\mathrm{ABCC}$ and $\mathrm{MCDP}(0)$ metrics in Tables \ref{tab:main_tabular} and \ref{tab:main_image} (the trade-off curves can be referred in Appendix \ref{sec:moretradeoff}). While DiffMCDP achieves the optimal $\mathrm{MCDP}(0)$ results across various datasets, it also obtains comparable even superior $\Delta\mathrm{DP}$ or $\mathrm{ABCC}$ values compared to other baselines. For example, DiffMCDP achieves the lowest $\mathrm{ABCC}$ values in tabular datasets, and its $\Delta\mathrm{DP}$ values are optimal in two image datasets. These results illustrate that optimizing the maximal local disparity is also beneficial to improve both expectation-level and distribution-level fairness metrics.

Lastly, we plot the changes of $\mathrm{MCDP}(\epsilon)$ as varying local measurements $\epsilon$ in Figure \ref{fig:MCDP_eps} (the results of image datasets are deferred in Appendix \ref{sec:varyeps}). From the figure, we can observe that DiffMCDP still consistently outperforms other baselines in terms of different variants of the metric. This validates the effectiveness of regularizing the upper bound of $\mathrm{MCDP}(\epsilon)$ by $\mathrm{MCDP}(0)$ in the training objective, demonstrating that our proposed framework is applicable to various scenarios with varying $\epsilon$ values. Additionally, it is noteworthy that the relative performance of other baselines changes with increasing $\epsilon$ values (\eg there exists many intersection lines between 0.05 and 0.1 in Adult). This indicates that the relative fairness performance of various algorithms may change based on different manners of calculating local disparity, thus it's essential to select varying $\epsilon$ values in $\mathrm{MCDP}(\epsilon)$ for more comprehensive evaluation.

\begin{figure}[t]
    \centering
    \subcaptionbox{Estimation Accuracy Comparison\label{fig:err_epsK}}{
        \centering
        \includegraphics[width=0.968\linewidth]{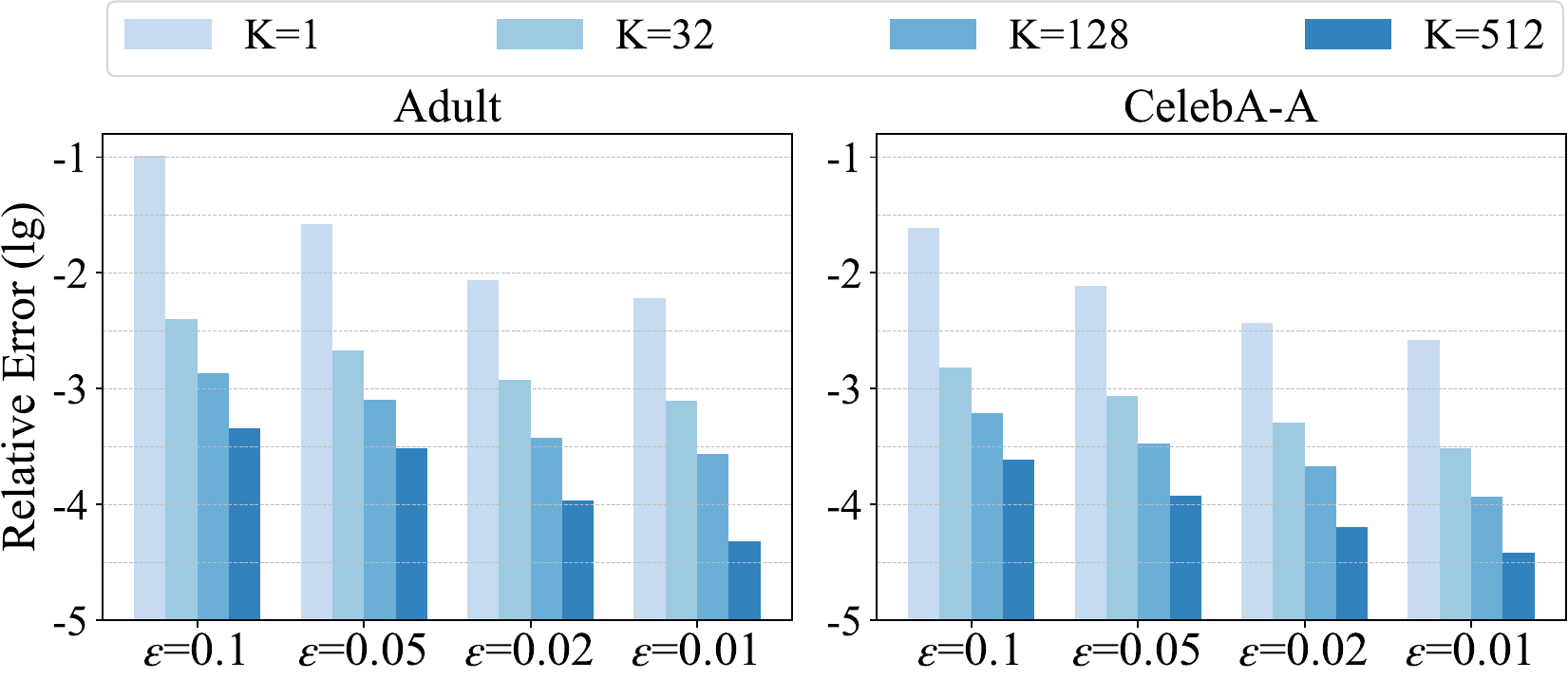}}
    \subcaptionbox{Calculation Efficiency Comparison}{
        \includegraphics[width=0.99\linewidth]{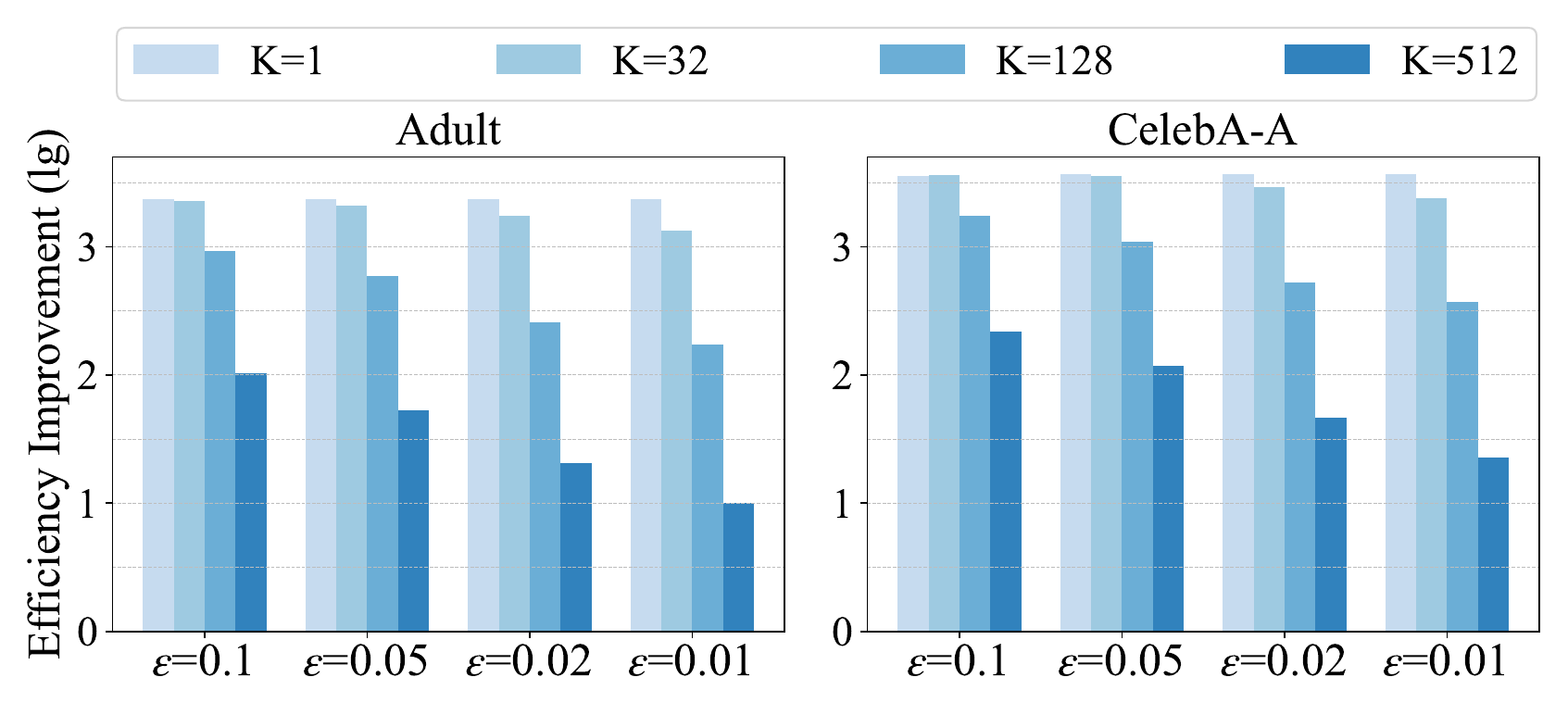}}
    \vspace{-6pt}
    \centering
    \centering
    \caption{Varying $K$ and $\epsilon$ in $\widehat{\mathrm{MCDP}}(\epsilon)$ calculation algorithms.}
    \centering%
    \label{fig:time_err}
    \vspace{-15pt}
\end{figure}

\begin{table*}[t]
    \vspace{-8pt}
    \caption{The performance \wrt different fairness metrics of baselines and our proposed DiffMCDP algorithm on tabular datasets. Following previous work \cite{jung2023re}, we select the fairest model which achieves at least 95\% of the vanilla model's accuracy (\ie $\mathrm{AP}$ of ERM) on validation set, and report the average and standard deviation of metric values of 5 independent runs with different seeds on testing set. The optimal and sub-optimal results are highlighted with \textbf{bold} and \underline{underline}, respectively.}
    \label{tab:main_tabular}
    \vspace{1pt}
    \centering
    \resizebox{0.95\textwidth}{!}{\begin{tabular}{l|rrrl|rrrl}
    \toprule
    \multicolumn{1}{c}{} & \multicolumn{4}{c}{Adult} & \multicolumn{4}{c}{Bank} \\
    \cmidrule{2-9}
    \multicolumn{1}{c}{} & \multicolumn{1}{c}{$\mathrm{AP}\uparrow$} & \multicolumn{1}{c}{$\Delta\mathrm{DP}\downarrow$} & \multicolumn{1}{c}{$\mathrm{ABCC}\downarrow$} & \multicolumn{1}{c}{$\mathrm{MCDP}(0)\downarrow$} & \multicolumn{1}{c}{$\mathrm{AP}\uparrow$} & \multicolumn{1}{c}{$\Delta\mathrm{DP}\downarrow$} & \multicolumn{1}{c}{$\mathrm{ABCC}\downarrow$} & \multicolumn{1}{c}{$\mathrm{MCDP}(0)\downarrow$} \\
    \midrule
    ERM & \bms{76.77}{0.54} & \ms{19.67}{0.94} & \ms{18.89}{0.29} & \ms{36.62}{0.53} & \bms{61.02}{1.12} & \ms{9.85}{1.28} & \ms{9.83}{0.98} & \ms{25.36}{1.81} \\ \midrule
    AdvDebias & \ums{76.52}{0.62} & \ms{12.73}{1.81} & \ms{12.95}{1.35} & \ms{29.28}{2.26} & \ums{60.54}{1.02} & \ms{1.95}{2.76} & \ms{3.95}{0.97} & \ms{17.52}{4.12} \\
    DiffDP & \ms{75.02}{0.28} & \ms{7.77}{1.20} & \ms{8.97}{0.73} & \ms{22.37}{1.11} & \ms{58.36}{1.67} & \ums{1.03}{0.23} & \ms{2.13}{0.20} & \ms{14.01}{1.34} \\
    FairMixup & \ms{74.52}{0.40} & \bms{3.87}{1.20} & \ums{7.34}{0.62} & \ms{24.87}{1.53} & \ms{59.45}{1.94} & \ms{1.28}{0.44} & \ms{2.95}{0.91} & \ums{13.22}{2.91} \\
    DRAlign & \ms{74.98}{0.28} & \ms{7.64}{1.22} & \ms{8.86}{0.75} & \ms{22.08}{1.19} & \ms{58.19}{1.54} & \ms{1.15}{0.41} & \ums{1.98}{0.23} & \ms{13.58}{1.26} \\
    DiffABCC & \ms{74.19}{0.24} & \ums{5.90}{1.28} & \ms{7.83}{0.74} & \ums{21.18}{1.23} & \ms{58.06}{1.60} & \bms{0.99}{0.44} & \ms{2.23}{0.19} & \ms{14.67}{1.35} \\ \midrule
    DiffMCDP & \ms{73.90}{0.29} & \ms{6.63}{0.85} & \bms{6.09}{0.59} & \bms{11.53}{1.12} & \ms{59.98}{1.80} & \ms{2.13}{0.72} & \bms{1.83}{0.28} & \bms{11.02}{1.00} \\ 
    \bottomrule
    \end{tabular}}
    \vspace{-10pt}
\end{table*}

\begin{table*}[t]    
  \centering
  \vspace{-5pt}
  \caption{Performance comparison on the CelebA-A and CelebA-W datasets. All details are the same as Table \ref{tab:main_tabular}.}
  \label{tab:main_image}
  \vspace{1pt}
  \centering
  \resizebox{0.95\textwidth}{!}{\begin{tabular}{l|rrrl|rrrl}
  \toprule
  \multicolumn{1}{c}{} & \multicolumn{4}{c}{CelebA-A} & \multicolumn{4}{c}{CelebA-W} \\
  \cmidrule{2-9}
  \multicolumn{1}{c}{} & \multicolumn{1}{c}{$\mathrm{AP}\uparrow$} & \multicolumn{1}{c}{$\Delta\mathrm{DP}\downarrow$} & \multicolumn{1}{c}{$\mathrm{ABCC}\downarrow$} & \multicolumn{1}{c}{$\mathrm{MCDP}(0)\downarrow$} & \multicolumn{1}{c}{$\mathrm{AP}\uparrow$} & \multicolumn{1}{c}{$\Delta\mathrm{DP}\downarrow$} & \multicolumn{1}{c}{$\mathrm{ABCC}\downarrow$} & \multicolumn{1}{c}{$\mathrm{MCDP}(0)\downarrow$} \\
  \midrule
  ERM & \bms{84.62}{0.38} & \ms{54.32}{0.77} & \ms{37.20}{0.69} & \ms{54.85}{0.92} & \ums{77.16}{0.68} & \ms{34.02}{2.24} & \ms{28.94}{1.36} & \ms{48.96}{2.16} \\ \midrule
  AdvDebias & \ums{84.04}{0.51} & \ms{46.90}{1.89} & \ms{30.72}{2.05} & \ms{47.45}{1.83} & \bms{77.19}{0.47} & \ms{31.14}{1.44} & \ms{25.27}{2.48} & \ms{43.11}{5.20} \\
  DiffDP & \ms{82.52}{0.42} & \ms{36.03}{1.77} & \ms{21.33}{1.40} & \ms{36.40}{1.61} & \ms{75.04}{1.60} & \ms{23.01}{0.95} & \ms{16.99}{1.29} & \ms{29.17}{3.00} \\
  FairMixup & \ms{80.70}{1.03} & \ms{34.64}{3.23} & \bms{14.15}{1.63} & \ms{35.52}{2.67} & \ms{75.45}{0.72} & \ms{26.49}{4.48} & \ms{17.96}{1.04} & \ms{39.61}{3.01} \\
  DRAlign & \ms{80.45}{1.05} & \ums{26.36}{2.48} & \ums{15.25}{1.71} & \ums{27.38}{2.46} & \ms{74.86}{1.49} & \ms{22.20}{1.54} & \ms{16.95}{1.11} & \ms{29.16}{2.76} \\
  DiffABCC & \ms{80.71}{0.86} & \ms{28.05}{1.83} & \ms{15.70}{1.67} & \ms{28.94}{1.89} & \ms{73.38}{1.69} & \ums{19.93}{1.85} & \bms{13.67}{0.88} & \ums{23.77}{3.44} \\ \midrule
  DiffMCDP & \ms{80.32}{0.67} & \bms{21.10}{1.94} & \ms{16.03}{1.61} & \bms{23.10}{1.66} & \ms{74.14}{1.26} & \bms{19.58}{1.91} & \ums{15.35}{1.26} & \bms{20.29}{1.78} \\ 
  \bottomrule
  \end{tabular}}
  \vspace{-6pt}
\end{table*}

\subsection{Exact and Approximate Calculation}\label{sec:expestimate}
To explore the estimation accuracy and efficiency of the approximate algorithm on real-world data, we run both the exact and approximate algorithms with varying $K$ and $\epsilon$ values to compute $\widehat{\mathrm{MCDP}}(\epsilon)$ of the testing results in Figure \ref{fig:ap_mcdp0}. Then we calculate the \emph{relative error} by $(V_a-V_e)/V_e$, where $V_e$ and $V_a$ are the estimated $\widehat{\mathrm{MCDP}}(\epsilon)$ metric value returned by the exact and approximate algorithms, respectively. Moreover, we measure the \emph{efficiency improvement} by $T_e/T_a$, where $T_e$ and $T_a$ represents the time of a single run of the exact and approximate algorithms.

We plot the results across different settings in Figure \ref{fig:time_err} (results on more datasets are deferred in Appendix \ref{sec:detailestimate}), where we have the following observations. \ding{172} With increasing $K$ values, the relative error decreases as Theorems \ref{theo:approover} and \ref{theo:appromono} suggests. Meanwhile, the calculation time also increases, which is consistent with the positive correlation between computational complexity and sampling frequency. It is noteworthy that selecting moderate $K$ values can achieve promising performance. For example, when $K=32$, the average estimation error and efficiency improvement  are $0.03\%-3\%$ and over $1000$ times, respectively. \ding{173} As $\epsilon$ increases, the calculation efficiency improvement over the exact algorithm keeps increasing, which can be explained by the inverse correlation between computational complexity and $\epsilon$. Meantime, the relative error also increases, and a potential reason is that estimating the minimal CDF value with $2K$ sampled prediction points within a larger interval tends to be more inaccurate. \ding{174} The efficiency improvements are more obvious in CelebA-A dataset compared with Adult dataset, which is attributed to that the number of instances in CelebA-A is larger than Adult, and the computational complexity of the exact algorithm is quadratically proportional to the sample size. This indicates that the efficiency superiority of approximation algorithm becomes more evident when adopted to predictions on more instances.

\subsection{In-Depth Analysis}\label{sec:indepth}

\begin{figure}[t]
    \centering
\includegraphics[width=0.99\linewidth]{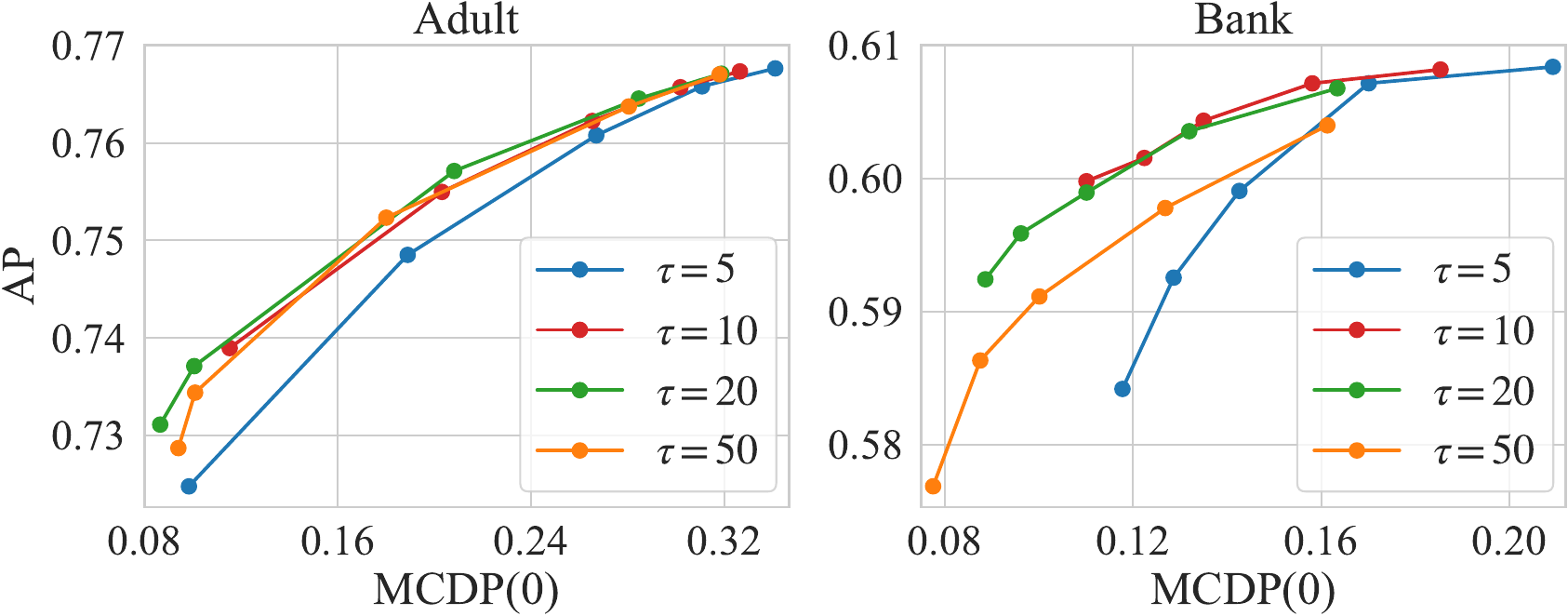}
    \vspace{-6pt}
    \centering
    \caption{Trade-offs with varying temperature $\tau$.}
    \label{fig:temperature}
    \vspace{-15pt}
\end{figure}

\textbf{Effect of Temperature $\tau$. }The temperature $\tau$ in Algorithm \ref{alg:learning} is crucial to approximate the true maximal CDF disparity. On the one hand, as shown in Theorem \ref{theo:tempsigmoid}, the estimation error vanishes when the temperature $\tau$ grows. On the other hand, in practice, a large $\tau$ value may arise the gradient vanishing problem \cite{roodschild2020new}, which limits the learning capacity and optimization convergence. To verify this point, we tune $\tau$ in $\{5,10,20,50\}$ and plot the fairness-accuracy trade-off curves as shown in Figure \ref{fig:temperature} (the results of image datasets are deferred in Appendix \ref{sec:varytau}). We find that very small or large temperatures ($\tau=5,50$) may lead to dissatisfactory results, thus it's better to adopt moderate temperatures ($\tau=10, 20$) to effectively trade-off the estimation accuracy and the gradient magnitude.

\textbf{Visualizations of Model Predictions. }Figure \ref{fig:vis} visualizes the empirical distribution functions of model predictions for two age groups in Bank dataset, where the training methods include ERM, DiffDP and DiffMCDP. As ERM optimizes the vanilla classification loss without fairness objectives, its prediction distribution gap is the most evident. Taking a step further, DiffDP regularizes the $\Delta\widehat{\mathrm{DP}}$ metric on empirical risk, thus the average prediction gap is much smaller than ERM. Nevertheless, it leaves a lot to be desired in terms of reducing the maximal local disparity. By contrast, our proposed DiffMCDP method is much more effective in narrowing the maximal gap between prediction distributions, \ie the $\mathrm{MCDP}(0)$ value of DiffMCDP declined by $47.98\%$ compared to DiffDP, which is also beneficial for minimizing the difference of average predictions.

\begin{figure}[t]
    \centering
    \includegraphics[width=0.99\linewidth]{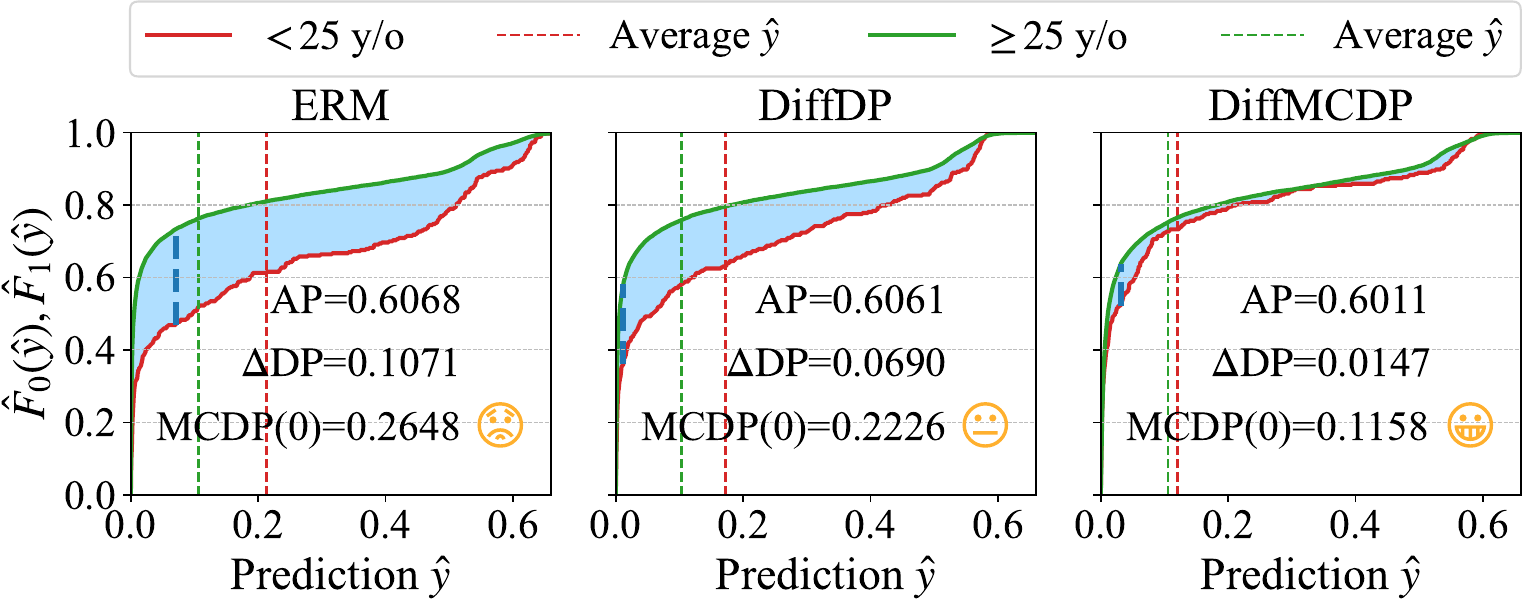}
    \vspace{-6pt}
    \centering
    \caption{Visualizations of empirical distribution functions.}
    \label{fig:vis}
    \vspace{-15pt}
\end{figure}
\section{Related Work}

\textbf{Fairness Notions and Metrics.} 
While machine learning algorithms are widely-used in high-stake applications with broad social impacts, algorithmic fairness has been a crucial requirement in accessing and regulating the models \cite{yeung2018algorithmic,chen2019fairness}. Generally, algorithmic fairness notions can be categorized into group fairness \cite{feldman2015certifying,shui2022learning,sun2023fair,jin2024fairly}, individual fairness \cite{dwork2012fairness,biega2018equity,li2023learning,wicker2023certification} and counterfactual fairness \cite{kusner2017counterfactual,chiappa2019path,wu2019counterfactual,ma2022learning,han2023achieving,rosenblatt2023counterfactual}. Group fairness requires the model to treat different groups specified by certain sensitive attributes equally without discrimination. For example, demographic parity \cite{zemel2013learning,jiang2020wasserstein,fukuchi2023demographic} defines fairness as the independence of model predictions and sensitive attributes, while equalized odds \cite{hardt2016equality} requires the prediction and group membership to be independent conditioned on the target label.

To measure the degree of fairness violation, various fairness metrics have been adopted for model evaluation \cite{garg2020fairness,franklin2022ontology,han2023ffb}. Specifically, to measure the violation of demographic parity, $\Delta\mathrm{DP}$ \cite{zemel2013learning} calculates the difference of average model predictions of instances in two demographic groups; $p$-Rule \cite{zafar2017fairness} computes the ratio between probabilities of two groups assigned the positive decision outcome; $\mathrm{SDD/SPDD}$ \cite{jiang2020wasserstein} averages the binary group prediction's disparity over 100 uniformly-spaced thresholds; and $\mathrm{ABCC}$ \cite{han2023retiring} measures the difference between the prediction distribution for different groups. In this work, we point out that these expectation-level or distribution-level metrics may fail to measure unfairness in certain cases, since the extreme but important local disparity is prone to be averaged by overall variation.

\textbf{Fair Machine Learning Algorithms.} 
To alleviate unfairness of machine learning systems, various fairness-aware algorithms have been proposed, which can be generally divided into three categories: pre-processing methods \cite{kamiran2012data,calmon2017optimized,zhang2018mitigating,biswas2021fair} attempt to adjust the training data distribution to remove the underlying data bias; in-processing methods \cite{kamishima2012fairness,zafar2017fairness2,zafar2017fairness,chuang2021fair,roh2020fairbatch} aim to reduce model's intrinsic discrimination during the training stage; and post-processing \cite{hardt2016equality,pleiss2017fairness,noriega2019active,chen2023post,yin2023fair} methods perform calibration on model predictions after the training. Among these three categories, in-processing methods fundamentally improve the fairness of both model outputs and representations, which can also lead to the highest utility \cite{barocas-fairbook,wan2023processing}. Recently, there is also a line of works focusing on how to achieve algorithmic fairness in special cases, for instance, fairness under distributional shift \cite{chai2022fairness,jiang2023chasing,roh2023improving} or with missing attribute values \cite{zhao2022towards,feng2023adapting,zhu2023weak}. Nevertheless, it is unknown that whether these algorithms can improve fairness in terms of reducing maximal local disparity or not, and this work re-evaluates some of the methods on benchmark datasets.

\vspace{-3pt}
\section{Conclusions and Future Work}

In this work, we reveal that previous expectation-level fairness metrics ($\Delta\mathrm{DP}$) cannot accurately measure the violation of demographic parity. Meanwhile, the overall variation of distribution-level metrics ($\mathrm{ABCC}$) may conceal the extreme but important local disparity. Upon this understanding, we propose a novel metric called $\mathrm{MCDP}(\epsilon)$, which calculates the maximal local disparity (defined by the minimal CDF disparity of a prediction's $\epsilon$-neighborhood) of two demographic groups. Accordingly, we propose two algorithms to estimate $\mathrm{MCDP}(\epsilon)$ with finite samples, where the approximate algorithm greatly reduces the computational complexity with high accuracy compared to the exact one. Furthermore, we develop a fair model learning framework which regularizes a differentiable estimation of $\mathrm{MCDP}(0)$. Extensive experiments are conducted on both tabular and image datasets demonstrate the effectiveness of our proposal.

In the future, we plan to extend the $\mathrm{MCDP}(\epsilon)$ metric to measure the unfairness in terms of other fairness notions, such as equalized odds \cite{hardt2016equality} and demographic parity over multiple or continuous attributes \cite{jiang2022generalized,grari2023fairness}. Moreover, designing more effective fair learning algorithms to control the maximal local disparity is also a interesting and promising research direction.

\section*{Acknowledgement}
This work was supported in part by National Natural Science Foundation of China (No. 623B2002 and 62272437) and the CCCD Key Lab of Ministry of Culture and Tourism.
\section*{Impact Statement}

Fairness is a crucial issue when designing trustworthy machine learning systems. Recent years have witnessed many impressive works which aims at developing fair machine learning algorithms under various scenarios. In this paper, by noticing the gap between overall distribution disparity (which is the core of traditional fairness measurements) and local disparity, we propose a novel fairness metric called $\mathrm{MCDP}(\epsilon)$, which captures the maximal local disparity of model predictions on two demographic groups. Compared with previous metrics such as $\Delta\mathrm{DP}$ and $\mathrm{ABCC}$, $\mathrm{MCDP}(\epsilon)$ is more strict and conservative in measuring the violation of demographic parity, which is beneficial for reducing the risk of discrimination and improving algorithmic fairness when developing high-stake applications. We call on the research community to investigate more rigorous and theoretically-guaranteed fairness evaluation protocols, which can advance the field of trustworthy machine learning.

\balance
\bibliography{refs}
\bibliographystyle{icml2024}

\newpage
\appendix
\onecolumn
\section*{Appendix}

\section{Proofs}\label{sec:proofs}

\subsection{Proofs of Theorem \ref{theo:prop_MCDP}}
\textbf{Theorem \ref{theo:prop_MCDP}} (Properties of $\mathrm{MCDP}(\epsilon)$). 
\emph{The proposed $\mathrm{MCDP}(\epsilon)$ metric has the following desired properties: 
\ding{172} $\mathrm{MCDP}(\epsilon)$ has a range of $[0,1]$. 
\ding{173} $\mathrm{MCDP}(0)=0$ holds if and only if demographic parity is established.
\ding{174} $\mathrm{MCDP}(0)$ is invariant to any monotone and invertible transformation $T:[0,1]\to[0,1]$ on $\hat{y}$.
\ding{175} $\mathrm{MCDP}(\epsilon)$ is a monotonically decreasing function \wrt $\epsilon$.
\ding{176} Assume $\Delta F(\hat{y})$ is continuous on $[0,1]$ with Lipschitz constant $L$ \cite{goldstein1977optimization}, then
\begin{equation*}
    \mathrm{ABCC}\leq
    \begin{cases}
        \mathrm{MCDP}(\epsilon)+\frac{\epsilon L}{2},& \text{if }L\leq2, \\
        \mathrm{MCDP}(\epsilon)+2\epsilon\left(1-\frac{1}{L}\right),& \text{if }L>2.
    \end{cases}
\end{equation*}
}

\begin{proof}
\ding{172} If demographic parity is established, \ie $\forall\hat{y}\in[0,1],\ F_0(\hat{y})=F_1(\hat{y})$, then $\mathrm{MCDP}(\epsilon)=0$. In addition, let $F_0(\hat{y})=\mathbb{I}(\hat{y}=1)$ and $F_1(\hat{y})=1$, we have $\mathrm{MCDP}(\epsilon)=1$. This proves that $\mathrm{MCDP}(\epsilon)\in[0,1]$.

\ding{173} According to Eq.~\eqref{eq:MCDPeps}, $\mathrm{MCDP}(0)$ can be written as
\begin{equation*}
    \mathrm{MCDP}(0)= \max\limits_{\hat{y}\in[0,1]} \Delta F(\hat{y})= \max\limits_{\hat{y}\in[0,1]} \left|F_0(\hat{y})-F_1(\hat{y})\right|.
\end{equation*}
Thus the following holds
\begin{equation*}
    \mathrm{MCDP}(0)=0\ \iff\ \forall\hat{y}\in[0,1],\ F_0(\hat{y})=F_1(\hat{y}),
\end{equation*}
which indicates that $\mathrm{MCDP}(0)=0$ and demographic parity are equivalent.

\ding{174} Denote the prediction which achieves the largest $\mathrm{MCDP}(0)$ value as $\hat{y}^*= \mathop{\arg\max}_{\hat{y}\in[0,1]} \Delta F(\hat{y})$. Moreover, denote the CDFs for the transformed predictions in two groups as $F_0^T(\hat{y})$ and $F_1^T(\hat{y})$. As $T$ is a monotone and invertible transformation on $[0,1]$, for any $\hat{y}^{\prime}\in[0,1]$, there exists $\hat{y}=T^{-1}(\hat{y}^{\prime})$, such that $F_a^T(\hat{y}^{\prime})= F_a(\hat{y}),\ a\in\{0,1\}$. Thus we have
\begin{equation*}
    \left|F_0^T(\hat{y}^{\prime})-F_1^T(\hat{y}^{\prime})\right|= \left|F_0(\hat{y})-F_1(\hat{y})\right|\leq \left|F_0(\hat{y}^*)-F_1(\hat{y}^*)\right|= \mathrm{MCDP}(0),\ \forall\hat{y}^{\prime}\in[0,1],
\end{equation*}
and the equal sign holds if and only if $\hat{y}^{\prime}=T(\hat{y}^*)$. This indicates that $\mathrm{MCDP}(0)$ is invariant to $T$.

\ding{175} For any $y_0\in[0,1]$ and $0\leq\epsilon_a<\epsilon_b\leq1$, we have 
\begin{equation*}
    \{\hat{y}:\left|\hat{y}-y_0\right|\leq\epsilon_b\} \subset \{\hat{y}:\left|\hat{y}-y_0\right|\leq\epsilon_a\}\ \Longrightarrow\ \max\limits_{y_0\in[0,1]} \min\limits_{\left|\hat{y}-y_0\right|\leq\epsilon_b} \Delta F(\hat{y})\geq \max\limits_{y_0\in[0,1]} \min\limits_{\left|\hat{y}-y_0\right|\leq\epsilon_a} \Delta F(\hat{y}),
\end{equation*}
which yields that $\mathrm{MCDP}(\epsilon_a)\leq \mathrm{MCDP}(\epsilon_b)$.
\end{proof}

\ding{176} To complete the proof of Property \ding{176}, we firstly present and prove the following two lemmas:
\begin{lemma}\label{lemma:L<2}
    If $\Delta F(\hat{y})$ is continuous on $[0,1]$ with Lipschitz constant $L\leq2$, then $\mathrm{MCDP}(\epsilon)=0 \Rightarrow \mathrm{ABCC}\leq\frac{\epsilon L}{2}$.
\end{lemma}

\begin{lemma}\label{lemma:L>2}
    If $\Delta F(\hat{y})$ is continuous on $[0,1]$ with Lipschitz constant $L>2$, then $\mathrm{MCDP}(\epsilon)=0 \Rightarrow \mathrm{ABCC}\leq\frac{2\epsilon(L-1)}{L}$.
\end{lemma}

\textbf{To complete the proof of Lemma \ref{lemma:L<2}}, we first give some useful lemmas (\textbf{which is based on $\mathrm{MCDP}(\epsilon)=0$}) below.

\begin{lemma}\label{lemma:triangle_left}
    Let $\tilde{y}_1=\min\{\hat{y}\in[0,1]:\Delta F(\hat{y})=0\}$, then $\tilde{y}_1\leq\epsilon$ and $\int_{0}^{\tilde{y}_1}\Delta F(\hat{y})\mathrm{d}\hat{y}\leq\frac{\epsilon L\tilde{y}_1}{2}$.
\end{lemma}
\begin{proof}
    Firstly, we note that the following holds
    \begin{equation*}
        \min\limits_{\hat{y}\leq\epsilon} \Delta F(\hat{y}) \leq \max\limits_{y_0\in[0,1]} \min\limits_{\left|\hat{y}-y_0\right|\leq\epsilon} \Delta F(\hat{y}) = \mathrm{MCDP}(\epsilon)=0 \Longrightarrow \min\limits_{\hat{y}\leq\epsilon}\Delta F(\hat{y})=0.
    \end{equation*}
    Thus $\tilde{y}_1\leq \mathop{\arg\min}_{\hat{y}\leq\epsilon}\Delta F(\hat{y})=0 \leq\epsilon.$ Furthermore, for any $\hat{y}\in[0,\tilde{y}_1]$, by Lipschitz continuity, we have
    \begin{equation}\label{eq:lip_deltaF0_left}
        \left|\frac{\Delta F(\hat{y})-\Delta F(\tilde{y}_1)}{\hat{y}-\tilde{y}_1}\right|\leq L \Longrightarrow \Delta F(\hat{y})\leq L(\tilde{y}_1-\hat{y}),
    \end{equation}
    which yields that
    \begin{equation*}
        \int_{0}^{\tilde{y}_1}\Delta F(\hat{y})\mathrm{d}\hat{y} \leq \int_{0}^{\tilde{y}_1} L(\tilde{y}_1-\hat{y})\mathrm{d}\hat{y} =\frac{L\tilde{y}_1^2}{2}\leq\frac{\epsilon L\tilde{y}_1}{2}.
    \end{equation*}
\end{proof}

\begin{lemma}\label{lemma:triangle}
    For any $0\leq a<b\leq1$ such that $\Delta F(a)=\Delta F(b)=0$, $\int_{a}^{b}\Delta F(\hat{y})\mathrm{d}\hat{y} \leq\frac{(b-a)^2L}{4}$.
\end{lemma}
\begin{proof}
For any $\hat{y}\in[a,b]$, similar to the derivation in Eq.~\eqref{eq:lip_deltaF0_left}, we have
\begin{equation*}\label{eq:lip_deltaF0}
    \Delta F(\hat{y})\leq \min\{(\hat{y}-a)L,(b-\hat{y})L\}=
    \begin{cases}
        (\hat{y}-a)L,& \text{if }\hat{y}\leq\frac{a+b}{2}, \\
        (b-\hat{y})L,& \text{if }\hat{y}\geq\frac{a+b}{2}.
\end{cases}
\end{equation*}
Thus we have
\begin{equation*}
    \int_{a}^{b}\Delta F(\hat{y})\mathrm{d}\hat{y} \leq \int_{a}^{\frac{a+b}{2}}(\hat{y}-a)L\mathrm{d}\hat{y} +\int_{\frac{a+b}{2}}^{b}(b-\hat{y})L\mathrm{d}\hat{y} =\frac{(b-a)^2L}{4}.
\end{equation*}
This completes the proof of Lemma \ref{lemma:triangle}.
\end{proof}

\begin{lemma}\label{lemma:DeltaF0_2eps}
    For any $\hat{y}\in(\tilde{y}_1,1)$ such that $\Delta F(\hat{y})=0$ (where $\tilde{y}_1$ is defined as Lemma \ref{lemma:triangle_left} stated), there exists $\hat{y}_l,\hat{y}_r$ such that $\max\{0,\hat{y}-2\epsilon\}\leq\hat{y}_l\leq\hat{y}\leq\hat{y}_r\leq\min\{\hat{y}+2\epsilon,1\}$ and $\Delta F(\hat{y}_l)=\Delta F(\hat{y}_r)=0$.
\end{lemma}
\begin{proof}
Assume to the contrary that there exists $y_0\in(\tilde{y}_1,1)$, such that $\Delta F(y_0)=0$ and one of the following satisfies. 
(i) If $\min_{\hat{y}\in[y_0-2\epsilon,y_0)}\Delta F(\hat{y})>0$, let $\tilde{y}_l= \max\{\hat{y}<y_0:\Delta F(\hat{y})=0\}$ and $y_l=\frac{\tilde{y}_l+y_0}{2}$, we have $\tilde{y}_l\in[\tilde{y}_1,y_0-2\epsilon)$ and $[y_l-\epsilon,y_l+\epsilon]\subset(\tilde{y}_l,y_0)$. Thus the $\mathrm{MCDP}(\epsilon)$ value should satisfy that
\begin{equation}\label{eq:DeltaF0_2eps_l}
    \mathrm{MCDP}(\epsilon)= \max\limits_{y_0\in[0,1]} \min\limits_{\left|\hat{y}-y_0\right|\leq\epsilon} \Delta F(\hat{y}) \geq \min_{|\hat{y}-y_l|\leq\epsilon}\Delta F(\hat{y}) \geq \min_{\hat{y}\in(\tilde{y}_l,y_0)}\Delta F(\hat{y})>0.
\end{equation}
(ii) If $\min_{\hat{y}\in(y_0,y_0+2\epsilon]}\Delta F(\hat{y})>0$, let $\tilde{y}_r= \min\{\hat{y}>y_0:\Delta F(\hat{y})=0\}$ and $y_r=\frac{y_0+\tilde{y}_r}{2}$. Similar to (i), we have
\begin{equation}\label{eq:DeltaF0_2eps_r}
    \mathrm{MCDP}(\epsilon)\geq \min_{|\hat{y}-y_r|\leq\epsilon}\Delta F(\hat{y}) \geq \min_{\hat{y}\in(y_0,\tilde{y}_r)}\Delta F(\hat{y})>0.
\end{equation}
The fact that at least one of Eq.~\eqref{eq:DeltaF0_2eps_l} and Eq.~\eqref{eq:DeltaF0_2eps_r} holds derives that $\mathrm{MCDP}(\epsilon)>0$, which contradicts with $\mathrm{MCDP}(\epsilon)=0$. Therefore, the original assumption must be false, which completes the proof of Lemma \ref{lemma:DeltaF0_2eps}.
\end{proof}

\emph{Proof of Lemma \ref{lemma:L<2}. }
According to Lemma \ref{lemma:DeltaF0_2eps}, we can construct a sequence of predictions $0\leq\tilde{y}_1<\cdots <\tilde{y}_M<\tilde{y}_{M+1}=1$, such that $\tilde{y}_1=\min\{\hat{y}\in[0,1]:\Delta F(\hat{y})=0\}$ and for any $i=1,\cdots,M$, $\Delta F(\tilde{y}_i)=0$ and $\tilde{y}_{i+1}-\tilde{y}_i\leq2\epsilon$. By Lemmas \ref{lemma:triangle_left}, \ref{lemma:triangle} and \ref{lemma:DeltaF0_2eps}, the value of $\mathrm{ABCC}$ should satisfy that
\begin{equation}
\begin{aligned}
    \mathrm{ABCC}=\int_{0}^{1}\Delta F(\hat{y})\mathrm{d}\hat{y}&= \int_{0}^{\tilde{y}_1} \Delta F(\hat{y})\mathrm{d}\hat{y} + \sum\limits_{i=1}^{M} \int_{\tilde{y}_i}^{\tilde{y}_{i+1}} \Delta F(\hat{y})\mathrm{d}\hat{y} \\
    &\leq \frac{\epsilon L\tilde{y}_1}{2} +\sum\limits_{i=1}^{M}\frac{(\tilde{y}_{i+1}-\tilde{y}_i)^2L}{4} \leq \frac{\epsilon L\tilde{y}_1}{2} +\sum\limits_{i=1}^{M}\frac{(\tilde{y}_{i+1}-\tilde{y}_i)\cdot2\epsilon L}{4} \\
    &= \frac{\epsilon L}{2}(\tilde{y}_1+\tilde{y}_{M+1}-\tilde{y}_1)= \frac{\epsilon L}{2}.
\end{aligned}
\end{equation}
This completes the proof of Lemma \ref{lemma:L<2}.

\textbf{To complete the proof of Lemma \ref{lemma:L>2}}, we first give the following lemmas (\textbf{which is based on $\mathrm{MCDP}(\epsilon)=0$ and $L>2$}).

\begin{lemma}[Monotonicity of CDFs]\label{lemma:cdf_mono}
    For any $0\leq a<b\leq1$, $F_0(a)<F_0(b)$ and $F_1(a)<F_1(b)$.
\end{lemma}

\begin{lemma}\label{lemma:trap_left}
    Let $\tilde{y}_1=\min\{\hat{y}\in[0,1]:\Delta F(\hat{y})=0\}$, then the following holds: (i) $h=\max_{\hat{y}\in[0,\tilde{y}_1]}\Delta F(\hat{y})\leq F_0(\tilde{y}_1)$, and (ii) $\int_{0}^{\tilde{y}_1}\Delta F(\hat{y})\mathrm{d}\hat{y} \leq\left(\tilde{y}_1-\frac{h}{2L}\right)h$.
\end{lemma}
\begin{proof}
    Let $\hat{y}^*=\mathop{\arg\max}_{\hat{y}\in[0,\tilde{y}_1]}\Delta F(\hat{y})$. By Lemma \ref{lemma:cdf_mono} and $F_0(\tilde{y}_1)=F_1(\tilde{y}_1)$, we have
    \begin{equation*}
        h=\Delta F(\hat{y}^*)\leq \max\{F_0(\hat{y}^*)-F_1(\hat{y}^*), F_1(\hat{y}^*)-F_0(\hat{y}^*)\} \leq\max\{F_0(\tilde{y}_1)-F_1(0), F_1(\tilde{y}_1)-F_0(0)\} \leq F_0(\tilde{y}_1).
    \end{equation*}
    Moreover, by Lipschitz continuity, $h$ should also satisfy that $h=\Delta F(\hat{y}^*)\leq L(\tilde{y}_1-\hat{y}^*)\leq L\tilde{y}_1$. 
    Denote $\tilde{y}_r=\tilde{y}_1-\frac{h}{L}\geq0$. Combining Eq.~\eqref{eq:lip_deltaF0_left} and $h\geq\Delta F(\hat{y}),\forall\hat{y}\in[0,\tilde{y}_1]$, we can control the bound of $\int_{0}^{\tilde{y}_1}\Delta F(\hat{y})\mathrm{d}\hat{y}$ as follows
    \begin{equation*}
    \begin{aligned}
         \int_{0}^{\tilde{y}_1}\Delta F(\hat{y})\mathrm{d}\hat{y}= \int_{0}^{\tilde{y}_r}\Delta F(\hat{y})\mathrm{d}\hat{y}+ \int_{\tilde{y}_r}^{\tilde{y}_1}\Delta F(\hat{y})\mathrm{d}\hat{y} \leq \int_{0}^{\tilde{y}_r} h\mathrm{d}\hat{y}+ \int_{\tilde{y}_r}^{\tilde{y}_1}L(\tilde{y}_1-\hat{y})\mathrm{d}\hat{y}= \tilde{y}_1h-\frac{h^2}{2L}.
    \end{aligned}
    \end{equation*}
    This completes the proof of Lemma \ref{lemma:trap_left}.
\end{proof}

\begin{lemma}\label{lemma:trap}
    For any $0\leq a<b\leq1$ such that $\Delta F(a)=\Delta F(b)=0$, the following holds: (i) $h=\max_{\hat{y}\in[a,b]}\Delta F(\hat{y})\leq F_0(b)-F_0(a)$, and (ii) $\int_{a}^{b}\Delta F(\hat{y})\mathrm{d}\hat{y} \leq\left(b-a-\frac{h}{L}\right)h$.
\end{lemma}
\begin{proof}
    Denote $\hat{y}^*=\mathop{\arg\max}_{\hat{y}\in[a,b]}\Delta F(\hat{y})$. Combining Lemma \ref{lemma:cdf_mono} and $\Delta F(a)=\Delta F(b)=0$, we have $F_0(a)=F_1(a)\leq\min\{F_0(\hat{y}^*),F_1(\hat{y}^*)\}$ and $F_0(b)=F_1(b)\geq\max\{F_0(\hat{y}^*),F_1(\hat{y}^*)\}$. Thus the following holds
    \begin{equation*}
        h=\Delta F(\hat{y}^*)= \max\{F_0(\hat{y}^*),F_1(\hat{y}^*)\}- \min\{F_0(\hat{y}^*),F_1(\hat{y}^*)\} \leq F_0(b)-F_0(a).
    \end{equation*}
    Furthermore, according to Lipschitz continuity, for any $\hat{y}\in[a,b]$, we have
    \begin{equation}\label{eq:h_upbound}
        \Delta F(\hat{y})\leq \min\{(\hat{y}-a)L,(b-\hat{y})L\}\leq \frac{(b-a)L}{2},
    \end{equation}
    which also implies that $h\leq\frac{(b-a)L}{2}$. Let $\tilde{y}_l=a+\frac{h}{L}$ and $\tilde{y}_r=b-\frac{h}{L}$, we have $a\leq\tilde{y}_l\leq\tilde{y}_r\leq b$. Combining Eq.~\eqref{eq:h_upbound} and $h\geq\Delta F(\hat{y}),\forall\hat{y}\in[a,b]$, we obtain an upper bound of $\int_a^b\Delta F(\hat{y})\mathrm{d}\hat{y}$ as follows
    \begin{equation*}
    \begin{aligned}
         \int_a^b\Delta F(\hat{y})\mathrm{d}\hat{y}&= \int_a^{\tilde{y}_l}\Delta F(\hat{y})\mathrm{d}\hat{y}+ \int_{\tilde{y}_l}^{\tilde{y}_r}\Delta F(\hat{y})\mathrm{d}\hat{y} +\int_{\tilde{y}_r}^b\Delta F(\hat{y})\mathrm{d}\hat{y} \\
         &\leq \int_a^{\tilde{y}_l} (\hat{y}-a)L \mathrm{d}\hat{y}+ \int_{\tilde{y}_l}^{\tilde{y}_r} h\mathrm{d}\hat{y} +\int_{\tilde{y}_r}^b (b-\hat{y})L \mathrm{d}\hat{y} \\
         &= \left(b-a-\frac{h}{L}\right)h.
    \end{aligned}
    \end{equation*}
    This completes the proof of Lemma \ref{lemma:trap}.
\end{proof}

\emph{Proof of Lemma \ref{lemma:L>2}. }
According to Lemma \ref{lemma:DeltaF0_2eps}, for any $\Delta F(\hat{y})=|F_0(\hat{y})-F_1(\hat{y})|$, we can construct a sequence of predictions $0=\tilde{y}_0\leq\tilde{y}_1<\cdots <\tilde{y}_M<\tilde{y}_{M+1}=1$, such that (i) $\tilde{y}_1=\min\{\hat{y}\in[0,1]:\Delta F(\hat{y})=0\}\leq\epsilon$, and (ii) for any $i=1,\cdots,M$, $\Delta F(\tilde{y}_i)=0$ and $\tilde{y}_{i+1}-\tilde{y}_i\leq2\epsilon$. Denote $\delta_i$ and $h_i$ ($i=0,\cdots,M$) as $\delta_i=\tilde{y}_{i+1}-\tilde{y}_i$ and $h_i=\max_{\hat{y}\in[\tilde{y}_i,\tilde{y}_{i+1}]}\Delta F(\hat{y})$, respectively. By Lemmas \ref{lemma:trap_left} and \ref{lemma:trap}, the following holds
\begin{equation}\label{eq:h_S_require}
\begin{aligned}
    0&\leq h_0\leq\delta_0L,\qquad 0\leq h_i\leq\frac{\delta_iL}{2},\ \text{for any }i=1,\cdots,M, \\
    \sum\limits_{i=0}^Mh_i&\leq F_0(\tilde{y}_1)+ \sum\limits_{i=1}^M \left(F_0(\tilde{y}_{i+1})-F_0(\tilde{y}_i)\right) = F_0(\tilde{y}_{i+1})=1, \\
    \int_{0}^{\tilde{y}_1}\Delta F(\hat{y})\mathrm{d}\hat{y} &\leq\left(\delta_0-\frac{h_0}{2L}\right)h_0,\qquad \int_{\tilde{y}_i}^{\tilde{y}_{i+1}}\Delta F(\hat{y})\mathrm{d}\hat{y} \leq\left(\delta_i-\frac{h_i}{L}\right)h_i,\ \text{for any }i=1,\cdots,M.
\end{aligned}
\end{equation}
Combining (i) $\sum_{i=0}^M\delta_i= \tilde{y}_{M+1}-\tilde{y}_0=1$, (ii) $\delta_0\in[0,\epsilon]$, (iii) $\delta_i\in(0,2\epsilon],\ \forall i=1,\cdots,M$, (iv) $\mathrm{ABCC}=\sum_{i=0}^M \int_{\tilde{y}_i}^{\tilde{y}_{i+1}}\Delta F(\hat{y})\mathrm{d}\hat{y}$, and Eq.~\eqref{eq:h_S_require}, for a given $M$, we can construct the following optimization problem
\begin{equation}\label{eq:opt}
\begin{aligned}
    \max\limits_{\delta_i,h_i} &{} \left(\delta_0-\frac{h_0}{2L}\right)h_0+ \sum\limits_{i=1}^M\left(\delta_i-\frac{h_i}{L}\right)h_i, \\
    \st &{} 0\leq\delta_0\leq\epsilon,\quad 0\leq h_0\leq\delta_0L, \\
    &{} 0\leq\delta_i\leq2\epsilon,\quad 0\leq h_i\leq\frac{\delta_iL}{2},\qquad \forall i=1,\cdots,M, \\
    &{} \sum\limits_{i=0}^M\delta_i=1,\quad \sum\limits_{i=0}^Mh_i\leq1.
\end{aligned}
\end{equation}
As discussed before, for any $\Delta F(\hat{y})$, we can construct a feasible solution of the optimization problem in Eq.~\eqref{eq:opt}; moreover, the $\mathrm{ABCC}$ value should less than or equal to the corresponding objective value. Therefore, the optimal objective value of Eq.~\eqref{eq:opt} should be an upper bound of $\mathrm{ABCC}$. To solve the problem, we firstly consider the following problem
\begin{equation}\label{eq:opt_relax}
\begin{aligned}
    \min\limits_{\delta_i,h_i} &{} \left(\frac{h_0}{2L}-\delta_0\right)h_0+ \sum\limits_{i=1}^M\left(\frac{h_i}{L}-\delta_i\right)h_i, \\
    \st 
    &{} \sum\limits_{i=1}^M\delta_i=1,\quad \sum\limits_{i=1}^Mh_i\leq1.
\end{aligned}
\end{equation}
Note that the problem in Eq.~\eqref{eq:opt_relax} is a convex relaxation of the original problem in Eq.~\eqref{eq:opt}, thus the (negative of) optimal value of Eq.~\eqref{eq:opt_relax} is upper bound on the optimal value of Eq.~\eqref{eq:opt} and $\mathrm{ABCC}$ \cite{boyd2004convex}. The Lagrange function of the above optimization problem is
\begin{equation*}
    \mathcal{L}(\delta_i,h_i;\alpha,\beta)= \left(\frac{h_0}{2L}-\delta_0\right)h_0+ \sum\limits_{i=1}^M\left(\frac{h_i}{L}-\delta_i\right)h_i +\alpha\left(\sum\limits_{i=1}^M\delta_i-1\right) +\beta\left(\sum\limits_{i=1}^Mh_i-1\right),
\end{equation*}
where $\beta\geq0,\alpha$ are the Lagrange multipliers. By requiring the derivatives of $\mathcal{L}(\delta_i,h_i;\alpha,\beta)$ \wrt $\delta_i,h_i$ to be zero, we have
\begin{equation}\label{eq:lagrange_sol}
    \delta_0=h_0=0,\quad \delta_i=h_i=\frac{1}{M},\ i=1,\cdots,M.
\end{equation}
It's easy to validate that Eq.~\eqref{eq:lagrange_sol} also satisfies all constraints in Eq.~\eqref{eq:opt} (as $L>2$), indicating that Eq.~\eqref{eq:lagrange_sol} is the optimal solution of the problem in Eq.~\eqref{eq:opt}. 
Note that $1=\sum_{i=0}^M\delta_i\leq M\cdot2\epsilon$ holds, thus $M\geq\frac{1}{2\epsilon}$. Therefore, we can compute an upper bound of $\mathrm{ABCC}$ as below
\begin{equation*}
    \mathrm{ABCC}\leq 0+\sum\limits_{i=1}^M \left(\frac{1}{M}-\frac{1}{ML}\right)\frac{1}{M} =\frac{1}{M}\left(1-\frac{1}{L}\right) \leq2\epsilon\left(1-\frac{1}{L}\right).
\end{equation*}
This completes the proof of Lemma \ref{lemma:L>2}. \textbf{Note that} the whole proof process does not require $L\leq2$, \ie Lemmas \ref{lemma:triangle_left}-\ref{lemma:trap} applies to all $L$ values (including $L>2$).

\textbf{At last, we restate Property \ding{176} as the following theorem, and complete its proof accordingly.}
\begin{theorem}[An upper bound of $\mathrm{ABCC}$]\label{theo:mcdp_lip_abcc}
Assume $\Delta F(\hat{y})$ is continuous on $[0,1]$ with Lipschitz constant $L$, then
\begin{equation*}\label{eq:mcdp_lip_abcc}
    \mathrm{ABCC}\leq
    \begin{cases}
        \mathrm{MCDP}(\epsilon)+\frac{\epsilon L}{2},& \text{if }L\leq2, \\
        \mathrm{MCDP}(\epsilon)+2\epsilon\left(1-\frac{1}{L}\right),& \text{if }L>2.
    \end{cases}
\end{equation*}
\end{theorem}

\begin{proof}
For any $\Delta F(\hat{y})=|F_0(\hat{y})-F_1(\hat{y})|$ such that $\mathrm{MCDP}(\epsilon)=M>0$, construct $\tilde{F}_0(\hat{y}),\tilde{F}_1(\hat{y})$ as follows
\begin{equation}\label{eq:tildeF01}
    \tilde{F}_0(\hat{y})=
    \begin{cases}
    \max\{F_0(\hat{y})-M,F_1(\hat{y})\},& \text{if }F_0(\hat{y})\geq F_1(\hat{y}) \\
    \min\{F_0(\hat{y})+M,F_1(\hat{y})\},& \text{if }F_0(\hat{y})< F_1(\hat{y})
    \end{cases}
    ,\quad \tilde{F}_1(\hat{y})=F_1(\hat{y}),
\end{equation}
and denote $\Delta\tilde{F}(\hat{y})=|\tilde{F}_0(\hat{y})-\tilde{F}_1(\hat{y})|$. We show that $\tilde{F}_0(\hat{y}),\tilde{F}_1(\hat{y})$, and $\Delta\tilde{F}(\hat{y})$ have the following properties.

\begin{lemma}\label{lemma:tildeF_value}
    The value of $\tilde{F}_0(\hat{y})$ satisfies that (i) $\tilde{F}_0(0)\geq0$, (ii) $\tilde{F}_0(1)=1$, (iii) for any $0\leq a<b\leq1$, $\tilde{F}_0(a)\leq\tilde{F}_0(b)$, and (iv) $|\tilde{F}_0(\hat{y})-F_0(\hat{y})|\leq M$.
\end{lemma}
\begin{proof}
(i)(ii) By Eq.~\eqref{eq:tildeF01}, the following holds
\begin{equation*}
    \tilde{F}_0(0)=
    \begin{cases}
    \max\{F_0(0)-M,F_1(0)\}\geq F_1(0)\geq0,& \text{if }F_0(\hat{y})\geq F_1(\hat{y}) \\
    \min\{F_0(0)+M,F_1(0)\}\geq \min\{M,0\}\geq0,& \text{if }F_0(\hat{y})< F_1(\hat{y})
    \end{cases}
    ,\quad \tilde{F}_0(1)=\tilde{F}_1(1)=1.
\end{equation*}
(iii) By the monotonicity of CDF, for any $0\leq a<b\leq1$, we have $\tilde{F}_0(b)\geq F_1(b)\geq F_1(a)\geq\tilde{F}_0(a)$. 

(iv) By Eq.~\eqref{eq:tildeF01}, $|\tilde{F}_0(\hat{y})-F_0(\hat{y})|\leq M$ satisfies that
\begin{equation*}
    |\tilde{F}_0(\hat{y})-F_0(\hat{y})|=
    \begin{cases}
    |\max\{-M,F_1(\hat{y})-F_0(\hat{y})\}|=|\min\{M,F_0(\hat{y})-F_1(\hat{y})\}|\leq M,& \text{if }F_0(\hat{y})\geq F_1(\hat{y}) \\
    |\min\{M,F_1(\hat{y})-F_0(\hat{y})\}|\leq M,& \text{if }F_0(\hat{y})< F_1(\hat{y}).
    \end{cases}
\end{equation*}
This completes the proof of Lemma \ref{lemma:tildeF_value}.
\end{proof}

\begin{lemma}\label{lemma:deltatildeF_value}
    The value of $\Delta\tilde{F}(\hat{y})$ can be computed by $\Delta\tilde{F}(\hat{y})=\max\{0,\Delta F(\hat{y})-M\}$.
\end{lemma}
\begin{proof}
By Eq.~\eqref{eq:tildeF01}, we have
\begin{equation*}
\begin{aligned}
    \Delta\tilde{F}(\hat{y})= |\tilde{F}_0(\hat{y})-F_1(\hat{y})|&=
    \begin{cases}
    |\max\{F_0(\hat{y})-F_1(\hat{y})-M,0\}|= |\max\{\Delta F(\hat{y})-M,0\}|,& \text{if }F_0(\hat{y})\geq F_1(\hat{y}) \\
    |\min\{F_0(\hat{y})-F_1(\hat{y})+M,0\}|\ = |\min\{M-\Delta F(\hat{y}),0\}|,& \text{if }F_0(\hat{y})< F_1(\hat{y})
    \end{cases} \\
    &= 
    \begin{cases}
    |\Delta F(\hat{y})-M|,& \text{if }\Delta F(\hat{y})\geq M \\
    0,& \text{if }\Delta F(\hat{y})<M
    \end{cases}.
\end{aligned}
\end{equation*}
which yields that $\Delta\tilde{F}(\hat{y})=\max\{0,\Delta F(\hat{y})-M\}$.
\end{proof}

\begin{lemma}\label{lemma:deltatildeF_mono}
    $\Delta\tilde{F}(\hat{y})$ and $\Delta F(\hat{y})$ exhibit the same monotonic behavior.
\end{lemma}
\begin{proof}
One the one hand, for any $a,b\in[0,1]$ such that $\Delta F(a)\leq\Delta F(b)$, by Lemma \ref{lemma:deltatildeF_value}, we have
\begin{equation}\label{eq:deltaFab}
\begin{aligned}
    \Delta\tilde{F}(a)-\Delta\tilde{F}(b)&= \max\{0,\Delta F(a)-M\}-\max\{0,\Delta F(b)-M\} \\
    &=
    \begin{cases}
    0,& \text{if }M\geq \Delta F(b) \\
    M-\Delta F(b)<0,& \text{if }\Delta F(a)<M<\Delta F(b) \\
    \Delta F(a)-\Delta F(b)\leq0,& \text{if }M\leq \Delta F(a)
    \end{cases},
\end{aligned}
\end{equation}
which yields that $\Delta\tilde{F}(a)\leq\Delta\tilde{F}(b)$. On the other hand, for any $a,b\in[0,1]$ such that $\Delta F(a)\geq\Delta F(b)$, $\Delta\tilde{F}(a)\geq\Delta\tilde{F}(b)$ can be proved in a similar manner. This yields that $\Delta\tilde{F}(\hat{y})$ and $\Delta F(\hat{y})$ have the identical monotonic trend.
\end{proof}

\begin{lemma}\label{lemma:deltatildeF_diff}
    For any $a,b\in[0,1]$, $|\Delta\tilde{F}(a)-\Delta\tilde{F}(b)|\leq |\Delta F(a)-\Delta F(b)|$.
\end{lemma}
\begin{proof}
If $\Delta F(a)\leq\Delta F(b)$, by Eq.~\eqref{eq:deltaFab}, we have
\begin{equation*}
    |\Delta\tilde{F}(a)-\Delta\tilde{F}(b)|= 
    \begin{cases}
    0\leq \Delta F(b)-\Delta F(a),& \text{if }M\geq \Delta F(b) \\
    \Delta F(b)-M<\Delta F(b)-\Delta F(a),& \text{if }\Delta F(a)<M<\Delta F(b) \\
    \Delta F(b)-\Delta F(a),& \text{if }M\leq \Delta F(a)
    \end{cases},
\end{equation*}
which yields that $|\Delta\tilde{F}(a)-\Delta\tilde{F}(b)|\leq |\Delta F(a)-\Delta F(b)|$. When $\Delta F(a)\geq\Delta F(b)$, $|\Delta\tilde{F}(a)-\Delta\tilde{F}(b)|\leq |\Delta F(a)-\Delta F(b)|$ can be proved in a similar way. This completes the proof of Lemma \ref{lemma:deltatildeF_diff}.
\end{proof}

Next we continue the proof of Theorem \ref{theo:mcdp_lip_abcc}. By (i),(ii) and (iii) in Lemma \ref{lemma:tildeF_value}, $\tilde{F}_0(\hat{y})$ and $\tilde{F}_1(\hat{y})$ in Eq.~\eqref{eq:tildeF01} are valid CDFs. Moreover, by Lemma \ref{lemma:deltatildeF_diff} and Lipschitz continuity of $\Delta F(\hat{y})$, the following holds
\begin{equation*}
    \left|\frac{\Delta\tilde{F}(a)-\Delta\tilde{F}(b)}{a-b}\right|\leq \left|\frac{\Delta F(a)-\Delta F(b)}{a-b}\right|\leq L,\quad\text{for any }a,b\in[0,1],
\end{equation*}
which demonstrates that $\Delta\tilde{F}(\hat{y})$ is also continuous on $[0,1]$ with Lipschitz constant $L$. Furthermore, as Lemma \ref{lemma:deltatildeF_mono} illustrates that the monotonicity of $\tilde{F}_0(\hat{y})$ and $\tilde{F}_1(\hat{y})$ is consistent, the following holds
\begin{equation*}
\begin{aligned}
    \mathop{\arg\min}\limits_{|\hat{y}-y_0|\leq\epsilon} \Delta F(\hat{y})&= \mathop{\arg\min}\limits_{|\hat{y}-y_0|\leq\epsilon} \Delta\tilde{F}(\hat{y}),\quad\text{for any }y_0\in[0,1], \\
    \Longrightarrow \mathop{\arg\max}\limits_{y_0\in[0,1]} \Delta F\left(\mathop{\arg\min}\limits_{|\hat{y}-y_0|\leq\epsilon} \Delta F(\hat{y})\right)&= \mathop{\arg\max}\limits_{y_0\in[0,1]} \Delta\tilde{F}\left(\mathop{\arg\min}\limits_{|\hat{y}-y_0|\leq\epsilon} \Delta\tilde{F}(\hat{y})\right), \\
    \Longrightarrow \max\limits_{y_0\in[0,1]} \min\limits_{|\hat{y}-y_0|\leq\epsilon} \Delta\tilde{F}(\hat{y})= \Delta\tilde{F}(\hat{y}^*) &\overset{(*)}{=}\max\{0,\Delta F(\hat{y}^*)-M\}=0,
\end{aligned}
\end{equation*}
where $\hat{y}^*\in[0,1]$ satisfies that $\Delta F(\hat{y}^*)=M$, and equality (*) holds due to Lemma \ref{lemma:deltatildeF_value}. To conclude, $\tilde{F}_0(\hat{y})$ and $\tilde{F}_1(\hat{y})$ in Eq.~\eqref{eq:tildeF01} satisfies all conditions of Lemmas \ref{lemma:L<2} and \ref{lemma:L>2}. Denote $A$ and $\tilde{A}$ as the $\mathrm{ABCC}$ values based on $\Delta F(\hat{y})$ and $\Delta\tilde{F}(\hat{y})$, respectively. By (iv) in Lemma \ref{lemma:tildeF_value} and $||a-b|-|c-d||\leq|a-c|+|b-d|$, we have
\begin{equation*}
\begin{aligned}
    A-\tilde{A}&= \int_{0}^{1}\left|F_0(\hat{y})-F_1(\hat{y})\right|\mathrm{d}\hat{y}- \int_{0}^{1}\left|\tilde{F}_0(\hat{y})-\tilde{F}_1(\hat{y})\right|\mathrm{d}\hat{y} \\
    &\leq \int_{0}^{1}\left|F_0(\hat{y})-\tilde{F}_0(\hat{y})\right|\mathrm{d}\hat{y}+ \int_{0}^{1}\left|F_1(\hat{y})-\tilde{F}_1(\hat{y})\right|\mathrm{d}\hat{y} \\
    &\leq \int_{0}^{1}M\mathrm{d}\hat{y}+0= M.
\end{aligned}
\end{equation*}
By Lemmas \ref{lemma:L<2} and \ref{lemma:L>2}, we have 
\begin{equation*}
    A\leq\tilde{A}+M\leq
    \begin{cases}
        M+\frac{\epsilon L}{2},& \text{if }L\leq2, \\
        M+2\epsilon\left(1-\frac{1}{L}\right),& \text{if }L>2.
    \end{cases}
\end{equation*}
This completes the proof of Theorem \ref{theo:mcdp_lip_abcc}.

\end{proof}

\subsection{Proofs of Theorems \ref{theo:exactalg}, \ref{theo:approover} and \ref{theo:appromono}}

To complete the proofs of theorems about Algorithms \ref{alg:exact} and \ref{alg:approximate}, we start with giving some important lemmas and corollaries.

\begin{lemma}[Subset min-max property \cite{rosen2007discrete}]\label{lemma:set_maxmin}
For any two finite ordered sets $A,B$, if $A\subseteq B$, then $\min A\geq\min B$ and $\max A\leq\max B$.
\end{lemma}

\begin{lemma}\label{lemma:delta_ecdf_range}
The range of $\Delta\hat{F}(\hat{y})= |\hat{F}_0(\hat{y})-\hat{F}_1(\hat{y})|$ over the interval $[y_1,y_2]\subseteq[0,1]$ is a finite set 
\begin{equation*}
    \left\{\Delta\hat{F}(\hat{y}):\hat{y}\in[y_1,y_2]\right\}= \left\{\Delta\hat{F}(\hat{y}_i):\hat{y}_i\in[\hat{y}_l,\hat{y}_r]\right\},\ l,r\in[0,N+1],
\end{equation*}
where $\hat{y}_0=0,\ \hat{y}_{N+1}=1,\ \hat{y}_l=\max\limits_i\{\hat{y}_i:\hat{y}_i\leq y_1\},\ \hat{y}_r=\max\limits_i\{\hat{y}_i:\hat{y}_i\leq y_2\},\ i\in[0,N+1]$.
\end{lemma}

\begin{proof}
\textbf{In the following, for the simplicity of exposition, the default value domain of subscript indices $i,j$ in model predictions $\hat{y}_i,\hat{y}_j$ is $\{0,1,\cdots,N+1\}$, unless we point out different value domains in equations.}

According to Eq.~\eqref{eq:ecdf}, for any $\hat{y}\in[y_1,y_2]$, there exists an instance's prediction $\hat{y}_j=\max_i\{\hat{y}_i:\hat{y}_i\leq \hat{y}\}$ such that $\hat{F}_a(\hat{y})= \hat{F}_a(\hat{y}_j),\ a\in\{0,1\}$, which implies that $\Delta\hat{F}(\hat{y})= \Delta\hat{F}(\hat{y}_j)$. As $\hat{y}\geq y_1$ holds, $\{\hat{y}_i:\hat{y}_i\leq y_1\}$ should be a subset of $\{\hat{y}_i:\hat{y}_i\leq \hat{y}\}$. According to Lemma \ref{lemma:set_maxmin}, we have $\hat{y}_j\geq\hat{y}_l$. Similarly, we can prove that $\hat{y}_j\leq\hat{y}_r$. Therefore,
\begin{equation}\label{eq:lemma2_subset}
    \{\Delta\hat{F}(\hat{y}):\hat{y}\in[y_1,y_2]\} \subseteq\{\Delta\hat{F}(\hat{y}_i): \hat{y}_i\in[\hat{y}_l,\hat{y}_r]\}.
\end{equation}
Meanwhile, for any $\hat{y}_i\in[\hat{y}_l,\hat{y}_r]$, we have $\hat{F}_a(\hat{y}_i)= \hat{F}_a(y^{\prime}),\ a\in\{0,1\}$ and $\Delta\hat{F}(\hat{y}_i)= \Delta\hat{F}(y^{\prime})$, where $y^{\prime}=\max\{y_1,\hat{y}_i\}\in[y_1,y_2]$. Thus the following holds
\begin{equation}\label{eq:lemma2_supset}
    \{\Delta\hat{F}(\hat{y}):\hat{y}\in[y_1,y_2]\} \supseteq\{\Delta\hat{F}(\hat{y}_i): \hat{y}_i\in[\hat{y}_l,\hat{y}_r]\}.
\end{equation}
Combining Eq.~\eqref{eq:lemma2_subset} and Eq.~\eqref{eq:lemma2_supset} completes the proof of Lemma \ref{lemma:delta_ecdf_range}. 
\end{proof}

According to Lemma \ref{lemma:set_maxmin} and Lemma \ref{lemma:delta_ecdf_range}, we have the following corollary:
\begin{corollary}[An over-estimation of the minimal value of $\Delta\hat{F}(\hat{y})$]\label{cor:f01over}
The minimal value of $\Delta\hat{F}(\hat{y})$ over the interval $[y_1,y_2]\subseteq[0,1],\ y_2\geq y_1+\delta$ can be over-estimated by the minimal value of 
\begin{equation*}
    \left\{\Delta\hat{F}(k\delta): k=\left\lceil\frac{y_1}{\delta}\right\rceil, \left\lceil\frac{y_1}{\delta}+1\right\rceil,\cdots, \left\lfloor\frac{y_2}{\delta}\right\rfloor\right\},
\end{equation*}
where $\delta\in(0,1)$ is the step-size hyper-parameter.
\end{corollary}

\textbf{Theorem \ref{theo:exactalg}} (Exactness). 
\emph{The $\widetilde{\mathrm{MCDP}}(\epsilon)$ value returned by Algorithm \ref{alg:exact} equals to the $\widehat{\mathrm{MCDP}}(\epsilon)$ value in Eq.~\eqref{eq:mcdphat}, \ie Algorithm \ref{alg:exact} calculates $\widehat{\mathrm{MCDP}}(\epsilon)$ exactly without error.}

\begin{proof}
According to Lemma \ref{lemma:delta_ecdf_range}, when $\epsilon=0$, we have
\begin{equation}\label{eq:exact_eps0}
\begin{aligned}
    & \{\Delta\hat{F}(\hat{y}):\hat{y}\in[0,1]\}= \{\Delta\hat{F}(\hat{y}_i):\hat{y}_i\in[0,1]\}= \{\Delta\hat{F}(\hat{y}_i)\}, 
    \\ \Rightarrow & \ \widehat{\mathrm{MCDP}}(0)= \max\limits_{\hat{y}\in[0,1]} \Delta\hat{F}(\hat{y})= \max\limits_i\Delta\hat{F}(\hat{y}_i)= \widetilde{\mathrm{MCDP}}(0).
\end{aligned}
\end{equation}

When $\epsilon>0$, for any $y_0\in[0,\epsilon)$, as $[0,\epsilon]\subseteq[0,y_0+\epsilon]\subset[0,2\epsilon]$, by Lemma \ref{lemma:set_maxmin} and \ref{lemma:delta_ecdf_range}, we have
\begin{equation}\label{eq:exact_line6}
    \max\limits_{y_0\in[0,\epsilon)} \min\limits_{|\hat{y}-y_0|\leq\epsilon} \Delta\hat{F}(\hat{y})= \min\limits_{\hat{y}\leq\epsilon} \Delta\hat{F}(\hat{y})= \min\limits_{i:\hat{y}_i\leq\epsilon} \Delta\hat{F}(\hat{y}_i),
\end{equation}
where $\min_{i:\hat{y}_i\leq\epsilon} \Delta\hat{F}(\hat{y}_i)$ equals to the initial $\widetilde{\mathrm{MCDP}}(\epsilon)$ value in line 6 of Algorithm \ref{alg:exact}. For any $y_0\in[\epsilon,1]$, let $\hat{y}_{0l}=\max_i\{\hat{y}_i:\hat{y}_i\leq y_0-\epsilon\}$ and $\hat{y}_{0r}=\max_i\{\hat{y}_i:\hat{y}_i\leq y_0+\epsilon\}$. By Lemmas \ref{lemma:delta_ecdf_range} and \ref{lemma:set_maxmin}, we have
\begin{equation}\label{eq:exact_line7_eq1}
\begin{aligned}
    & \{\Delta\hat{F}(\hat{y}): |\hat{y}-y_0|\leq\epsilon\}= \{\Delta\hat{F}(\hat{y}_i): \hat{y}_i\in[\hat{y}_{0l},\hat{y}_{0r}]\}\supseteq \{\Delta\hat{F}(\hat{y}_i): \hat{y}_i\in[\hat{y}_{0l},\hat{y}_{0l}+2\epsilon]\}, \\
    \Rightarrow & \min\limits_{i:\hat{y}_i\in[\hat{y}_{0l},\hat{y}_{0l}+2\epsilon]} \Delta\hat{F}(\hat{y}_i)\geq \min\limits_{|\hat{y}-y_0|\leq\epsilon}\Delta\hat{F}(\hat{y}), \\
    \Rightarrow & \max\limits_{i:\hat{y}_i\leq1-\epsilon} \min\limits_{j:\hat{y}_j\in[\hat{y}_i,\hat{y}_i+2\epsilon]} \Delta\hat{F}(\hat{y}_j)\geq \max\limits_{y_0\in[\epsilon,1]} \min\limits_{|\hat{y}-y_0|\leq\epsilon}\Delta\hat{F}(\hat{y}).
\end{aligned}
\end{equation}
where the last step uses the fact that $\{\hat{y}_{0l}:y_0\in[\epsilon,1]\}= \{\hat{y}_i:\hat{y}_i\leq1-\epsilon\}$. 

On the other hand, for any $i$ such that $\hat{y}_i\leq1-\epsilon$, by Lemma \ref{lemma:delta_ecdf_range}, there exists $y_i^{\prime}=\hat{y}_i+\epsilon$ (thus $y_i^{\prime}\in[\epsilon,1]$) such that
\begin{equation}\label{eq:exact_line7_eq2}
\begin{aligned}
    \{\Delta\hat{F}(\hat{y}):|\hat{y}-y_i^{\prime}|\leq\epsilon\}&= \{\Delta\hat{F}(\hat{y}):\hat{y}\in[y_i^{\prime}-\epsilon,\min(1,y_i^{\prime}+\epsilon)]\}= \{\Delta\hat{F}(\hat{y}_j):\hat{y}_j\in[y_i^{\prime}-\epsilon,\min(1,y_i^{\prime}+\epsilon)]\} \\
    &= \{\Delta\hat{F}(\hat{y}_j):\hat{y}_j\in[\hat{y}_i,\hat{y}_i+2\epsilon]\}.
\end{aligned}
\end{equation}
Note that for any $i$ such that $\hat{y}_i\leq1-\epsilon$, the following also holds
\begin{equation}\label{eq:exact_line7_eq3}
    \left\{\min\limits_{|\hat{y}-y_i^{\prime}|\leq\epsilon} \Delta\hat{F}(\hat{y}): \hat{y}_i\leq1-\epsilon\right\} \subseteq \left\{\min\limits_{|\hat{y}-y_0|\leq\epsilon} \Delta\hat{F}(\hat{y}): y_0\in[\epsilon,1]\right\}.
\end{equation}
By Lemma \ref{lemma:set_maxmin}, Eq.~\eqref{eq:exact_line7_eq2} and Eq.~\eqref{eq:exact_line7_eq3}, we have
\begin{equation}\label{eq:exact_line7_eq4}
    \max\limits_{i:\hat{y}_i\leq1-\epsilon} \min\limits_{\hat{y}_j\in[\hat{y}_i,\hat{y}_i+2\epsilon]} \Delta\hat{F}(\hat{y}_j) =\max\limits_{i:\hat{y}_i\leq1-\epsilon} \min\limits_{|\hat{y}-y_i^{\prime}|\leq\epsilon} \Delta\hat{F}(\hat{y}) \leq\max\limits_{y_0\in[\epsilon,1]} \min\limits_{|\hat{y}-y_0|\leq\epsilon} \Delta\hat{F}(\hat{y}).
\end{equation}

Taking a step further, combining Eq.~\eqref{eq:exact_line7_eq1} and Eq.~\eqref{eq:exact_line7_eq4}, yields that
\begin{equation}\label{eq:exact_line7}
    \max\limits_{i:\hat{y}_i\leq1-\epsilon} \min\limits_{\hat{y}_j\in[\hat{y}_i,\hat{y}_i+2\epsilon]} \Delta\hat{F}(\hat{y}_j) =\max\limits_{y_0\in[\epsilon,1]} \min\limits_{|\hat{y}-y_0|\leq\epsilon} \Delta\hat{F}(\hat{y}).
\end{equation}

At last, according to Eq.~\eqref{eq:exact_line6} and Eq.~\eqref{eq:exact_line7}, we can derive the relationship between $\widehat{\mathrm{MCDP}}(\epsilon)$ (Eq.~\eqref{eq:mcdphat}) and $\widetilde{\mathrm{MCDP}}(\epsilon)$ (Algorithm \ref{alg:exact}) as follows
\begin{equation}\label{eq:exact_exaxt=estimate}
\begin{aligned}
    \widehat{\mathrm{MCDP}}(\epsilon)&= \max\limits_{y_0\in[0,1]} \min\limits_{|\hat{y}-y_0|\leq\epsilon} \Delta\hat{F}(\hat{y})= \max\left\{\max\limits_{y_0\in[0,\epsilon)} \min\limits_{|\hat{y}-y_0|\leq\epsilon} \Delta\hat{F}(\hat{y}), \max\limits_{y_0\in[\epsilon,1]} \min\limits_{|\hat{y}-y_0|\leq\epsilon} \Delta\hat{F}(\hat{y})\right\} \\
    &= \max\left\{\min\limits_{i:\hat{y}_i\leq\epsilon} \Delta\hat{F}(\hat{y}_i), \max\limits_{i:\hat{y}_i\leq1-\epsilon} \min\limits_{\hat{y}_j\in[\hat{y}_i,\hat{y}_i+2\epsilon]} \Delta\hat{F}(\hat{y}_j)\right\} = \widetilde{\mathrm{MCDP}}(\epsilon),\ \epsilon>0.
\end{aligned}
\end{equation}
Combining Eq.~\eqref{eq:exact_eps0} and Eq.~\eqref{eq:exact_exaxt=estimate} completes the proof of Theorem \ref{theo:exactalg}.

\end{proof}

\textbf{Theorem \ref{theo:approover}} (Over-estimation). 
The $\widetilde{\mathrm{MCDP}}(\epsilon)$ value returned by the approximate algorithm satisfies that $\widetilde{\mathrm{MCDP}}(\epsilon)\ge\widehat{\mathrm{MCDP}}(\epsilon)$, \ie Algorithm \ref{alg:approximate} never underestimates $\widehat{\mathrm{MCDP}}(\epsilon)$.

\begin{proof}
For any $y_0\in[0,\epsilon],\ \epsilon>0$, according to Eq.~\eqref{eq:exact_line6} and Corollary \ref{cor:f01over}, we have
\begin{equation}\label{eq:appro_left}
    \max\limits_{y_0\in[0,\epsilon)} \min\limits_{|\hat{y}-y_0|\leq\epsilon} \Delta\hat{F}(\hat{y})= \min\limits_{\hat{y}\leq\epsilon} \Delta\hat{F}(\hat{y})\leq \min\limits_{j\in\{0,\cdots,K\}} \Delta\hat{F}(j\delta),
\end{equation}
where $\min_j\Delta\hat{F}(j\delta),j\in\{0,\cdots,K\}$ is the initial value of $\widetilde{\mathrm{MCDP}}(\epsilon)$ in Algorithm \ref{alg:approximate} (line 3).

For any $y_0\in(\epsilon,1-\epsilon)$, by Corollary \ref{cor:f01over} and Corollary \ref{lemma:set_maxmin}, we have
\begin{equation}\label{eq:appro_mid_eq1}
\begin{aligned}
    &\min\limits_{|\hat{y}-y_0|\leq\epsilon} \Delta\hat{F}(\hat{y}) \leq\min\limits_{k\in\{l,\cdots,r\}} \Delta\hat{F}(k\delta) \leq\min\limits_{k\in\{l,\cdots,l+2K-1\}} \Delta\hat{F}(k\delta), \\
    \text{where } & l=\left\lceil\frac{y_0-\epsilon}{\delta}\right\rceil =\left\lceil\frac{y_0}{\delta}\right\rceil-K,\quad r=\left\lfloor\frac{y_0+\epsilon}{\delta}\right\rfloor =\left\lfloor\frac{y_0}{\delta}\right\rfloor+K\geq l+2K-1.
\end{aligned}
\end{equation}
Note that $\{\left\lceil\frac{y_0}{\delta}\right\rceil-K: y_0\in(\epsilon,1-\epsilon)\}= \{1,\cdots,\left\lceil\frac{1}{\delta}\right\rceil-2K\}$. Taking a step further, Eq.~\eqref{eq:appro_mid_eq1} derives that
\begin{equation}\label{eq:appro_mid}
    \max\limits_{y_0\in(\epsilon,1-\epsilon)} \min\limits_{|\hat{y}-y_0|\leq\epsilon} \Delta\hat{F}(\hat{y})\leq \max\limits_{l\in\{1,\cdots,\left\lceil\frac{1}{\delta}\right\rceil-2K\}} \min\limits_{k\in\{l,\cdots,l+2K-1\}} \Delta\hat{F}(k\delta).
\end{equation}

For any $y_0\in[1-\epsilon,1]$, $\min_{|\hat{y}-y_0|\leq\epsilon} \Delta\hat{F}(\hat{y})= \Delta\hat{F}(1)=0$ holds, which yields that
\begin{equation}\label{eq:appro_right}
    \max\limits_{y_0\in[1-\epsilon,1]} \min\limits_{|\hat{y}-y_0|\leq\epsilon} \Delta\hat{F}(\hat{y})= \Delta\hat{F}(1)=0.
\end{equation}

Combining Eq.~\eqref{eq:appro_left}, Eq.~\eqref{eq:appro_mid} and Eq.~\eqref{eq:appro_right}, the $\widetilde{\mathrm{MCDP}}(\epsilon)$ value returned by Algorithm \ref{alg:approximate} satisfies that
\begin{equation*}
\begin{aligned}
    \widetilde{\mathrm{MCDP}}(\epsilon)&= \max\left\{\min\limits_{j\in\{0,\cdots,K\}} \Delta\hat{F}(j\delta), \max\limits_{l\in\{1,\cdots,\left\lceil\frac{1}{\delta}\right\rceil-2K\}} \min\limits_{k\in\{l,\cdots,l+2K-1\}} \Delta\hat{F}(k\delta)\right\} \\
    &\geq \max\left\{\max\limits_{y_0\in[0,\epsilon)} \min\limits_{|\hat{y}-y_0|\leq\epsilon} \Delta\hat{F}(\hat{y}), \max\limits_{y_0\in(\epsilon,1-\epsilon)} \min\limits_{|\hat{y}-y_0|\leq\epsilon} \Delta\hat{F}(\hat{y}), \max\limits_{y_0\in[1-\epsilon,1]} \min\limits_{|\hat{y}-y_0|\leq\epsilon} \Delta\hat{F}(\hat{y})\right\} \\
    &= \max\limits_{y_0\in[0,1]} \min\limits_{|\hat{y}-y_0|\leq\epsilon} \Delta\hat{F}(\hat{y})= \widehat{\mathrm{MCDP}}(\epsilon),\ \epsilon>0.
\end{aligned}
\end{equation*}
This completes the proof of Theorem \ref{theo:approover}.

\end{proof}

\textbf{Theorem \ref{theo:appromono}} (Monotonicity \wrt sampling frequency $K$). 
Denote $\widetilde{\mathrm{MCDP}}(\epsilon;K)$ as the $\widehat{\mathrm{MCDP}}(\epsilon)$ value returned by Algorithm \ref{alg:approximate} with sampling frequency $K$. For any $p>q\ge0,\ p,q\in\mathbb{N}$, we have $\widetilde{\mathrm{MCDP}}(\epsilon;2^p)\le\widetilde{\mathrm{MCDP}}(\epsilon;2^q)$.
\begin{proof}
Let $K_p=2^p,\ K_q=2^q,\ \delta_p=\frac{\epsilon}{K_p},\ \delta_q=\frac{\epsilon}{K_q}$, and
\begin{equation*}
    \mathcal{Y}_K^{(j)}=\{j\cdot\frac{\epsilon}{K},\cdots,(j+2K-1)\cdot\frac{\epsilon}{K}\},\quad\text{for all }j=1,\cdots,\left\lceil\frac{K}{\epsilon}\right\rceil-2K,\ K\in\mathbb{N}_+.
\end{equation*}
Before giving the proof, we first introduce two lemmas which will be used later.

\begin{lemma}\cite{graham1989concrete}\label{lemma:ceilfloor}
    If $m\in\mathbb{N}_+,\ n\in\mathbb{Z}$, then the equation below can be used to convert ceilings to floors
    \begin{equation*}
        \left\lceil\frac{n}{m}\right\rceil= \left\lfloor\frac{n+m-1}{m}\right\rfloor= \left\lfloor\frac{n-1}{m}\right\rfloor+1.
    \end{equation*}
\end{lemma}

\begin{lemma}\label{lemma:appro_mono}
    For any $j_p=1,\cdots,\left\lceil\frac{1}{\delta_p}\right\rceil-2K_p$, there exists $j_q\in\{1,\cdots,\left\lceil\frac{1}{\delta_q}\right\rceil-2K_q\}$, such that $\mathcal{Y}_{K_q}^{(j_q)} \subseteq \mathcal{Y}_{K_p}^{(j_p)}$.
\end{lemma}

\begin{proof}
Let $j_q=\left\lceil j_p\cdot2^{q-p}\right\rceil \geq\left\lceil 1\cdot2^{-1}\right\rceil=1$. By lemma \ref{lemma:ceilfloor}, we have
\begin{equation*}
\begin{aligned}
    j_q&= \left\lfloor(j_p-1)\cdot2^{q-p}\right\rfloor \\
    &\leq \left\lfloor(\lceil\frac{1}{\delta_p}\rceil-2K_p-1) \cdot2^{q-p}\right\rfloor =\left\lfloor(\lceil\frac{1}{\delta_p}\rceil-1) \cdot2^{q-p}\right\rfloor-2K_q \\
    &\leq \left\lfloor\frac{1}{\delta_p} \cdot2^{q-p}\right\rfloor-2K_q =\left\lfloor\frac{1}{\delta_q}\right\rfloor-2K_q \leq\left\lceil\frac{1}{\delta_q}\right\rceil-2K_q,
\end{aligned}
\end{equation*}
which suggests that $j_q\in\{1,\cdots,\left\lceil\frac{1}{\delta_q}\right\rceil-2K_q\}$. For any $y_q=(j_q+l_q)\cdot\delta_q \in\mathcal{Y}_{K_q}^{(j_q)},\ l_q=0,\cdots,2K_q-1$, we  denote $y_p$ as $y_p=(j_p+l_p)\cdot\delta_p$, where $l_p=(j_q+l_q)\cdot2^{p-q}-j_p\in\mathbb{Z}$. 
On the one hand, we have $l_p\geq0$ since the following holds
\begin{equation*}
    \frac{(j_q+l_q)\cdot2^{p-q}}{j_p}\geq \frac{(j_p\cdot2^{q-p}+0)\cdot2^{p-q}}{j_p}=1.
\end{equation*}
On the other hand, according to Lemma \ref{lemma:ceilfloor}, we have $j_q=\left\lfloor (j_p-1)\cdot2^{q-p}\right\rfloor\leq(j_p-1)\cdot2^{q-p}$. Note that 
\begin{equation*}
\begin{aligned}
    l_p-(2K_p-1)&= \left((j_q+l_q)\cdot2^{p-q}-j_p\right) -(2K_p-1) \\
    &\leq \left((j_p-1)\cdot2^{q-p}+1+2K_q-1\right) \cdot2^{p-q}-j_p-2K_p+1 \\
    &=j_p-1+2\cdot2^q\cdot2^{p-q}-j_p-2\cdot2^p+1 =0.
\end{aligned}
\end{equation*}
Thus $l_p\in\{0,\cdots,2K_p-1\}$. Moreover, note that
\begin{equation*}
    y_q-y_p= (j_q+l_q)\cdot\frac{\epsilon}{2^q}-\left(j_p+(j_q+l_q)\cdot2^{p-q}-j_p\right) \cdot\frac{\epsilon}{2^p}=0,
\end{equation*}
in other words, $y_p=y_q$ and $y_p\in\mathcal{Y}_{K_p}^{(j_p)}$. This yields that $\mathcal{Y}_{K_q}^{(j_q)} \subseteq \mathcal{Y}_{K_p}^{(j_p)}$.
\end{proof}

\textbf{Proof of Theorem \ref{theo:appromono}.} Rewrite $\widetilde{\mathrm{MCDP}}(\epsilon;K)$ as the following form:
\begin{equation}\label{eq:appro_rewrite}
\begin{aligned}
    \widetilde{\mathrm{MCDP}}(K;\epsilon)&= \max\{\underbrace{\min\limits_{j\in\{0,\cdots,K\}} \Delta\hat{F}(j\cdot\frac{\epsilon}{K})}_{M_1(K)}, \underbrace{\max\limits_{l\in\{1,\cdots,\left\lceil\frac{K}{\epsilon}\right\rceil-2K\}} \min\limits_{k\in\{l,\cdots,l+2K-1\}} \Delta\hat{F}(k\cdot\frac{\epsilon}{K})}_{M_2(K)}\} \\
    &=\max\left(M_1(K),M_2(K)\right).
\end{aligned}
\end{equation}
To prove the theorem, we show that for any $p>q\ge0,\ p,q\in\mathbb{N}$, $M_1(K_p)\leq M_1(K_q)$ and $M_2(K_p)\leq M_2(K_q)$ holds.

\ding{172} \emph{Proof of $M_1(K_p)\leq M_1(K_q)$}: 
For any $j_q=0,\cdots,K_q$, there exists $j_p=j_q\cdot2^{p-q}\in[0,K_p],\ j_p\in\mathbb{N}$, such that $j_q\delta_q= j_p\delta_p$. By Lemma \ref{lemma:set_maxmin}, we have
\begin{equation}\label{eq:m1p>m1q}
\begin{aligned}
    \{\Delta\hat{F}(j\delta_q)\}_{j=0}^{K_q} \subseteq \{\Delta\hat{F}(j\delta_p)\}_{j=0}^{K_p}\ \Rightarrow\ M_1(K_q)\geq M_1(K_p).
\end{aligned}
\end{equation}

\ding{173} \emph{Proof of $M_2(K_p)\leq M_2(K_q)$}: 
According to Lemma \ref{lemma:appro_mono}, for any $j_p=1,\cdots,\left\lceil\frac{1}{\delta_p}\right\rceil-2K_p$, there exists $j_q^{(p)}=\left\lceil j_p\cdot2^{q-p}\right\rceil \in\{1,\cdots,\left\lceil\frac{1}{\delta_q}\right\rceil-2K_q\}$, such that $\mathcal{Y}_{K_q}^{(j_q)} \subseteq \mathcal{Y}_{K_p}^{(j_p)}$. By Lemma \ref{lemma:set_maxmin}, we have
\begin{equation*}
    \min\limits_{k\in\{j_q^{(j_p)},\cdots,j_q^{(j_p)}+2K_q-1\}} \Delta\hat{F}(k\epsilon_q) \geq \min_{k\in\{j_p,\cdots,j_p+2K_p-1\}} \Delta\hat{F}(k\epsilon_p).
\end{equation*}
Note that $\left\{j_q^{(j_p)}:j_p\in\left\{1,\cdots,\left\lceil\frac{1}{\delta_p}\right\rceil-2K_p\right\}\right\} \subseteq \left\{1,\cdots,\left\lceil\frac{1}{\delta_q}\right\rceil-2K_q\right\}$. Using Lemma \ref{lemma:set_maxmin} again, we have
\begin{equation}\label{eq:m2p>m2q}
\begin{aligned}
    \max\limits_{j_q\in\{1,\cdots,\left\lceil\frac{1}{\delta_q}\right\rceil-2K_q\}} \min\limits_{k\in\{j_q,\cdots,j_q+2K_q-1\}} \Delta\hat{F}(k\epsilon_q) &\geq \max\limits_{j_p\in\{1,\cdots,\left\lceil\frac{1}{\delta_p}\right\rceil-2K_p\}} \min_{k\in\{j_p,\cdots,j_p+2K_p-1\}} \Delta\hat{F}(k\epsilon_p), \\
    \Rightarrow\quad M_2(K_q)&\geq M_2(K_p).
\end{aligned}
\end{equation}

Combining Eq.~\eqref{eq:appro_rewrite}, Eq.~\eqref{eq:m1p>m1q}, and Eq.~\eqref{eq:m2p>m2q}, we obtain that
\begin{equation*}
    \widetilde{\mathrm{MCDP}}(\epsilon;K_p) =\max\{M_1(K_p),M_2(K_p)\} \leq\max\{M_1(K_q),M_2(K_q)\} \leq\widetilde{\mathrm{MCDP}}(\epsilon;K_q).
\end{equation*}
This completes the proof of Theorem \ref{theo:appromono}.

\end{proof}

\subsection{Proofs of Theorem \ref{theo:tempsigmoid}}

\textbf{Theorem \ref{theo:tempsigmoid}. }
$\Delta\Tilde{F}_{\tau}(\hat{y}) \overset{\text{a.e.}}{\longrightarrow} \Delta\hat{F}(\hat{y})$ as $\tau\to\infty$, and in particular, $\lim_{\tau\to\infty} \Delta\Tilde{F}_{\tau}(\hat{y})= \Delta\hat{F}(\hat{y}),\ \forall\hat{y}\notin\{\hat{y}_i\}_{i=1}^N$.
\begin{proof}
Note that for any $i=1,\cdots,N$, we have
\begin{equation*}
    \lim\limits_{\tau\to\infty} \sigma_{\tau}(\hat{y}-\hat{y}_i)= \frac{1}{1+\exp(-\tau(\hat{y}-\hat{y}_i))}=
    \begin{cases}
        1,& \text{if }\hat{y}_i<\hat{y}, \\
        \frac{1}{2},& \text{if }\hat{y}_i=\hat{y}, \\
        0,& \text{if }\hat{y}_i>\hat{y}.
    \end{cases}
\end{equation*}
Thus for any $\hat{y}\notin\{\hat{y}_i\}_{i=1}^N$, the following holds
\begin{equation}
    \lim_{\tau\to\infty} \sum\limits_{i\in\mathcal{S}_a} \sigma_{\tau}(\hat{y}-\hat{y}_i)= \sum\limits_{i\in\mathcal{S}_a} \lim_{\tau\to\infty} \sigma_{\tau}(\hat{y}-\hat{y}_i)= \sum\limits_{i\in\mathcal{S}_a} \mathbb{I}(\hat{y}_i\leq\hat{y}),\ a\in\{0,1\}.
\end{equation}
Then we have
\begin{equation}
    \lim\limits_{\tau\to\infty} \Delta\Tilde{F}_{\tau}(\hat{y})= \left|\frac{1}{|\mathcal{S}_0|} \sum\limits_{i\in\mathcal{S}_0} \mathbb{I}(\hat{y}_i\leq\hat{y})- \frac{1}{|\mathcal{S}_1|} \sum\limits_{i\in\mathcal{S}_1} \mathbb{I}(\hat{y}_i\leq\hat{y})\right| =|\hat{F}_0(\hat{y})-\hat{F}_1(\hat{y})|=\Delta\hat{F}(\hat{y}),\quad\forall\hat{y}\notin\{\hat{y}_i\}_{i=1}^N.
\end{equation}
Note that $\{\hat{y}_i\}_{i=1}^N$ has measure zero, which further yields that $\Delta\Tilde{F}_{\tau}(\hat{y}) \overset{\text{a.e.}}{\longrightarrow} \Delta\hat{F}(\hat{y})$.
\end{proof}

\section{Estimation Error Analysis of $\widehat{\mathrm{MCDP}}(\epsilon)$ Metric}\label{sec:tractability} 
To conduct analysis of estimation error of the $\widehat{\mathrm{MCDP}}(\epsilon)$ metric, we start with giving two lemmas that will be used in proofs.
\begin{lemma}\label{lemma:abs_ineq1}
    For real numbers $a,b,c,d\in\mathbb{R}$, $\left||a-b|-|c-d|\right|\leq|a-c|+|b-d|$ holds.
\end{lemma}
\begin{lemma}\label{lemma:abs_ineq2}
    For $a,b,c,d\in\mathbb{R}$ where $a\leq b$ and $c\leq d$, $\max\{|a-c|,|b-d|\}\leq\max\{|b-c|,|a-d|\}$ holds.
\end{lemma}
\begin{proof}
    If $b\leq d$, then we have $|b-d|=d-b\leq d-a\leq\max\{|b-c|,|a-d|\}$. Likewise, $d\leq b$ derives that $|d-b|=b-d\leq b-c\leq\max\{|b-c|,|a-d|\}$. Thus the Lemma \ref{lemma:abs_ineq2} gets proved.
\end{proof}

Next we provide the statement of the Glivenko–Cantelli theorem as below:
\begin{theorem}[Glivenko-Cantelli]\label{theo:glivenko}\cite{tucker1959generalization}
    Let $X_1,\cdots,X_n$ be \emph{i.i.d.} random variables in $\mathbb{R}$ with common cumulative distribution function $F(x)=\mathbb{P}(X_1\leq x)$. Let $\hat{F}_n(x)=\frac{1}{n}\sum\limits_{i=1}^n\mathbb{I}(X_i\leq x)$ be the empirical distribution function. Then
    \begin{equation}\label{eq:glivenko_bound}
        \mathbb{P}\left\{\sup\limits_{x\in\mathbb{R}} |F(x)-\hat{F}_n(x)|>\delta\right\} \leq8(n+1)e^{-\frac{n\delta^2}{32}}.
    \end{equation}
    In particular, by the Borel-Cantelli Lemma \cite{chung1952application}, $\hat{F}_n(x)\stackrel{\text{a.s.}}{\longrightarrow}F(x),\ n\to\infty$, \ie
    \begin{equation*}
        \lim_{n\to\infty}\sup\limits_{x\in\mathbb{R}}|F(x)-\hat{F}_n(x)|=0.
    \end{equation*}
\end{theorem}

By Theorem \ref{theo:glivenko}, we have the following corollaries in the binary classification task's setting of this work:
\begin{corollary}\label{lemma:CDF_as}
    Suppose the instances $\{\boldsymbol{x}_i\}_{i=1}^N$ in $\mathcal{D}$ are \emph{i.i.d.}, then
    \begin{equation*}
        \hat{F}_{a,n_a}(\hat{y}) \stackrel{\text{a.s.}}{\longrightarrow} F_a(\hat{y}),\ n_a\to\infty,\ a=0,1,
    \end{equation*}
    where $n_a$ denotes the number of instances with group label $a$ in $\mathcal{D}$ (\ie $n_a=|\mathcal{S}_a|$), and $\hat{F}_{a,n_a}(\hat{y})$ is the empirical distribution function of predictions of group $a$ when estimated by $N_a$ samples.
\end{corollary}

\begin{corollary}\label{cor:delta_as}
    Denote $\Delta\hat{F}_{n_0,n_1}(\hat{y})= |\hat{F}_{0,n_0}(\hat{y})-\hat{F}_{1,n_1}(\hat{y})|$, then we have
    \begin{equation}
        \Delta\hat{F}_{n_0,n_1}(\hat{y}) \stackrel{\text{a.s.}}{\longrightarrow} \Delta F(\hat{y}),\ n_0\to\infty,\ n_1\to\infty.
    \end{equation}
\end{corollary}
\begin{proof}
According to Lemma \ref{lemma:CDF_as}, for any $\delta>0$ and $a\in\{0,1\}$, there exists $N_a(\delta)>0$, such that
\begin{equation}\label{eq:CDF_epsdelta}
    \max\limits_{\hat{y}\in[0,1]}|\hat{F}_{a,n_a}(\hat{y})-F_a(\hat{y})|\leq\delta,\quad\text{for any }n_a>N_a(\epsilon).
\end{equation}
Let $n_0>N_0(\frac{\delta}{2})$ and $n_1>N_1(\frac{\delta}{2})$, we have
\begin{equation}\label{eq:deltaF_epsdelta}
\begin{aligned}
    &\ \max\limits_{\hat{y}\in[0,1]} |\Delta\hat{F}_{n_0,n_1}(\hat{y}) -\Delta F(\hat{y})| \\
    =&\ \max\limits_{\hat{y}\in[0,1]} \left||\hat{F}_{0,n_0}(\hat{y})-\hat{F}_{1,n_1}(\hat{y})| -|F_0(\hat{y})-F_1(\hat{y})|\right| \\
    \stackrel{\text{(a)}}{\leq}& \max\limits_{\hat{y}\in[0,1]} \left(|\hat{F}_{0,n_0}(\hat{y})-\hat{F}_0(\hat{y})| +|\hat{F}_{1,n_1}(\hat{y})-\hat{F}_1(\hat{y})|\right) \\
    \leq&\ \max\limits_{\hat{y}\in[0,1]} |\hat{F}_{0,n_0}(\hat{y})-\hat{F}_0(\hat{y})| +\max\limits_{\hat{y}\in[0,1]} |\hat{F}_{1,n_1}(\hat{y})-\hat{F}_1(\hat{y})| \\
    \stackrel{\text{(b)}}{\leq}&\ \frac{\delta}{2}+\frac{\delta}{2}=\delta,
\end{aligned}
\end{equation}
where inequality (a) holds due to Lemma \ref{lemma:abs_ineq1}, and inequality (b) holds due to Eq.~\eqref{eq:CDF_epsdelta}. By Eq.~\eqref{eq:deltaF_epsdelta}, we have
\begin{equation*}
\lim\limits_{\substack{n_0\to\infty\\n_1\to\infty}}\sup\limits_{\hat{y}\in[0,1]}\left|\Delta\hat{F}_{n_0,n_1}(\hat{y}) -\Delta F(\hat{y})\right|=0,
\end{equation*}
which yields that $\Delta\hat{F}_{n_0,n_1}(\hat{y}) \stackrel{\text{a.s.}}{\longrightarrow} \Delta F(\hat{y})$.
\end{proof}

For notation clarity, let $M(\epsilon)=\mathrm{MCDP}(\epsilon)$ and $\hat{M}_{n_0,n_1}(\epsilon)= \widehat{\mathrm{MCDP}}(\epsilon)= \max\limits_{y_0\in[0,1]} \min\limits_{\left|\hat{y}-y_0\right|\leq\epsilon} \Delta\hat{F}_{n_0,n_1}(\hat{y})$. 
We have the following theorem about the estimation error convergence rate of the empirical and true metrics:
\begin{theorem}[Convergence of $\widehat{\mathrm{MCDP}}(\epsilon)$]\label{theo:mcdpeps_conv}
As the sample size $n_0$ and $n_1$ of two demographic groups increase, the value of $\widehat{\mathrm{MCDP}}(\epsilon)$ metric converges with probability one to the true metric value $\mathrm{MCDP}(\epsilon)$, \ie
\begin{equation}\label{eq:mcdpeps_as}
    \hat{M}_{n_0,n_1}(\epsilon) \stackrel{\text{a.s.}}{\longrightarrow} M(\epsilon),\ n_0\to\infty,\ n_1\to\infty.
\end{equation}
Additionally, the estimation error satisfies
\begin{equation}\label{eq:mcdpeps_err}
    Err_{\text{emp}}= \mathbb{E}_{\mathcal{D}}\left[|\hat{M}_{n_0,n_1}(\epsilon)-M(\epsilon)|^2\right]= \mathcal{O}\left(\frac{\ln n}{n}\right),\ n=\min\{n_0,n_1\}.
\end{equation}
\end{theorem}

\begin{proof}
Let $y^*=\mathop{\arg\max}_{y^*\in[0,1]}\Delta F(\hat{y})$ and $\hat{y}^*=\mathop{\arg\max}_{y^*\in[0,1]}\Delta\hat{F}_{n_0,n_1}(\hat{y})$, then both \ding{172} $\Delta F(\hat{y}^*)\leq\Delta F(y^*)$ and \ding{173} $\Delta\hat{F}_{n_0,n_1}(y^*)\leq\Delta\hat{F}_{n_0,n_1}(\hat{y}^*)$ hold. By Lemma \ref{lemma:abs_ineq2} and Eq.~\eqref{eq:deltaF_epsdelta}, given $\delta>0$, for any $n_0>N_0(\delta),n_1>N_1(\delta)$, we have
\begin{equation}\label{eq:conv_eps=0}
\begin{aligned}
    |M(0)-\hat{M}_{n_0,n_1}(0)|&= |\Delta F(y^*)-\Delta\hat{F}_{n_0,n_1}(\hat{y}^*)| \\
    &\leq \max\left\{|\Delta F(y^*)-\Delta\hat{F}_{n_0,n_1}(y^*)|, |\Delta F(\hat{y}^*)-\Delta\hat{F}_{n_0,n_1}(\hat{y}^*)|\right\} \\
    &\leq \max\limits_{\hat{y}\in[0,1]} |\Delta F(\hat{y})-\Delta\hat{F}_{n_0,n_1}(\hat{y})| \leq\delta.
\end{aligned}
\end{equation}
As to the cases when $\epsilon>0$, we denote $G(y_0),\ y_0^*,\ \hat{G}_{n_0,n_1}(y_0),\ \widehat{y_0^*}$ as follows
\begin{equation*}
\begin{aligned}
    G(y_0)=& \min\limits_{|\hat{y}-y_0|\leq\epsilon} \Delta F(\hat{y}), &  y_0^*&= \mathop{\arg\max}\limits_{y_0\in[0,1]}G(y_0), \\
    \hat{G}_{n_0,n_1}(y_0)=& \min\limits_{|\hat{y}-y_0|\leq\epsilon} \Delta\hat{F}_{n_0,n_1}(\hat{y}), & \widehat{y_0^*}&= \mathop{\arg\max}\limits_{y_0\in[0,1]}\hat{G}_{n_0,n_1}(y_0).
\end{aligned}
\end{equation*}
Thus we have $G(\widehat{y_0^*})\leq G(y_0^*)$ and $\hat{G}_{n_0,n_1}(y_0^*)\leq \hat{G}_{n_0,n_1}(\widehat{y_0^*})$. Moreover, for any $y_0^{\prime}\in[0,1]$, there exists $\hat{y}^{\prime},\hat{y}^{\prime\prime}\in[y_0^{\prime}-\epsilon,y_0^{\prime}+\epsilon]\cap[0,1]$, such that $G(y_0^{\prime})= \Delta F(\hat{y}^{\prime}) \leq\Delta F(\hat{y}^{\prime\prime})$ and $\hat{G}(y_0^{\prime})= \Delta\hat{F}_{n_0,n_1}(\hat{y}^{\prime\prime}) \leq\Delta\hat{F}_{n_0,n_1}(\hat{y}^{\prime})$. By Lemma \ref{lemma:abs_ineq2} and Eq.~\eqref{eq:deltaF_epsdelta}, given $\delta>0$ and $\epsilon>0$, for any $n_0>N_0(\delta),n_1>N_1(\delta)$, we have
\begin{equation}\label{eq:conv_eps>0}
\begin{aligned}
    |M(\epsilon)-\hat{M}_{n_0,n_1}(\epsilon)|&= |G(y_0^*)-\hat{G}_{n_0,n_1}(\widehat{y_0^*})| \\
    &\leq \max\left\{|G(y_0^*)-\hat{G}_{n_0,n_1}(y_0^*)|, |G(\widehat{y_0^*})-\hat{G}_{n_0,n_1}(\widehat{y_0^*})|\right\} \\
    &\leq \max\limits_{y_0\in[0,1]} |G(y_0)-\hat{G}_{n_0,n_1}(y_0)| \stackrel{\text{(c)}}{=} |\Delta F(\hat{y}^{\prime})- \Delta\hat{F}_{n_0,n_1}(\hat{y}^{\prime\prime})| \\
    &\leq \max\left\{|\Delta F(\hat{y}^{\prime\prime})- \Delta\hat{F}_{n_0,n_1}(\hat{y}^{\prime\prime})|, |\Delta F(\hat{y}^{\prime})- \Delta\hat{F}_{n_0,n_1}(\hat{y}^{\prime})|\right\} \\
    &\leq \max\limits_{\hat{y}\in[0,1]} |\Delta F(\hat{y})- \Delta\hat{F}_{n_0,n_1}(\hat{y})| \leq\delta,
\end{aligned}
\end{equation}
Equality (c) holds by taking $y_0^{\prime}=\mathop{\arg\max}\limits_{y_0\in[0,1]} |G(y_0)-\hat{G}_{n_0,n_1}(y_0)|$. Combining Eq.~\eqref{eq:conv_eps=0} and Eq.~\eqref{eq:conv_eps>0} yields that $\hat{M}_{n_0,n_1}(\epsilon)$ converges to $M(\epsilon)$ almost surely (\ie Eq.~\eqref{eq:mcdpeps_as} holds). Furthermore, let Eq.~\eqref{eq:glivenko_bound} equal to $1$, we have $\delta=\mathcal{O}\left(\sqrt{\frac{\ln n}{n}}\right)$. Combining Eq.~\eqref{eq:conv_eps=0}, Eq.~\eqref{eq:conv_eps>0} and Eq.~\eqref{eq:deltaF_epsdelta}, for any $\epsilon\geq0$, the estimation error should satisfy that
\begin{equation*}
\begin{aligned}
    Err_{\text{emp}}&= \mathbb{E}_{\mathcal{D}}\left[|\hat{M}_{n_0,n_1}(\epsilon)-M(\epsilon)|^2\right] \\
    &\leq \left(\max\limits_{\hat{y}\in[0,1]} |\Delta F(\hat{y})- \Delta\hat{F}_{n_0,n_1}(\hat{y})|\right)^2 \\
    &\leq \left(\max\limits_{\hat{y}\in[0,1]} |\hat{F}_{0,n_0}(\hat{y})-\hat{F}_0(\hat{y})| +\max\limits_{\hat{y}\in[0,1]} |\hat{F}_{1,n_1}(\hat{y})-\hat{F}_1(\hat{y})|\right)^2 \\
    &= \left(\mathcal{O}\left(\sqrt{\frac{\ln n_0}{n_0}}\right) +\mathcal{O}\left(\sqrt{\frac{\ln n_1}{n_1}}\right)\right)^2 \\
    &= \mathcal{O}\left(\frac{\ln n}{n}\right),\ n=\min\{n_0,n_1\}.
\end{aligned}
\end{equation*}
Thus Eq.~\eqref{eq:mcdpeps_err} gets proved. This completes the proof of Theorem \ref{theo:mcdpeps_conv}.
\end{proof}

\section{Detailed Analysis about the Computational Complexity}\label{sec:complexity}

In Section \ref{sec:estimate}, we conduct a rough computational complexity analysis of the exact and approximate $\widehat{\mathrm{MCDP}}(\epsilon)$ calculation algorithms, where we only focus on the traverse strategies of $y_0$ and $\hat{y}$ in Eq.~\eqref{eq:mcdphat}. Here we provide a more thorough computational complexity analysis of both algorithms.

\textbf{Exact Calculation Analysis. }
In Algorithm \ref{alg:exact}, the calculation process can be divided into the following steps.
\begin{enumerate}[label=(\arabic*)]
    \item \textbf{Calculating $\Delta\hat{F}(\hat{y})$ (line 1).} As $\Delta\hat{F}(\hat{y})= |\hat{F}_0(\hat{y})- \hat{F}_1(\hat{y})|$, this step is equivalent to calculating the empirical distribution function $\hat{F}_a(\hat{y})$ using the predictions $\{\hat{y}_i\}_{i\in\mathcal{S}_a}$ in two groups ($a=0,1$). Due to its step-like pattern’s property, the empirical distribution function is implemented by constructing a sorted array of sample values with the corresponding cumulative ratio in practice (\eg the Statsmodels library \cite{seabold2010statsmodels} in Python), and then perform binary-search on the array given a specific value. In line 1, the complexity of initializing the sorted array structure is $\mathcal{O}(|\mathcal{S}_0|\log|\mathcal{S}_0|)+ \mathcal{O}(|\mathcal{S}_1|\log|\mathcal{S}_1|)= \mathcal{O}(N\log N)$.
    \item \textbf{Initializing the $\widetilde{\mathrm{MCDP}}(\epsilon)$ (lines 2-6).} When $\epsilon=0$, the algorithm returns the maximal $\Delta\hat{F}(\hat{y})$ value on predictions of data samples, and its computational complexity is $\mathcal{O}(N)$. In cases where $\epsilon>0$, the $\widetilde{\mathrm{MCDP}}(\epsilon)$ value is initialized with the minimal $\Delta\hat{F}(\hat{y})$ prediction smaller than or equal to $\epsilon$, whose complexity is also $\mathcal{O}(N)$.
    \item \textbf{Traversing $y_0$ and $\hat{y}$ when $\epsilon>0$ (lines 7-12).} For each prediction $\hat{y}_i\leq1-\epsilon$, the algorithm finds the minimal $\Delta\hat{F}(\hat{y})$ value of predictions in the $2\epsilon$ interval with left endpoint $\hat{y}_i$, and finally keeps the maximum of the selected minimums. As the above process involves a double traversal of predictions in $\mathcal{D}$, its complexity is $\mathcal{O}(N^2)$.
\end{enumerate}

In summary, the overall computational complexity when $\epsilon=0$ is $\mathcal{O}(N\log N) +\mathcal{O}(N) =\mathcal{O}(N\log N)$. When $\epsilon>0$, the total complexity is $\mathcal{O}(N\log N) +\mathcal{O}(N) +\mathcal{O}(N^2) =\mathcal{O}(N^2)$, most of which is made up by the traverse process.

\textbf{Approximate Calculation Analysis. }
Similar to the exact calculation, Algorithm \ref{alg:approximate} firstly calculates $\Delta\hat{F}(\hat{y})$ using the predictions of data samples in two groups (line 1), whose computational complexity is $\mathcal{O}(N\log N)$. Next, the algorithm samples equally-spaced prediction points by step-size $\delta=\frac{\epsilon}{K}$ on $[0,1]$, where the total number of samples is $\mathcal{O}(\frac{1}{\delta})$. Afterwards, it firstly initializes $\widetilde{\mathrm{MCDP}}(\epsilon)$ with the minimal $\Delta\hat{F}(\hat{y})$ value of the first $K+1$ sampled points (line 3), and then updates it with the maximum of the minimal $\Delta\hat{F}(\hat{y})$ value of all consecutive $2K$ sampled points (lines 4-7). The computational complexity of initialization and traverse update are $\mathcal{O}(K)$ and $\mathcal{O}(\frac{K}{\delta})$, respectively. To conclude, the total complexity of Algorithm \ref{alg:approximate} is $\mathcal{O}(N\log N)+ \mathcal{O}(\frac{1}{\delta})+ \mathcal{O}(K)+ \mathcal{O}(\frac{K}{\delta})= \max\{\mathcal{O}(N\log N),\mathcal{O}(\frac{K^2}{\epsilon})\}$. Specifically, when the sampling frequency $K$ is set with small values such that $K\ll N$, the calculation of $\Delta\hat{F}(\hat{y})$ will make up more complexity than traversing sampled predictions, and its total complexity will be much smaller than that of exact algorithm.

\section{Additional Experimental Settings}\label{sec:moresetup}

\textbf{Implementation Details. }
To make our work standardized and extensible, the implementations are based on a latest open-sourced fairness benchmark, FFB \cite{han2023ffb}. We conduct experiments with a 96-core Intel CPU (Intel(R) Xeon(R) Platinum 8268 @ 2.90GHz * 2) and a Nvidia-2080Ti GPU (11 GB memory). For tabular datasets Adult \cite{kohavi1996scaling} and Bank \cite{moro2014data}, we use a two-layer MLP \cite{bishop1995neural} with 256 hidden neurons and ReLU activation function \cite{nair2010rectified} as the classifier; for the image dataset CelebA \cite{liu2015deep} (which derives two meta-datasets CelebA-A and CelebA-W), we adopt Resnet-18 \cite{he2016deep} with as the backbone model, and initialize it with pretrained weights. The batch size for tabular and image datasets are set as $1024$ and $128$, respectively, and the total training step is set as $150$. We use the Adam optimizer with initial learning rate $0.001$, which is decayed by the piecewise strategy (\ie StepLR scheduler in Pytorch \cite{paszke2019pytorch}) during training. We use average precision ($\mathrm{AP}$) to evaluate classification accuracy, and adopt different metrics ($\Delta\mathrm{DP},\ \mathrm{ABCC},\ \mathrm{MCDP}(\epsilon)$) to measure algorithmic fairness. For $\Delta\mathrm{DP}$, we follow \citet{creager2019flexibly,dai2021say,chen2023post} to compute the difference of positive prediction proportion of two groups. The code implementation of this paper is available at \url{https://github.com/mitao-cat/icml24_mcdp}.

\begin{table*}[t]
\vspace{-8pt}
\caption{Hyper-parameter settings for different methods and datasets.}
\label{tab:hyper}
\vspace{1pt}
\centering
\begin{tabular}{@{}lcl@{}}
\toprule
\textbf{Method}            & \textbf{Other HP}    & \textbf{Selected 5 Trade-off HP Values on Each Dataset}                     \\ \midrule
\multirow{2}{*}{AdvDebias} & \multirow{2}{*}{N/A} & Adult: $[0.2,0.4,0.6,0.8,1.0]$;\quad Bank: $[0.3,0.6,0.9,1.2,1.5]$;      \\ 
                           &                      & CelebA-A, CelebA-W: $[0.5,1.0,1.5,2.0,2.5]$ \\ \midrule
DiffDP                     & N/A                  & All datasets: $[0.1,0.2,0.3,0.4,0.5]$                                    \\ \midrule
\multirow{2}{*}{FairMixup} & \multirow{2}{*}{N/A} & Adult: $[0.1,0.2,0.3,0.4,0.5]$;\quad Bank: $[0.3,0.4,0.5,0.6,0.7]$       \\ 
                           &                      & CelebA-A: $[0.5,1,2,3,4]$;\quad CelebA-W: $[0.5,1.0,1.5,2.0,2.5]$        \\ \midrule
DRAlign                    & Alignment strength $\beta$  & All datasets: $[0.1,0.2,0.3,0.4,0.5]$                                    \\ \midrule
DiffABCC                   & Temperature $\tau$        & All datasets: $[0.1,0.2,0.3,0.4,0.5]$                                    \\ \midrule
DiffMCDP                   & Temperature $\tau$        & All datasets: $[0.1,0.15,0.2,0.25,0.3]$                                  \\ \bottomrule
\end{tabular}
\vspace{-10pt}
\end{table*}

\textbf{Hyper-Parameters of Baselines. }
As fairness-accuracy trade-off is a widely-exist phenomenon in various applications \cite{wick2019unlocking,dutta2020there}, the hyper-parameters of different fairness algorithms should be carefully selected to ensure reliable comparisons. Generally, all hyper-parameters can be divided into two categories based on whether they directly control the fairness and accuracy (\emph{trade-off HPs}) or not (\emph{other HPs}). For each compared method, we firstly perform grid search on a wide value range for both trade-off HPs and other HPs, and then select the optimal other HP value that achieves the best trade-off performance on each dataset. Afterwards, we select 5 trade-off HP values that can well reflect the fairness-accuracy trade-off trend for each algorithm on the validation set. We additionally require that (i) both fairness and accuracy metrics of different algorithms are within a similar value range, and (ii) the classification accuracy should not drop too sharply. When using a single value for model evaluation, following previous work \cite{jung2023re}, we select the optimal trade-off HP value which achieves at least 95\% of the vanilla model’s accuracy (\ie $\mathrm{AP}$ of ERM) on validation set. We summarize the detailed hyper-parameters in Table \ref{tab:hyper}.

\section{Supplementary Experimental Results}

\subsection{Fairness-Accuracy Trade-offs}\label{sec:moretradeoff}

In Section \ref{sec:exp_learning}, we show the trade-offs between $\mathrm{AP}$ and $\mathrm{MCDP}(0)$ metrics of baselines and our proposed method. To further investigate their performances, we also plot the fairness-accuracy trade-off curves when adopting $\Delta\mathrm{DP}$ and $\mathrm{ABCC}$ as fairness metrics in Figures \ref{fig:ap_dp} and \ref{fig:ap_abcc}, respectively. We can observe that DiffMCDP can achieve comparable or even better trade-off performances compared with other baselines, and in particular, DiffDP achieves the optimal results in terms of both $\Delta\mathrm{DP}$ and $\mathrm{ABCC}$ fairness metrics on Bank dataset. This indicates that optimizing the maximal disparity is also beneficial to minimize the overall disparity, which is consistent with the results in Table \ref{tab:main_tabular}.

\begin{figure*}[htb]
    \centering
    \vspace{-5pt}
    \centering
    \includegraphics[width=0.99\linewidth]{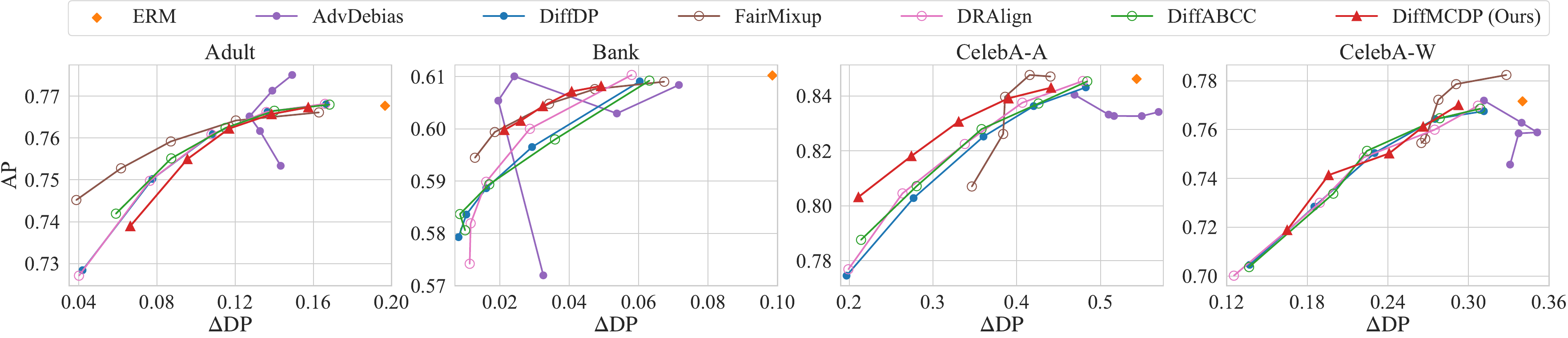}
    \vspace{-12pt}
    \caption{Trade-offs between $\mathrm{AP}$ and $\Delta\mathrm{DP}$ of baselines and the proposed method.}
    \label{fig:ap_dp}
    \vspace{-5pt}
\end{figure*}

\begin{figure*}[htb]
    \centering
    \vspace{-5pt}
    \centering
    \includegraphics[width=0.99\linewidth]{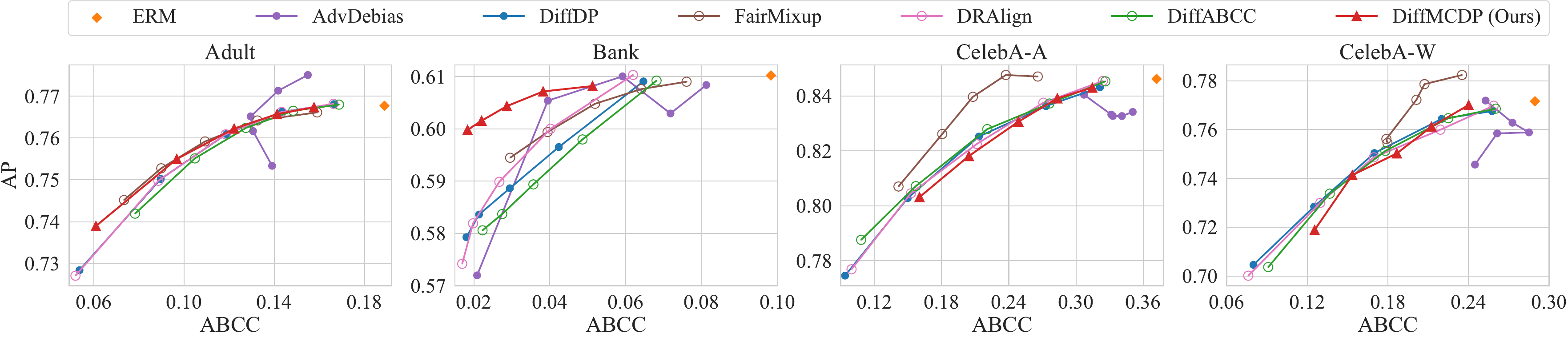}
    \vspace{-12pt}
    \caption{Trade-offs between $\mathrm{AP}$ and $\mathrm{ABCC}$ of baselines and the proposed method.}
    \label{fig:ap_abcc}
    \vspace{-10pt}
\end{figure*}

\subsection{Varying Local Measurements $\epsilon$}\label{sec:varyeps}

In Figure \ref{fig:MCDP_eps} of Section \ref{sec:exp_learning}, we report the different $\mathrm{MCDP}(\epsilon)$ results of each algorithm, in which the relative performance of some baselines changes with increasing $\epsilon$ values. To further explore the effect of varying neighborhood hyper-parameter, we plot the fairness-accuracy trade-off curves with $\epsilon$ in $\{0.01,0.05,0.1\}$. Figures \ref{fig:dp_mcdpeps_adult} and \ref{fig:dp_mcdpeps_bank} show the results of Adult and Bank datasets, respectively, from which we observe that DiffMCDP always outperforms other baselines in terms of different $\mathrm{MCDP}(\epsilon)$ metrics. Moreover, the trade-off patterns of some baselines may change with $\epsilon$ values, which brings about changing relative performances evaluated by $\mathrm{MCDP}(\epsilon)$. This further reminds practitioners to select varying $\epsilon$ values for more comprehensive evaluation of different fair algorithms.

\begin{figure*}[htb]
    \centering
    \includegraphics[width=0.78\linewidth]{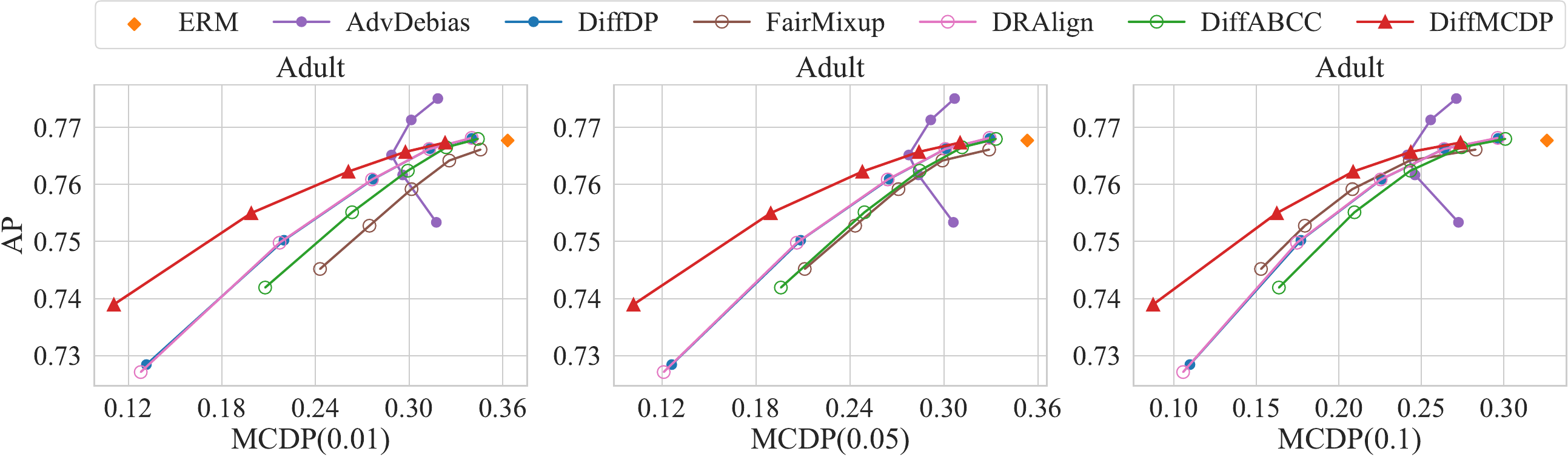}
    \vspace{-10pt}
    \caption{Trade-offs between $\mathrm{AP}$ and $\mathrm{MCDP}(\epsilon)$ with varying neighborhood hyper-parameter $\epsilon$ on Adult dataset.}
    \label{fig:dp_mcdpeps_adult}
\end{figure*}

\begin{figure*}[htb]
    \centering
    \includegraphics[width=0.78\linewidth]{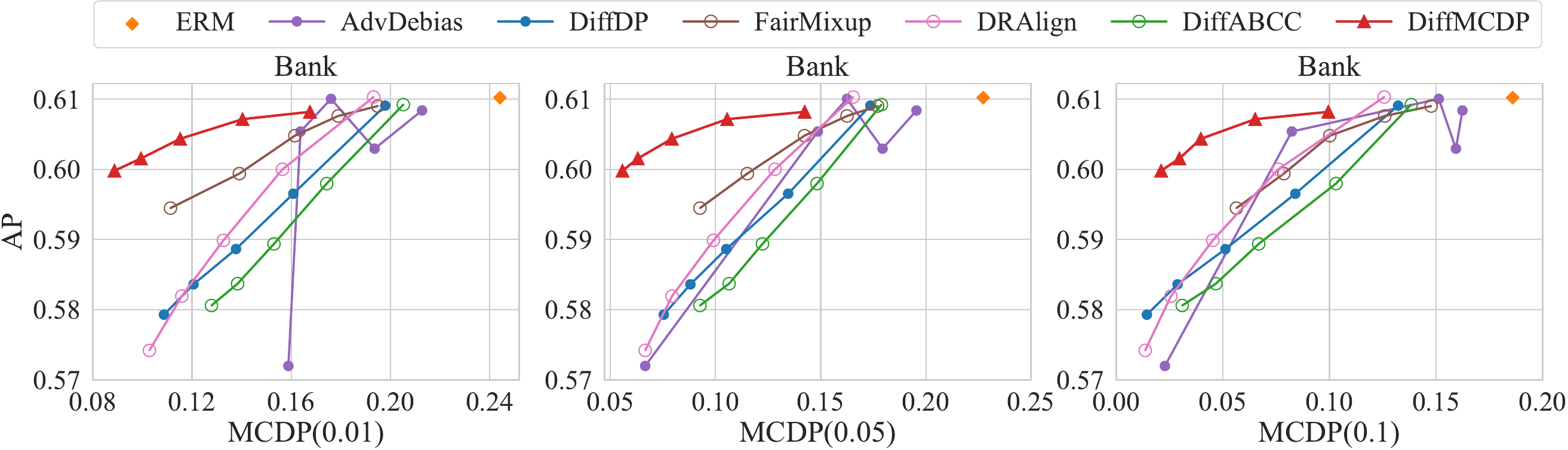}
    \vspace{-10pt}
    \caption{Trade-offs between $\mathrm{AP}$ and $\mathrm{MCDP}(\epsilon)$ with varying $\epsilon$ values on Bank dataset.}
    \label{fig:dp_mcdpeps_bank}
\end{figure*}

Additionally, we plot the $\mathrm{MCDP}(\epsilon)$ results with varying $\epsilon$ values on CelebA dataset in Figure \ref{fig:MCDP_eps_image}, which shows that our proposed DiffMCDP algorithm can consistently achieve lower maximal local disparity than other baselines. 
Additionally, compared to the results of the tabular data in Figure \ref{fig:MCDP_eps}, the differences in the magnitude of $\mathrm{MCDP}(\epsilon)$ values under varying $\epsilon$ are much lower in Figure \ref{fig:MCDP_eps_image}, and the relative performance of all methods does not change. We postulate that the CDFs of predictions of two groups on tabular datasets are sharper than those of image datasets, thus the value of neighborhood hyper-parameter $\epsilon$ is more prone to affect the value of maximal local disparity.

\begin{figure*}[!h]
    \centering
    \includegraphics[width=0.5\linewidth]{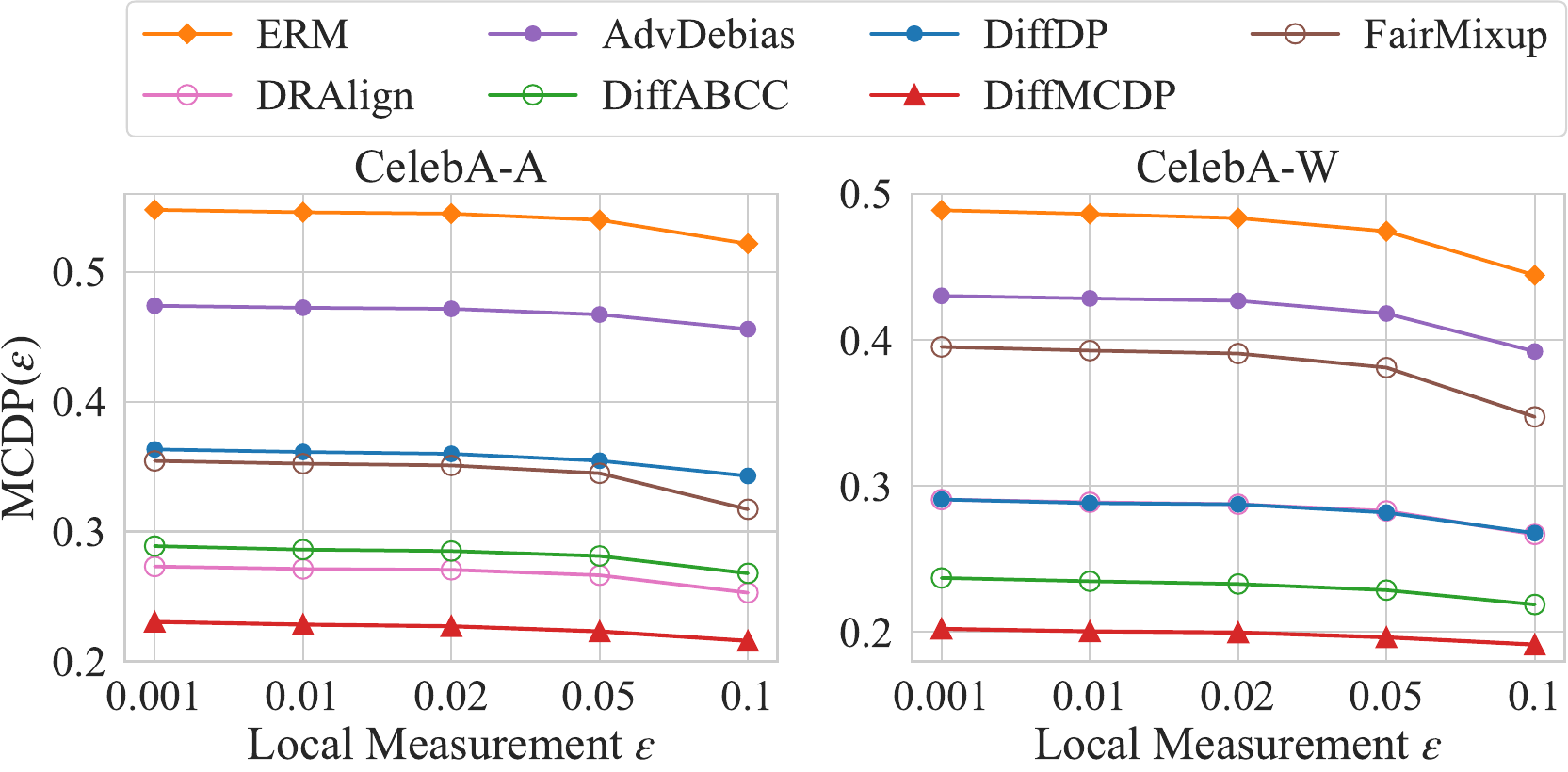}
    \vspace{-10pt}
    \caption{Comparison of $\mathrm{MCDP}(\epsilon)$ results with varying $\epsilon$ on image datasets.}
    \label{fig:MCDP_eps_image}
\end{figure*}

\subsection{$\widehat{\mathrm{MCDP}}(\epsilon)$ Calculation Algorithms}\label{sec:detailestimate}

Continuing from Section \ref{sec:expestimate}, we show the effect of sampling frequency $K$ and local measurement $\epsilon$ on the estimation accuracy and calculation efficiency of $\widehat{\mathrm{MCDP}}(\epsilon)$ calculation algorithms. The performance comparison on Bank and CelebA-W datasets are reported in Figures \ref{fig:bank_epsK} and \ref{fig:celeba-w_epsK}, respectively. As the exact $\widehat{\mathrm{MCDP}}(0.1)$ values of a small fraction of cases in Bank are very close to zero ($<0.001$), which results in extreme relative errors, we report the \emph{absolute error} $V_a-V_e$ in Figure \ref{fig:bank_epsK} instead. Similar to the observations from Figure \ref{fig:time_err}, the approximate algorithm can greatly improve the calculation efficiency with low error.

\begin{figure*}[htb]
    \centering
    \vspace{-5pt}
    \includegraphics[width=0.5\linewidth]{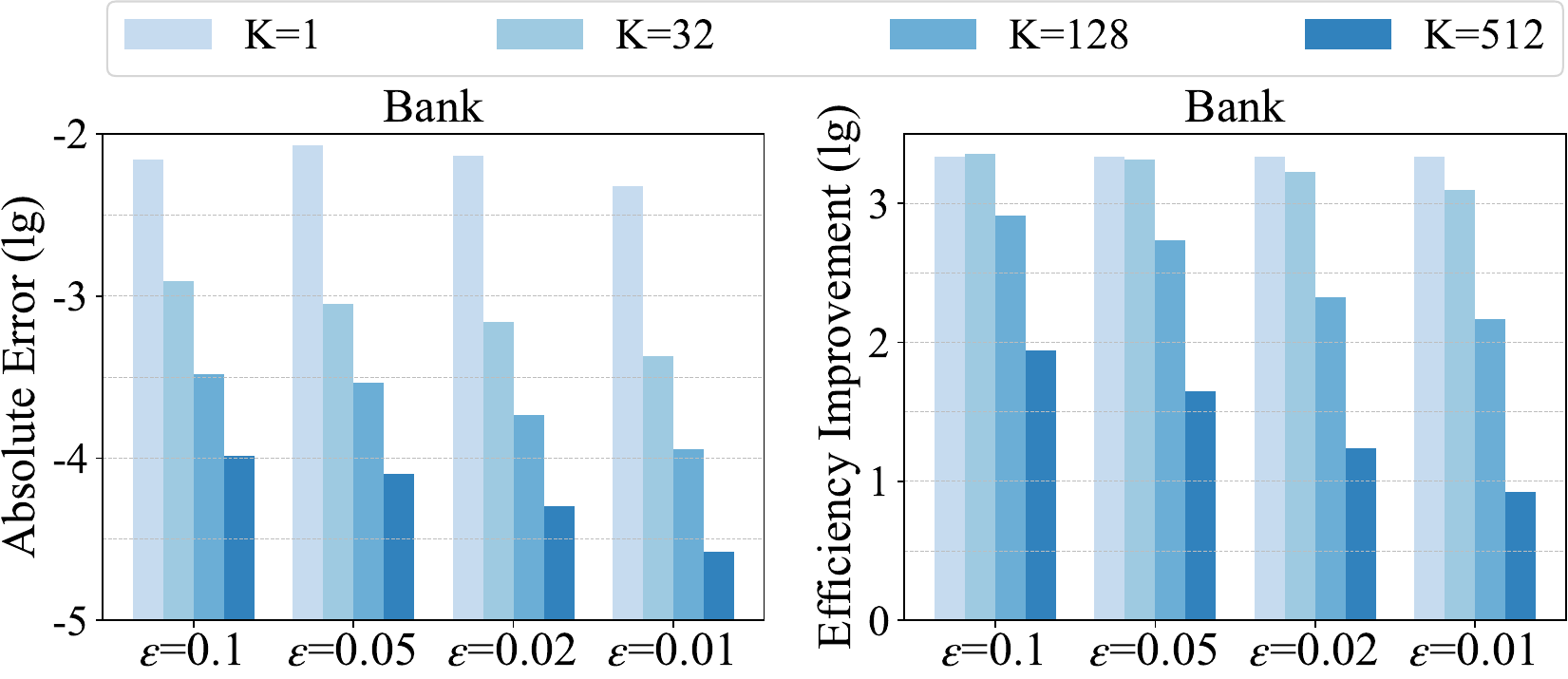}
    \vspace{-12pt}
    \caption{Varying $K$ and $\epsilon$ in $\widehat{\mathrm{MCDP}}(\epsilon)$ calculation algorithms on the Bank dataset.}
    \label{fig:bank_epsK}
    \vspace{-12pt}
\end{figure*}

\begin{figure*}[htb]
    \centering
    \includegraphics[width=0.5\linewidth]{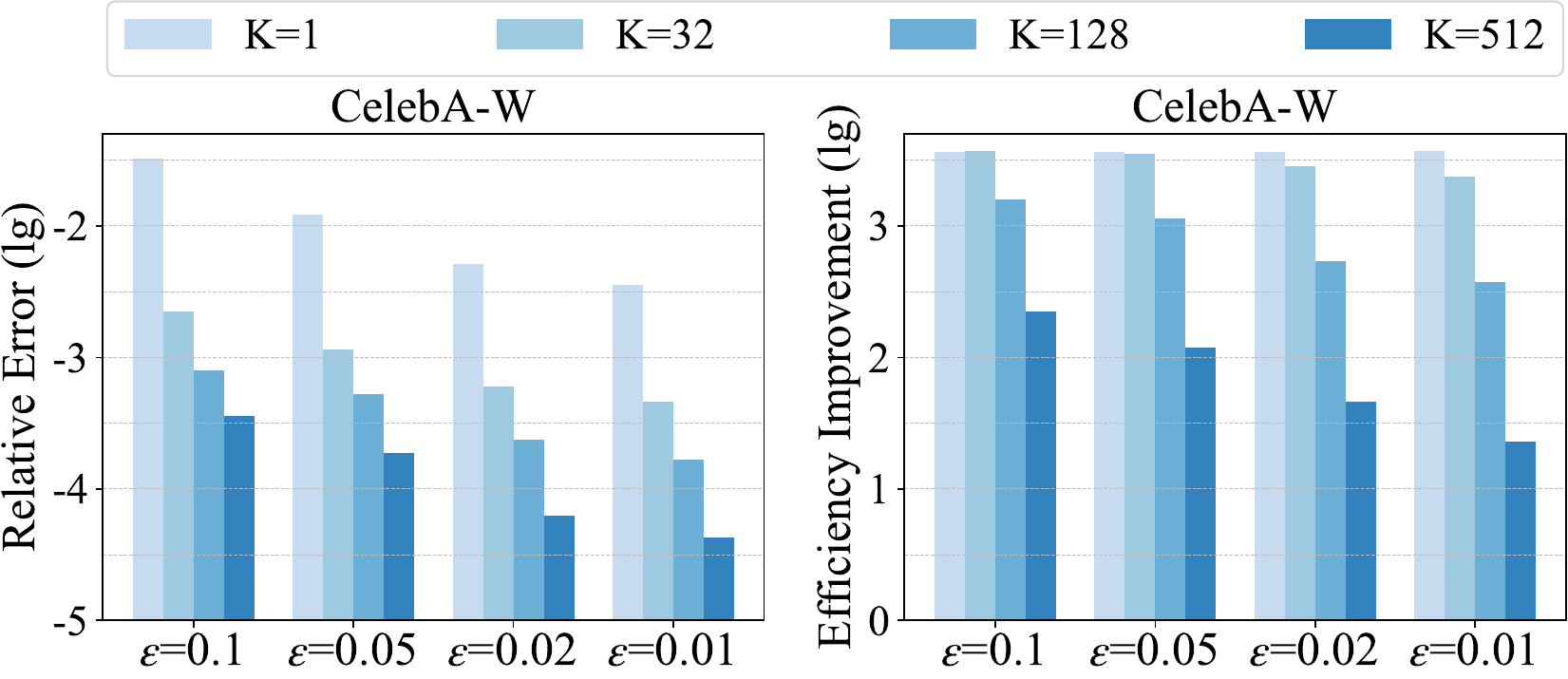}
    \vspace{-12pt}
    \caption{Varying $K$ and $\epsilon$ in $\widehat{\mathrm{MCDP}}(\epsilon)$ calculation algorithms on the CelebA-W dataset.}
    \label{fig:celeba-w_epsK}
    \vspace{-5pt}
\end{figure*}

To further illustrate the importance of improving computational efficiency, we report the running time of Algorithm \ref{alg:exact} $T_e$ in Table \ref{tab:runtime}. We can observe that the exact calculation algorithm exhibits prolonged execution time in a single run. In research and industry scenarios which involves frequent model evaluations, the total evaluation time for multiple experiments would be unaffordable if the exact algorithm is employed. Nevertheless, the runtime can be shortened to hundreds or thousands of times using the approximate calculation algorithm (\ie $T_a\ll T_e$), which successfully resolves the efficiency issue.

\begin{table*}[htb]
\vspace{-8pt}
\caption{The single-run execution time $T_e$ of the exact calculation algorithm.}
\label{tab:runtime}
\vspace{1pt}
\centering
\begin{tabular}{l|cccc}
\toprule
\textbf{Dataset} & \textbf{Adult} & \textbf{Bank} & \textbf{CelebA-A} & \textbf{CelebA-W} \\ \hline
Time (s) & \ms{3.529}{0.122} & \ms{2.900}{0.006} & \ms{7.748}{0.248} & \ms{7.745}{0.169} \\ \bottomrule
\end{tabular}
\vspace{-10pt}
\end{table*}

\subsection{Varying Temperature $\tau$}\label{sec:varytau}
Continuing from Section \ref{sec:indepth}, we show the effect of varying temperature $\tau$ on the performance of DiffMCDP algorithm on CelebA dataset in Figure \ref{fig:temperature_image}. Similarly, very large or small temperatures ($\tau=5,50$) causes sub-optimal results, whereas moderate temperatures ($\tau=10,20$) more effectively trade-off the accuracy and fairness.

\begin{figure*}[!h]
    \centering
    \includegraphics[width=0.5\linewidth]{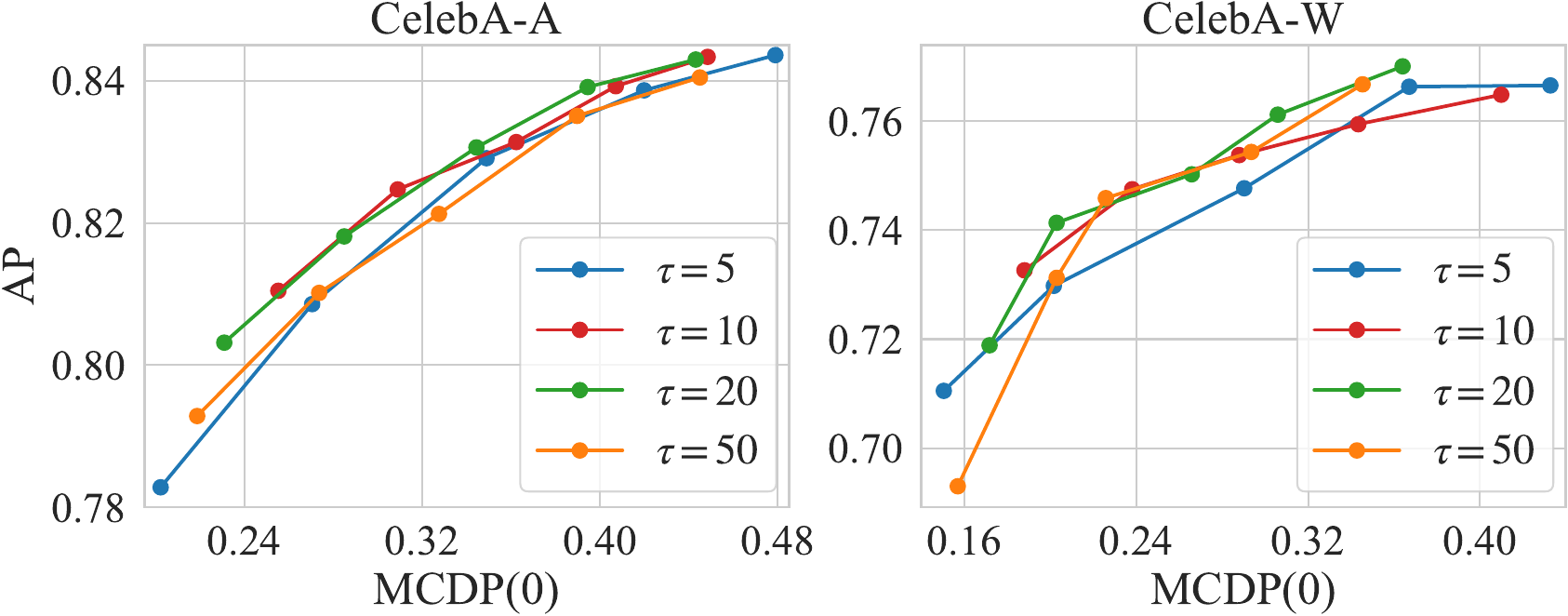}
    \vspace{-10pt}
    \caption{Trade-offs of DiffMCDP with varying temperature $\tau$ on image datasets.}
    \label{fig:temperature_image}
    \vspace{-15pt}
\end{figure*}

\section{Implementation of $\widehat{\mathrm{MCDP}}(\epsilon)$ Calculation Algorithms}

In this section, we provide the implementations of the exact and approximate calculation of $\widehat{\mathrm{MCDP}}(\epsilon)$ metric in Algorithms \ref{code:exact} and \ref{code:approximate}, respectively, which shows that our proposed metric can be employed in fairness evaluation of various scenarios.

\begin{center}
\begin{minipage}{.9\linewidth}
\begin{algorithm}[H]
\caption{Implementation of the Exact Calculation of $\widehat{\mathrm{MCDP}}(\epsilon)$ in Algorithm \ref{alg:exact}}\label{code:exact}
\begin{lstlisting}
from statsmodels.distributions.empirical_distribution import ECDF

def MCDP_exact(y_pred, s, epsilon):
    # y_pred: continuous model predictions in [0,1]
    # s: binary sensitive attribute in {0,1}
    # epsilon: neighborhood hyper-parameter

    y_pred, s = y_pred.ravel(), s.ravel()
    y_pre_1, y_pre_0 = y_pred[s==1], y_pred[s==0]
    # calculate the empirical distribution functions (line 1)
    ecdf0, ecdf1 = ECDF(y_pre_0), ECDF(y_pre_1)

    y_pred = np.r_[0,y_pred,1]  # line 2
    if epsilon == 0:
        # return the maximal deltaF value of predictions (line 4)
        mcdp = np.max(np.abs(ecdf0(y_pred)-ecdf1(y_pred)))
    else:
        y_pred = np.sort(y_pred)    # for indice slicing operations
        # initialize the mcdp value (line 6)
        init_ypred = y_pred[y_pred <= epsilon]
        mcdp = np.min(np.abs(ecdf0(init_ypred) - ecdf1(init_ypred)))

        # traverse all predictions on instances in D (lines 7-9)
        I = y_pred[y_pred <= 1-epsilon].reshape((-1,1))     # line 7
        # select predictions in 2eps intervals with left endpoints in I (lines 8-9)
        J = np.where(I<=y_pred,1,0) * np.where(I+2*epsilon>=y_pred,1,0)
        J = J.reshape(-1) * np.tile(np.arange(len(y_pred))+1,len(I))
        J = J.reshape((len(I),len(y_pred)))
        
        # calculate the minimal deltaF values of each 2eps intervals (line 10)
        delta_ecdf = np.r_[1.1, np.abs(ecdf0(y_pred) - ecdf1(y_pred))]
        delta_ecdf = delta_ecdf[J]  # (len(I),len(y_pred))
        min_delta_ecdf = np.min(delta_ecdf,axis=1)
        # keep the maximum of minimal deltaF values (line 11)
        mcdp = np.max(mcdp, np.max(min_delta_ecdf))*100

    return mcdp
\end{lstlisting}
\end{algorithm}
\end{minipage}
\end{center}

\newpage
\begin{center}
\begin{minipage}{.9\linewidth}
\begin{algorithm}[H]
\caption{Implementation of the Approximate Calculation of $\widehat{\mathrm{MCDP}}(\epsilon)$ in Algorithm \ref{alg:approximate}}\label{code:approximate}
\begin{lstlisting}
from statsmodels.distributions.empirical_distribution import ECDF

def MCDP_approximate(y_pred, s, epsilon, K):
    # y_pred: continuous model predictions in [0,1]
    # s: binary sensitive attribute in {0,1}
    # epsilon: neighborhood hyper-parameter
    # K: sampling frequency

    assert epsilon>0
    y_pred, s = y_pred.ravel(), s.ravel()
    y_pre_1, y_pre_0 = y_pred[s==1], y_pred[s==0]

    # calculate the empirical distribution functions (line 1)
    ecdf0, ecdf1 = ECDF(y_pre_0), ECDF(y_pre_1)
    delta = epsilon / K             # set the step-size (line 2)
    samples = np.arange(0,1,delta)  # sampling prediction points

    # initialize the mcdp value (line 3)
    mcdp = np.min(np.abs(ecdf0(samples[:K+1])-ecdf1(samples[:K+1])))
    # pre-compute the deltaF values for sampled points
    delta_ecdf = np.abs(ecdf0(samples) - ecdf1(samples))

    # traverse all consecutive 2K sampled points (lines 4-5)
    indices = np.arange(1,np.ceil(1/delta)-2*K+1).astype(int)   # line 4
    indices = indices + np.arange(2*K).reshape((-1,1))          # line 5
    # keep the maximum of the minimal deltaF value (line 6)
    mcdp = np.max(mcdp, np.max(np.min(delta_ecdf[indices],axis=0)))

    return mcdp
\end{lstlisting}
\end{algorithm}
\end{minipage}
\end{center}

\end{document}